%% file: sparse_and_spurious.tex
\documentclass[10pt,twocolumn,twoside]{IEEEtran} 
\usepackage[numbers]{natbib}
\usepackage{enumerate}
\usepackage{amsmath, amsthm}
\usepackage{amsfonts}
\usepackage{amssymb, stmaryrd}
\usepackage{bbm}
\usepackage{graphicx,caption, natbib}
\usepackage{subfig}
\usepackage{multirow,rotating}
\usepackage{textcomp}
\usepackage{booktabs}
\usepackage{footnote}
\usepackage{pifont}
\usepackage[usenames,dvipsnames]{color}
\newcommand{\xmark}{\ding{55}}
\input{macros.tex}

\renewcommand{\leq}[0]{\leqslant}
\def\yes {\textcolor{ForestGreen}{\checkmark}}
\def\no {\textcolor{red}{\xmark}}
\title{Sparse and spurious:\\ dictionary learning with noise and outliers}

\author{
R\'emi Gribonval,~\IEEEmembership{IEEE Fellow}, Rodolphe Jenatton, Francis Bach
}

\begin{document}

\maketitle

\begin{abstract}
A popular approach within the signal processing and machine learning communities consists in modelling
signals as sparse linear combinations of atoms selected from a \emph{learned} dictionary.
While this paradigm has led to numerous empirical successes in various fields ranging from image to audio processing, 
there have only been a few theoretical arguments supporting these evidences.
In particular, sparse coding, or sparse dictionary learning, relies on a non-convex procedure whose local minima have not been fully analyzed yet.
In this paper, we consider a probabilistic model of sparse signals, and show that, with high probability, 
sparse coding admits a local minimum around the reference dictionary generating the signals. 
Our study takes into account the case of over-complete dictionaries, noisy signals, and possible outliers, thus extending previous work limited to noiseless settings and/or under-complete dictionaries.
The analysis we conduct is non-asymptotic and makes it possible to understand how the key quantities of the problem, 
such as the coherence or the level of noise, can scale with respect to the dimension of the signals, the number of atoms, the sparsity and the number of observations.
\end{abstract}

\newcommand{\blfootnote}[1]{{\let\thefootnote\relax\footnotetext{#1}}}
\blfootnote{This is a substantially revised version of a first draft that appeared as a preprint titled ``Local stability and robustness of sparse dictionary learning in the presence of noise'', \cite{jenatton:hal-00737152}.}
\blfootnote{This work was supported in part by the EU FET- Open programme through the SMALL Project under Grant 225913 and in part by the European Research Council through the PLEASE Project (ERC-StG-2011-277906) and the SIERRA project (ERC-StG-2011-239993).}
\blfootnote{R. Gribonval is with the Institut de Recherche en Syst{\`e}mes Al{\'e}atoires (Inria \& CNRS UMR 6074), Rennes 35042, France (email: remi.gribonval@inria.fr).}
\blfootnote{R. Jenatton was with the Laboratoire d'Informatique, {\'E}cole Normale Sup{\'e}rieure, Paris 75005, France. He is now the Amazon Development Center Germany, Berlin 10178, Germany (e-mail: jenatton@amazon.com).}
\blfootnote{F. Bach is with the Laboratoire d'Informatique, {\'E}cole Normale Sup{\'e}rieure, Paris 75005, France (e-mail: francis.bach@ens.fr).}
\blfootnote{Copyright (c) 2015 IEEE. Personal use of this material is permitted.  However, permission to use this material for any other purposes must be obtained from the IEEE by sending a request to pubs-permissions@ieee.org.}
\section{Introduction}
Modelling signals as sparse linear combinations of atoms selected from a dictionary has become
a popular paradigm in many fields, including signal processing, statistics, and machine learning.
This line of research has witnessed the development of several well-founded
theoretical frameworks~(see, e.g.,~\citep{Wainwright2009, Zhang2009}) 
and efficient algorithmic tools~(see, e.g.,~\citep{Bach2011} and references therein).

However, the performance of such approaches hinges on the representation of the signals, which makes the question of designing ``good'' dictionaries prominent.
A great deal of effort has been dedicated to come up with efficient \emph{predefined} dictionaries, e.g., the various types of wavelets~\citep{Mallat:2008aa}.
These representations have notably contributed to 
many successful image processing applications such as compression, denoising and deblurring.
More recently, the idea of simultaneously \emph{learning} the dictionary and the sparse decompositions of the signals---also known as \emph{sparse dictionary learning}, or simply, \emph{sparse coding}---has emerged as a powerful framework, 
with state-of-the-art performance in many tasks, 
including inpainting and image classification~(see, e.g.,~\citep{Mairal2010} and references therein).

Although  sparse dictionary learning can sometimes be formulated as convex~\citep{Bach2008c, Bradley2009a}, 
non-parametric Bayesian~\citep{Zhou2009} and submodular~\citep{Krause2010} problems,
the most popular and widely used definition of sparse coding brings into play a non-convex optimization problem.
Despite its empirical and practical success, the theoretical analysis of the properties 
of sparse dictionary learning is still in its infancy. A recent line of work \citep{Maurer2010,Vainsencher2010,Mehta2012} establishes generalization bounds which quantify how much the \emph{expected} signal-reconstruction error differs from the \emph{empirical} one,
computed from a random and finite-size sample of signals. 
In particular, the bounds obtained by \citet{Maurer2010,Vainsencher2010,Gribonval:2013tx} are non-asymptotic, and uniform with respect to the whole class of dictionaries considered (e.g., those with normalized atoms).
\paragraph{Dictionary identifiability.} This paper focuses on a complementary theoretical aspect of dictionary learning: the characterization of local minima of an optimization problem associated to sparse coding, in spite of the non-convexity of its formulation. This problem is closely related to the question of \emph{identifiability}, that is, 
whether it is possible to {\em recover} a reference dictionary that is assumed to generate the observed signals.
Identifying such a dictionary is important when the interpretation of the learned atoms matters, e.g., 
in source localization~\citep{ComonJutten2010}, where the dictionary corresponds to the so-called mixing matrix indicating directions of arrival, in topic modelling~\citep{Jenatton2010b}, where the atoms of the dictionary are expected to carry semantic information, or in neurosciences, where learned atoms have been related to the properties of the visual cortex in the pioneering work of \citet{field96:_emerg}.

In fact, characterizing how accurately one can estimate a dictionary through a given learning scheme also matters beyond such obvious scenarii where the dictionary intrinsically carries information of interest. For example, when learning a dictionary for coding or denoising, two dictionaries are considered as perfectly equivalent if they lead to the same distortion-rate curve, or the same denoising performance. In such contexts, learning an ideal dictionary through the direct optimization of the idealized performance measure is likely to be intractable, and it is routinely replaced by heuristics involving the minimization of proxy, i.e., a better behaved cost function. Characterizing  (local) minima of the proxy is likely to help in providing guarantees that such minima exist close to those of the idealized performance measure and, more importantly, that they also achieve near-optimal performance.

\paragraph{Contributions and related work.} In contrast to early identifiability results in this direction by~\citet{2005_IEEE_TNN_SCA_GeorgievEtAL,Aharon:2006dj}, which focused on deterministic but combinatorial identifiability conditions with combinatorial algorithms, \citet{Gribonval2010} pioneered the analysis of identifiability using a non-convex objective involving an $\ell^{1}$ criterion, in the spirit of the cost function initially proposed by \citet{ZP01} in the context of blind signal separation. 
In the case where the reference dictionary forms a basis, they obtained local identifiability results with noiseless random $k$-sparse signals, possibly corrupted by some ``mild'' outliers naturally arising with the considered Bernoulli-Gaussian model. 
Still in a noiseless setting and without outliers, with a $k$-sparse Gaussian signal model, the analysis was extended by \citet{Geng2011} to \emph{over-complete} dictionaries, {\em i.e.}, dictionaries composed of more atoms than the dimension of the signals. 
Following these pioneering results, a number of authors have established theoretical guarantees on sparse coding that we summarize in  Table~\ref{tab:stateoftheart}.
 Most of the existing results do not handle noise, and none handles outliers. In particular, the structure of the proofs of \citet{Gribonval2010,Geng2011},
hinges on the absence of noise and cannot be straightforwardly transposed to take into account some noise. 

In this paper, we analyze the local minima of sparse coding {\em in the presence of noise and outliers}.
For that, we consider sparse coding with a regularized least-square cost function involving an $\ell^{1}$ penalty, under certain incoherence assumptions on the underlying ground truth dictionary and appropriate statistical assumptions on the distribution of the training samples.
To the best of our knowledge, this is the first analysis which relates to the widely used sparse coding objective function associated to the online learning approach of  \citet{Mairal2010}. In contrast, most of the emerging work on dictionary identifiability considers either an objective function based on $\ell^{1}$ minimization under equality constraints \citep{Gribonval2010,Geng2011}, for which there is no known efficient heuristic implementation, or on an $\ell^{0}$ criterion  \citep{Schnass:2013vd} {\em {\`a} la} K-SVD \citep{Aharon:2006dj}.
More algorithmic approaches have also recently emerged \citep{Spielman:2012ue,Arora:2013vq} demonstrating the existence of provably good (sometimes randomized) algorithms of polynomial complexity for dictionary learning. 
\citet{Agarwal:2013tya} combine the best of both worlds by providing a polynomial complexity algorithm based on a clever randomized clustering initialization \citep{Agarwal:2013ty,Arora:2013vq} followed by alternate optimization based on an $\ell^{1}$ minimization principle with equality constraints.
While this is a definite theoretical breakthrough, these algorithms are yet to be tested on practical problems, while on open source implementation (SPAMS\footnote{http://spams-devel.gforge.inria.fr/}) of the online learning approach of \citet{Mairal2010} is freely available   and has been extensively exploited on practical datasets over a range of applications.

\begin{savenotes}
\begin{table*}[htbp]
\small
   \centering
   \begin{tabular}{|p{3.8cm}|p{0.15cm}|p{0.15cm}|p{0.15cm}|p{0.6cm}|p{0.15cm}|p{1.8cm}|p{1.6cm}|p{2.8cm}||p{2cm}|} 
         \hline
      Reference    & \begin{sideways}Overcomplete\ \ \end{sideways} & \begin{sideways}Noise\end{sideways} & \begin{sideways}Outliers\ \ \end{sideways} & \begin{sideways} Global min / algorithm\end{sideways} & \begin{sideways}Polynomial algorithm\ \ \end{sideways} &  \begin{sideways}Exact (no noise, \end{sideways}\ \begin{sideways} no outlier, $n$ finite)\end{sideways} 
      & \qquad \begin{sideways}Sample complexity\end{sideways}\ \begin{sideways} (no noise)\end{sideways}  
      & \quad \begin{sideways}Admissible sparsity \end{sideways} \begin{sideways}for {\em exact} recovery\end{sideways}  & \qquad \begin{sideways}
      Coefficient model\end{sideways}  \begin{sideways}(main characteristics)\end{sideways}  \\
      \hline\hline
      \citet{2005_IEEE_TNN_SCA_GeorgievEtAL} 	& & & & & & &  &   $k=m-1$, & \\
	{\em Combinatorial approach} 
	& \yes 	& \no 		& \no 	& \yes &\no & \yes	 & $m {p \choose m-1}$ & $\LRIP_{m}(\Dbo) < 1$ & Combinatorial\\ \hline
      \citet{Aharon:2006dj} & & & & & & &  & & \\
	{\em Combinatorial approach}
	& \yes 	& \no 		& \no 	& \yes &\no & \yes  & $(k+1) {p \choose k}$ & $\LRIP_{2k}(\Dbo)<1$ & Combinatorial \\\hline
      \citet{Gribonval2010}    & & & & & & & & $\tfrac{k}{m} <$ & Bernoulli($k/p$) \\
      	{\em $\ell^{1}$ criterion}
	& \no 		& \no 		& \no  
      									& \no & \no & \yes	 & $\frac{m^{2} \log m}{k}$ &  $1-\|\Db^{\top}\Db-\Ib\|_{2,\infty}$ & -Gaussian\\\hline

      \citet{Geng2011}      	 & & & & & & & & & $k$-sparse \\
      	{\em $\ell^{1}$ criterion}
	& \yes 	& \no 		& \no   & \no & \no 	& \yes  & $k p^{3}$ & $O(1/\mu_{1}(\Dbo))$ &  -Gaussian \\\hline

      \citet{Spielman:2012ue} & & & & & & & & & Bernoulli($k/p$) \\
      	{\em $\ell^{0}$ criterion}			
	& \no 		& \no 		& \no  
      									& \yes & \no	& \yes   & $m \log m$  & $O(m)$ & -Gaussian or \\ 
      	{\em ER-SpUD ({\bf randomized})}  & & & & $P($\yes$)$ &  \yes & \yes & $m^{2} \log^{2} m$ &$O(\sqrt{m})$ &   -Rademacher  \\\hline

      \citet{Schnass:2013vd}  & & & & & & $\|\hat{\Db}-\Dbo\|_{2,\infty}$   & & & ``Symmetric \\
      {\em K-SVD criterion}\hfill \quad $~$ \qquad\qquad (unit norm tight frames only)
      & \yes 	& \yes 	& \no    
      									& \no & \no & $\leq r= O(pn^{-1/4})$	 & $mp^{3}$ & $O(1/\mu_{1}(\Dbo))$ &  decaying'': $\alphab_{j} = \epsilon_{j} \mathbf{a}_{\sigma(j)}$ \\ \hline
      \citet{Arora:2013vq}     	& & & & & & $\|\hat{\Db}-\Dbo\|_{2,\infty}$  & $\frac{p^{2}\log p}{k^{2}}$ & $O\big(\min(\frac{1}{\mu_{1}(\Dbo) \log m},$ & $k$-sparse \\
      {\em Clustering ({\bf randomized})}
      & \yes 	& \yes  
      							& \no  	& $P($\yes$)$ & \yes & $\leq r$   &  $+p \log p\cdot \big(k^{2}+ \log \tfrac{1}{r}\big)$  &  $p^{2/5}\big)$  & $1 \leq |\alpha_{j}| \leq C$ \\ \hline
      \citet{Agarwal:2013ty} & & & & & & & & $O\big(\min(1/\sqrt{\mu_1(\Dbo)},$ & $k$-sparse \\
      {\em Clustering ({\bf randomized}) \& $\ell^{1}$}
      & \yes 	& \no 		& \no 	& $P($\yes$)$ & \yes & \yes  & $p \log mp$ &  $m^{1/5},p^{1/6})\big)$ \quad\quad\quad(+ dynamic range)&  -Rademacher \\ \hline
      \citet{Agarwal:2013tya} & & & & & &  & &  $O\big(\min(1/\sqrt{\mu_1(\Dbo)},$ & $k$-sparse \\
      {\em $\ell^{1}$ optim with AltMinDict \& {\bf randomized} clustering init.}
      & \yes 	& \no 		& \no 	& $P(\yes)$ & \yes	& \yes  & $p^{2} \log p$ & $m^{1/9},p^{1/8})\big)$  & - i.i.d. \hfill \quad $~$ 
      \qquad $\loweralpha \leq |\alpha_{j}| \leq M$\\ \hline
      \citet{Schnass:2014vk}	& & & & & & $\|\hat{\Db}-\Dbo\|_{2,\infty}$ & & & ``Symmetric \\
      {\em Response maxim. criterion }
      & \yes 	& \yes & \no 
      									&	\no & \no & $\leq r$ & $\frac{mp^{3}k}{r^{2}}$	& $O(1/\mu_{1}(\Dbo))$ &  decaying'' \\ \hline
      \hline
      {\bf This contribution} & & & & & & $\|\hat{\Db}-\Dbo\|_{\fro}$  & & & $k$-sparse,  \\
      {\em Regularized $\ell^{1}$ criterion with penalty factor $\lambda$}
      & \yes 	& \yes 	& \yes &	\no & \no & $\leq r = O(\lambda)$\ \yes for $\lambda \to 0$  & $mp^{3}$ & $\mu_{k}(\Dbo) \leq 1/4$ & $\loweralpha \leq |\alpha_{j}|$, $\|\alphab\|_{2} \leq M_{\alphab}$ \\ \hline 
   \end{tabular}
   \caption{Overview of recent results in the field. For each approach, the table indicates (notations in Section~\ref{sec:notations}): \\
   1) whether the analysis can handle overcomplete dictionaries / the presence of noise / that of outliers; \\
   2) when an optimization criterion is considered: whether its global minima are characterized (in contrast to characterizing the presence of a local minimum close to the ground truth dictionary $\Dbo$); alternatively, whether a (randomized) algorithm with success guarantees is provided; the notation $P(\checkmark)$ indicates success with high probability of a randomized algorithm;\\
   3) whether a (randomized) algorithm with proved polynomial complexity is exhibited; \\
   4) whether the output $\hat{\Db}$ of the algorithm (resp. the characterized minimum of the criterion) is (with high probability) {\em exactly} the ground truth dictionary, in the absence of noise and outliers and with finitely many samples. Alternatively the guaranteed upper bound on the distance between $\hat{\Db}$ and $\Dbo$ is provided;\\
   5) the sample complexity $n = \Omega(\cdot)$, under the scaling $\triple \Dbo\triple_{2}=O(1)$, in the absence of noise;\\ 
   6) the sparsity levels $k$  allowing ``exact recovery''; \\ 
   7) a brief description of the models underlying the corresponding analyses. Most models are determined by: 
    \indent i) how the support is selected (a $k$-sparse support, or one selected through a $Bernoulli(k/p)$ distribution, i.e., each entry is nonzero with probability $k/p$); and 
    ii) how the nonzero coefficients are drawn: Gaussian, Rademacher ($\pm 1$ entries with equal probability), i.i.d.~with certain bound and variance constraints.
   The symmetric and decaying model of \citet[Definitions 2.1,2.2 ]{Schnass:2013vd} first generates a coefficient decay profile $\mathbf{a} \in \RR^{p}$, then the coefficient vector $\alphab$ using a random permutation $\sigma$ of indices and i.i.d. signs $\epsilon_{i}$. }
   \label{tab:stateoftheart}
\end{table*}
\end{savenotes}

\paragraph{Main contributions.}  Our main contributions can be summarized as follows:
 
\begin{enumerate}
\item We consider the recovery of a dictionary with $p$ atoms $\Dbo \in \Real^{m \times p}$ using $\ell_1$-penalized formulations with penalty factor $\lambda > 0$, given a training set of $n$ signals gathered in a data matrix $\Xb \in \Real^{m \times n}$. This is detailed in Section~\ref{sec:sparsecoding}.
\item We assume a general probabilistic model of sparse signals, where the data matrix $\Xb \in \Real^{m \times n}$ is generated as $\Dbo \Ab^o$ plus additive noise $\varepsilonb$.  Our model, described in Section~\ref{sec:sparseandspurious}, corresponds to a $k$-sparse support with loose decorrelation assumptions on the nonzero coefficients. It is closely connected to the $\Gamma_{k,C}$ model of  \citet[Definition 1.2]{Arora:2013vq}. In particular, unlike in independent component analysis (ICA) and in most related work, {\em no independence is assumed between nonzero coefficients}. 
\item We show that  under deterministic (cumulative) coherence-based sparsity assumptions (see Section~\ref{sec:coherence}) the minimized cost function has a guaranteed local minimum around the generating dictionary $\Dbo$ with high probability. 
\item We also prove support and coefficient recovery, which is important for blind source separation.
\item Our work makes it possible to better understand:
\begin{enumerate}
\item how small the neighborhood around the reference dictionary can be, i.e., tending to zero as the noise variance goes to zero. 
\item how many signals $n$ are sufficient to hope for the existence of such a controlled local minimum, i.e., $n = \Omega(m p^3)$.
In contrast to several recent results \citep{Schnass:2013vd,Arora:2013vq,Schnass:2014vk} where the sample complexity depends on the targeted resolution $r$ such that $\|\hat{\Db}-\Dbo\| \leq r$, our main sample complexity estimates are {\em resolution-independent} in the noiseless case. 
This is similar in nature to the better sample complexity results $n = \Omega(p^{2} \log p)$ obtained by \citet{Agarwal:2013tya} for a polynomial algorithm in a noiseless context, or $n = \Omega(p \log mp)$ obtained by \citet{Agarwal:2013ty} for Rademacher coefficients. 
This is achieved through a precise sample complexity analysis using Rademacher averages and Slepian's lemma. 
In the presence of noise, a factor $1/r^{2}$ seems unavoidable \citep{Schnass:2013vd,Arora:2013vq,Schnass:2014vk}.
\item what sparsity levels are admissible. Our main result is based on the cumulative coherence (see Section~\ref{sec:coherence}) $\mu_{k}(\Dbo) \leq 1/4$. It also involves a condition that restricts our analysis to overcomplete dictionaries where $p \lesssim m^{2}$, where previous works seemingly apply to very overcomplete settings. Intermediate results only involve restricted isometry properties. This may allow for much larger values of the sparsity level $k$, and more overcompleteness, but this is left to future work. 
\item what level of noise and outliers appear as manageable, with a precise control of the admissible ``energy'' of these outliers. While a first naive analysis would suggest a tradeoff between the presence of outliers and the targeted resolution $r$, we conduct a tailored analysis that demonstrates the existence of a {\em resolution-independent} threshold on the relative amount of  outliers to which the approach is robust. 
\end{enumerate}
 
\end{enumerate}

\section{Problem statement}
\label{sec:Statement}
We introduce in this section the material required to define our problem and state our results.
\paragraph{Notations.}\label{sec:notations}
For any integer $p$, we define the set $\SET{p} \defin \{1,\dots,p\}$.
For all vectors $\vb \in \Real^p$, we denote by $\sign(\vb) \in \{ -1,0,1 \}^p$ the vector
such that its $j$-th entry $[\sign(\vb)]_j$ is equal to zero if $\vb_j=0$, and to one (respectively, minus one) if $\vb_j > 0$ (respectively, $\vb_j < 0$).  
The notations $\Ab^{\top}$ and $\Ab^{+}$ denote the transpose and the Moore-Penrose pseudo-inverse of a matrix $\Ab$. 
We extensively manipulate matrix norms in the sequel. For any matrix $\Ab \in \Real^{m\times p}$, we define its Frobenius norm by 
$\|\Ab\|_\fro\defin [\sum_{i=1}^m\sum_{j=1}^p \Ab_{ij}^2 ]^{1/2}$; 
similarly, we denote the spectral norm of $\Ab$ by
$\triple \Ab \triple_2 \defin \max_{ \|\xb\|_2\leq1 } \|\Ab\xb\|_2$, we refer to the operator $\ell_\infty$-norm as
$\triple \Ab \triple_\infty \defin \max_{ \|\xb\|_\infty\leq 1 } \|\Ab\xb\|_\infty = \max_{ i\in\SET{m} } \sum_{j=1}^p |\Ab_{ij}|$, and we denote $\| \Ab\|_{1,2} \defin \sum_{j\in \SET{p}} \|\ab^{j}\|_{2}$ with $\ab^{j}$ the $j$-th column of $\Ab$. In several places we will exploit the fact that for any matrix $\Ab$ we have
\[
\triple \Ab\triple_{2} \leq \|\Ab\|_{\fro}.
\]
For any square matrix $\Bb \in \Real^{n\times n}$, we denote by $\diag(\Bb) \in \Real^n$ the vector formed by extracting the diagonal terms of $\Bb$, and conversely, for any $\bb \in \Real^n$, we use $\Diag(\bb) \in \Real^{n\times n}$ to represent the (square) diagonal matrix whose diagonal elements are built from the vector $\bb$.  Denote $\mathrm{off}(\Ab) \defin \Ab -\Diag(\diag(\Ab))$ the off-diagonal part of $\Ab$, which matches $\Ab$ except on the diagonal where it is zero. The identity matrix is denoted $\Ib$.  

For any $m \times p$ matrix $\Ab$ and index set $\J \subset \SET{p}$ we denote by $\Ab_{\J}$ the matrix obtained by concatenating the columns of $\Ab$ indexed by $\J$.
The number of elements or size of $\J$ is denoted $|\J|$, and its complement in $\SET{p}$ is denoted $\J^{c}$. Given a matrix $\Db \in \Real^{m \times p}$ and a support set $\J$ such that $\DbJ$ has linearly independent columns, we define the shorthands
\begin{eqnarray*}
\GramJb & \defin & \GramJb(\Db) \defin \DbJT \DbJ \label{eq:DefGram}\\
\ThetaJb &\defin & \ThetaJb(\Db) \defin \GramJb^{-1}\label{eq:DefGramInv}\\
\PJb & \defin & \PJb(\Db) \defin \DbJ \DbJ^{+} = \DbJ\ThetaJb \DbJT\label{eq:DefPJ},
\end{eqnarray*}
respectively the Gram matrix of $\DbJ$ and its inverse, and the orthogonal projector onto the span of the columns of $\Db$ indexed by $\J$. 

For any function $h(\Db)$ we define $\Delta h(\Db';\Db) \defin h(\Db')-h(\Db)$. Finally, the ball (resp.~the sphere) of radius $r>0$ centered on $\Db$ in $\Real^{m \times p}$ with respect to the Frobenius norm is denoted $\Bcal(\Db;r)$ (resp.~$\Scal(\Db;r)$). 

The notation $a = O(b)$, or $a \lesssim b$, indicates the existence of a finite constant $C$ such that $a \leq C b$. Vice-versa, $a = \Omega(b)$, or $a \gtrsim b$, means $b = O(a)$, and $a \asymp b$ means that $a = O(b)$ and $b = O(a)$ hold simultaneously.

\subsection{Background material on sparse coding}\label{sec:sparsecoding}

Let us consider a set of $n$ signals $\Xb \defin [\xb^1,\dots,\xb^n]\! \in\! \Real^{m\times n}$ each of dimension $m$,
along with a dictionary 
$\Db \defin [\db^1,\dots,\db^p]\! \in \Real^{m\times p}$ formed of $p$ columns called atoms---also known as dictionary elements.
Sparse coding simultaneously learns $\Db$ and
a set of $n$ sparse $p$-dimensional vectors $\Ab \defin[\alphab^1,\dots,\alphab^n]\! \in\! \Real^{p\times n}$, such that each signal $\xb^i$ can be well approximated by 
$
\xb^i \approx \Db \alphab^i
$
for $i$ in $\SET{n}$.
By sparse, we mean that the vector $\alphab^i$ has $k \ll p$ non-zero coefficients, 
so that we aim at reconstructing $\xb^i$ from only a few atoms.
Before introducing the sparse coding formulation~\citep{Olshausen1997,ZP01,Mairal2010}, 
we need some definitions. We denote by $g \colon  \Real^{p} \to \Real^{+}$ a penalty function that will typically promote sparsity. 

\begin{definition} For any dictionary $\Db \in \Real^{m\times p}$ 
and signal $\xb \in \Real^m$, we define 
\begin{eqnarray}
\label{eq:Li}
\Lcal_{\xb}(\Db,\alphab) 
&\defin& \tfrac{1}{2} \|\xb-\Db\alphab\|_2^2 + g(\alphab)\\
\label{eq:fi}
f_\xb(\Db) 
&\defin& 
\inf_{\alphab\in\Real^p} \Lcal_{\xb}(\Db,\alphab). 
\end{eqnarray}
Similarly for any set of $n$ signals $\Xb \defin [\xb^1,\dots,\xb^n] \in \Real^{m \times n}$, we introduce
\begin{equation}\label{eq:f}
F_\Xb(\Db) \defin \tfrac{1}{n} \sum_{i=1}^n f_{\xb^i}(\Db). 
\end{equation}
\end{definition}

Based on problem~(\ref{eq:fi}) with the $\ell^{1}$ penalty, 
\begin{equation}
\label{eq:DefL1Penalty}
g(\alphab) \defin \lambda \|\alphab\|_{1},
\end{equation}
refered to as Lasso in statistics~\citep{Tibshirani1996}, and 
basis pursuit in signal processing~\citep{Chen1998},
the standard approach to perform sparse coding~\citep{Olshausen1997,ZP01,Mairal2010} solves the minimization problem 
\begin{equation}\label{eq:min_fn}
\min_{\Db \in \mathcal{D}} F_\Xb(\Db),
\end{equation}
where the regularization parameter $\lambda$ in~(\ref{eq:DefL1Penalty}) controls the tradeoff between sparsity and approximation quality, 
while $\mathcal{D} \subseteq \Real^{m\times p}$ is a compact constraint set; 
in this paper, $\mathcal{D}$ denotes the set of dictionaries with unit $\ell_2$-norm atoms, also called the {\em oblique manifold}~\citep{Absil2008}, which is a natural choice in signal and image processing~\citep{Mairal2010,Gribonval2010,Rubinstein:2010aa,Tosic2011}.
Note however that other choices for the set $\mathcal{D}$ may also be relevant depending on the application 
at hand~(see, e.g.,~\citet{Jenatton2010b} where in the context of topic models, the atoms in $\mathcal{D}$ belong to the unit simplex). The sample complexity of dictionary learning with general constraint sets is studied by \citet{Maurer2010,Gribonval:2013tx} for various families of penalties $g(\alphab)$. 

\subsection{Main objectives}\label{sec:main_obj}
The goal of the paper is to characterize some local minima of the function $F_\Xb$ with the $\ell^{1}$ penalty, under a generative model for the signals $\xb^i$.
Throughout the paper, the main model we consider is that of observed signals generated \emph{independently} according to a specified probabilistic model.  The  signals are typically drawn  as $\xb^{i} \defin \Dbo \alphab^{i} + \varepsilonb^{i}$ where $\Dbo$ is a fixed reference dictionary, 
$\alphab^{i}$ is a sparse coefficient vector, and $\varepsilonb^{i}$ is a noise term. The specifics of the underlying probabilistic model, and its possible contamination with  {\em outliers} are considered in Section~\ref{sec:sparseandspurious}. 
Under this model, we can state more precisely our objective: we want to show that, for large enough $n$,
$$
\Pr\big(F_\Xb\ \text{has a local minimum in a ``neighborhood'' of}\ \Dbo\big) \approx 1.
$$
We loosely refer to a ``neighborhood'' since in our regularized formulation, a local minimum is not necessarily expected to appear exactly at $\Dbo$.
The proper meaning of this neighborhood is in the sense of the Frobenius distance $\|\Db-\Dbo\|_{\fro}$. Other metrics can be envisioned and are left as future work.   
How large $n$ should be for the results to hold is related to the notion of sample complexity. 

\paragraph{Intrinsic ambiguities of sparse coding.} Importantly, we so far referred to $\Dbo$ as \emph{the} reference dictionary generating the signals. 
However, and as already discussed by~\citet{Gribonval2010,Geng2011} and more generally in the related literature on blind source separation 
and independent component analysis~\citep[see, e.g.,][]{ComonJutten2010}, 
it is known that the objective of~(\ref{eq:min_fn}) is invariant by sign flips and permutations of the atoms.
As a result, while solving~(\ref{eq:min_fn}), we cannot hope to identify the specific~$\Dbo$.
We focus instead on the local identifiability of the whole \emph{equivalence class} defined by the transformations described above.
From now on, we simply refer to $\Dbo$ to denote one element of this equivalence class. 
Also, since these transformations are \emph{discrete}, 
our local analysis is not affected by invariance issues, as soon as we are sufficiently close to some representant of $\Dbo$. 

\subsection{The sparse and the spurious}\label{sec:sparseandspurious}
The considered training set is composed of two types of vectors: {\em the sparse}, drawn i.i.d. from a distribution generating (noisy) signals that are sparse in the dictionary $\Dbo$; and {\em the spurious}, corresponding to {\em outliers}.

\subsubsection{The sparse: probabilistic model of sparse signals ({\em inliers})}\label{sec:gen_model}

Given a reference dictionary $\Dbo \in \mathcal{D}$, 
each ({\em inlier}) signal $\xb \in \Real^m$ is built \emph{independently} in three steps:
\begin{itemize}
\item {\bf Support generation:} Draw uniformly without replacement $k$ atoms out of the $p$ available in $\Dbo$.
This procedure thus defines a support $\J \subset \SET{p}$ whose size is $|\J|=k$.

\item {\bf Coefficient vector:} Draw a sparse vector $\alphabo \in \Real^p$ supported on $\J$ (i.e., with $\alphabo_{\J^{c}}=0$).
\item {\bf Noise:} Eventually generate the signal $\xb=\Dbo \alphabo + \varepsilonb$.
\end{itemize}
The random vectors $\alphabo_{\J}$ and $\varepsilonb$ satisfy the following assumptions, where we denote $\sbo = \sign(\alphabo)$.
 \begin{assumption}[Basic signal model]\label{assume:basic}
 \begin{align}
\Exp \left \{\alphaboJ \alphaboJT\ |\ \J \right\} = &\ \Exp \{\alpha^{2}\} \cdot \Ib \label{eq:coeffwhiteness}  \\  
 & \textbf{\text{coefficient whiteness}} \notag\\
\Exp \left \{\sboJ \sboJT\ |\ \J \right\} = &\ \Ib \label{eq:signwhiteness}   \\ 
& \textbf{\text{sign whiteness}}\notag\\
\Exp \left \{\alphaboJ \sboJT\ |\ \J \right\} = &\ \Exp \{|\alpha|\} \cdot \Ib \label{eq:coeffmagnitude}  \\ 
& \textbf{\text{sign/coefficient decorrelation}} \notag\\
\Exp \left \{\varepsilonb \alphaboJT\ |\ \J \right\} = &\ \Exp \left \{\varepsilonb \sboJT\ |\ \J \right\}  = 0  \\ & \textbf{\text{noise/coefficient decorrelation}} \notag\\
\Exp \left \{\varepsilonb \varepsilonb^{\top} | \J \right\} = &\ \Exp \{\epsilon^{2}\} \cdot \Ib    \\ 
& \textbf{\text{noise whiteness}} \notag
\end{align}
\end{assumption}
In light of these assumptions we define the shorthand
\begin{equation}\label{eq:DefKappa}
\kappa_{\alpha} \defin \frac{\Exp |\alpha|}{\sqrt{\Exp \alpha^{2}}}.
\end{equation}
By Jensen's inequality, we have $\kappa_{\alpha}\leq 1$, with $\kappa_{\alpha}=1$ corresponding to the degenerate situation where $\alphab_{\J}$ almost surely has all its entries of the same magnitude, i.e., with the smallest possible dynamic range. Conversely, $\kappa_{\alpha} \ll 1$ corresponds to marginal distributions of the coefficients with a wide dynamic range. In a way, $\kappa_{\alpha}$ measures the typical ``flatness'' of $\alphab$ (the larger $\kappa_{\alpha}$, the flatter the typical $\alphab$)

A {\em boundedness} assumption will complete Assumption~\ref{assume:basic} to handle sparse recovery in our proofs.
\begin{assumption}[Bounded signal model]\label{assume:bounded}
\begin{align}
\Pr(\min_{j \in \J} |\alphabo_{j}| < \loweralpha\ |\ \J) = 0, &\ \text{for some}\ \loweralpha > 0 \label{eq:coeffthreshold}\\
 &  \textbf{\text{coefficient threshold}}\notag\\
\Pr( \|\alphabo\|_{2} > M_{\alphab}) =0, &\ \text{for some}\ M_{\alphab}  \label{eq:coeffboundeucl}\\
& \textbf{\text{coefficient boundedness}}\notag\\
\Pr( \|\varepsilonb\|_{2} > M_{\varepsilonb}) = 0, &\ \text{for some}\ M_{\varepsilonb} \label{eq:noiseboundeucl}.\\
& \textbf{\text{noise boundedness}} \notag
\end{align}
\end{assumption}

\begin{remark}
Note that neither Assumption~\ref{assume:basic} nor Assumption~\ref{assume:bounded} requires that the entries of $\alphabo$ indexed by~$\J$ be i.i.d.
In fact, the stable and robust identifiability of $\Dbo$ from the training set $\Signal$ rather stems from geometric properties of the training set (its concentration close to a union of low-dimensional subspaces spanned by few columns of $\Dbo$) than from traditional independent component analysis (ICA). This will be illustrated by a specific coefficient model (inspired by the symmetric decaying coefficient model of \citet{Schnass:2013vd})  in Example~\ref{ex:2}.
\end{remark}

To summarize, the signal model is parameterized by the sparsity $k$, the expected coefficient energy $\Exp\ \alpha^{2}$, the minimum coefficient magnitude $\loweralpha$, maximum norm $M_{\alphab}$, and the flatness $\kappa_{\alpha}$. These parameters are interrelated, e.g., $\loweralpha \sqrt{k} \leq M_{\alphab}$.

\paragraph{Related models}\label{sec:describerelatedmodels}
The Bounded model above is related to the $\Gamma_{k,C}$ model of \citet{Arora:2013vq} (which also covers \citep{Agarwal:2013ty,Agarwal:2013tya}): in the latter, our assumptions~\eqref{eq:coeffthreshold}-\eqref{eq:coeffboundeucl} are replaced by $1 \leq |\alpha_{j}| \leq C$. Note that the $\Gamma_{k,C}$ model of \citet{Arora:2013vq} does not assume that the support is chosen uniformly at random (among all k-sparse sets) and some mild dependencies are allowed. Alternatives to~\eqref{eq:coeffboundeucl} with a control on $\|\alphab\|_{q}$ for some $0<q\leq \infty$ can easily be dealt with through appropriate changes in the proofs, but we chose to focus on $q=2$ for the sake of simplicity.
Compared to early work in the field considering a Bernoulli-Gaussian model \citep{Gribonval2010} or a $k$-sparse Gaussian model \citep{Geng2011},  Assumptions~\ref{assume:basic}  \&~\ref{assume:bounded} are rather generic and do not assume a specific shape of the distribution $\Pr(\alphab)$. In particular, the conditional distribution of $\alphab_{\J}$ given $\J$ {\em may depend on $\J$}, provided its ``marginal moments'' $\Exp\ \alpha^{2}$ and $\Exp\ |\alpha|$ satisfy the expressed assumptions. 

\subsubsection{The spurious: {\em outliers}}
 In addition to a set of $n_{\textrm{in}}$ {\em inliers} drawn i.i.d.~as above, the training set may contain $n_{\textrm{out}}$ {\em outliers}, i.e., training vectors that may have completely distinct properties and may not relate in any manner to the reference dictionary $\Dbo$. Since the considered cost function $F_{\Signal}(\Db)$ is not altered when we permute the columns of the matrix $\Signal$ representing the training set, without loss of generality we will consider that $\Signal = [\Signal_{\textrm{in}},\Signal_{\textrm{out}}]$. As we will see, controlling the ratio $\|\Signal_{\textrm{out}}\|_{\fro}^{2}/n_{\textrm{in}}$ of the total energy  of outliers to the number of inliers will be enough to ensure that the local minimum of the sparse coding objective function is robust to outliers. While this control does not require any additional assumptions, the ratio $\|\Signal_{\textrm{out}}\|_{\fro}^{2}/n_{\textrm{in}}$  directly impacts the error in estimating the dictionary (i.e., the local minimum in $\Db$ is further away from $\Dbo$). With additional assumptions (namely that the reference dictionary is complete), we show that if $\|\Signal_{\textrm{out}}\|_{1,2}/n_{\textrm{in}}$ is sufficiently small, then our upper bound on the distance from the local minimum to $\Dbo$ remains valid. 
  
 \begin{figure*}
 \begin{center}
 \hspace*{-1cm}
 \includegraphics[scale=0.5]{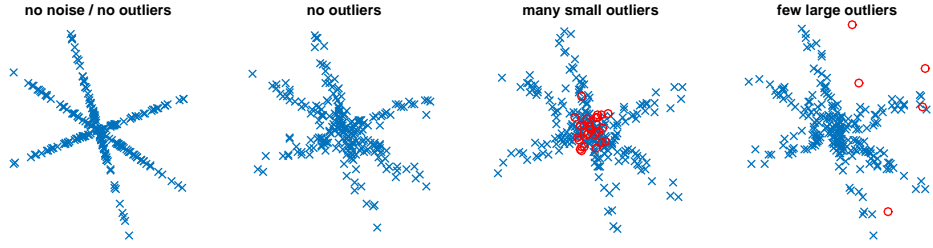}
 \end{center}
 
 \vspace*{-.5cm}

 \caption{Noise and outliers: illustration with three atoms in two dimensions (blue crosses: inliers, red circles:outliers).}
 \end{figure*}
\subsection{The dictionary: cumulative coherence and restricted isometry properties}\label{sec:coherence}

Part of the technical analysis relies on the notion of {\em sparse recovery}. A standard sufficient support recovery condition is referred to as the \emph{exact recovery condition} 
in signal processing~\citep{Fuchs2005,Tropp2004} or the \emph{irrepresentability condition} (IC) in the machine learning and statistics communities~\citep{Wainwright2009, Zhao2006}.
It is a key element to almost surely control the supports of the solutions of $\ell_1$-regularized least-squares problems.
To keep our analysis reasonably simple, we will impose the irrepresentability condition \emph{via} a condition on the \emph{cumulative coherence} of the reference dictionary $\Dbo \in \mathcal{D}$, which is a stronger requirement~\citep{Van2009,Foucart:2012wp}.  
This quantity is defined (see, e.g.,~\cite{Fuchs2005,Donoho2001}) for unit-norm columns (i.e., on the oblique manifold $\mathcal{D}$) as 
\begin{equation}
\mu_{k}(\Db) \defin \sup_{|\J| \leq k} \sup_{j \notin \J} \|\DbJT \db^{j}\|_{1}.
\end{equation}
The term $\mu_{k}(\Db)$ gives a measure of the level of correlation between columns of $\Db$. It is for instance equal to zero in the case of an orthogonal dictionary, and exceeds one if $\Db$ contains two colinear columns. For a given dictionary $\Db$, the cumulative coherence of $\mu_{k}(\Db)$ increases with $k$, and $\mu_{k}(\Db) \leq k \mu_{1}(\Db)$ where $\mu_{1}(\Db) = \max_{i \neq j} |\langle \db^{i},\db^{j}\rangle|$ is the plain coherence of $\Db$.

For the theoretical analysis we conduct, we consider a deterministic assumption based on the cumulative coherence, slightly weakening the coherence-based assumption considered for instance in previous work on dictionary learning~\citep{Gribonval2010,Geng2011}. Assuming that $\mu_{k}(\Dbo) < 1/2$ where $k$ is the level of sparsity of the coefficient vectors $\alphab^{i}$, an important step will be to show that such an upper bound on $\mu_{k}(\Dbo)$ loosely transfers to $\mu_{k}(\Db)$ provided that $\Db$ is close enough to $\Dbo$, leading to {\em locally stable exact recovery results in the presence of bounded noise} (Proposition~\ref{prop:simplified_exact_recovery}). 

Many elements of our proofs rely on a restricted isometry property (RIP), which is known to be weaker than the coherence assumption~\citep{Van2009}. 
By definition the {\em restricted isometry constant} of order $k$ of a dictionary $\Db$, $\RIP_{k}(\Db)$ is the smallest number $\RIP_{k}$ such that for any support set $\J$ of size $|\J| = k$ and $\zb \in \Real^k$,
\begin{equation}
\label{eq:DefRICk}
\left(1-\RIP_{k}\right) \|\zb\|_{2}^{2} \leq \|\Db_{\J}\zb\|_{2}^{2} \leq \left(1+\RIP_{k}\right) \|\zb\|_{2}^{2}.
\end{equation}
In our context, {\em the best lower bound and best upper bound will play significantly different roles}, so we define them separately as $\LRIP_{k}(\Db)$ and $\URIP_{k}(\Db)$, so that $\RIP_{k}(\Db) = \max(\LRIP_{k}(\Db),\URIP_{k}(\Db))$. Both can be estimated by the cumulative coherence as $\delta_{k}(\Db) \leq \mu_{k-1}(\Db)$ by Gersgorin's disc theorem \citep{Tropp2004}.
Possible extensions of this work that would fully relax the incoherence assumption and only rely on the RIP are discussed in Section~\ref{sec:discussion}.

\input{sketchedresults.tex}

\section{Main steps of the analysis}\label{sec:proofoutline}

For many classical penalty functions $g$, including the considered $\ell^{1}$ penalty $g(\alphab) = \lambda \|\alphab\|_{1}$, the function $\Db \mapsto F_{\Xb}(\Db)$ is continuous, and in fact Lipschitz \citep{Gribonval:2013tx} with respect to the Frobenius metric $\rho(\Db',\Db) \defin \|\Db'-\Db\|_{\fro}$ on all $\Real^{m \times p}$, hence in particular on the compact constraint set $\mathcal{D} \subset \Real^{m \times p}$. Given a dictionary $\Db \in \mathcal{D}$, we have $\|\Db\|_{\fro}=\sqrt{p}$, and for any radius $0<r\leq 2\sqrt{p}$ the sphere  
\[
\Scal(r) \defin \Scal(\Dbo;r) = \{ \Db \in \mathcal{D} : \|\Db-\Dbo\|_{\fro} = r\}
\]
is non-empty (for $r = 2\sqrt{p}$ it is reduced to $\Db = -\Dbo$). We define
\begin{eqnarray}
\label{eq:DefDeltaFn}
\Delta F_{\Xb}(r) 
&\defin& 
\inf_{\Db \in \Scal(r)} \Delta F_{\Xb}(\Db;\Dbo). 
\end{eqnarray}
where we recall that for any function $h(\Db)$ we define $\Delta h(\Db;\Db') \defin h(\Db)-h(\Db')$. 
Our proof technique will consist in choosing the radius $r$ to ensure that $\Delta F_{\Xb}(r)>0$ (with high probability on the draw of $\Signal$): the compactness of the closed balls 
\begin{equation}
\Bcal(r) \defin \Bcal(\Dbo;r) =  \left\{\Db \in \mathcal{D}: \|\Db-\Dbo\|_{\fro} \leq r \right\}
\end{equation}
will then imply the existence of a local minimum $\hat{\Db}$ of $\Db \mapsto F_{\Xb}(\Db)$ such that $\|\hat{\Db}-\Dbo\|_{\fro} < r$.

\subsection{The need for a precise finite-sample ({\em vs.}~asymptotic) analysis}\label{sec:sample_complexity}
Under common assumptions on the penalty function $g$ and the distribution of ``clean'' training vectors $\xb \sim \mathbb{P}$, the empirical cost function $F_{\Xb}(\Db)$ converges uniformly to its expectation $\mathbb{E}_{\xb \sim \mathbb{P}} f_{\xb}(\Db)$: except with probability at most $2e^{-x}$ \citep{Maurer2010,Vainsencher2010,Gribonval:2013tx}, we have
\begin{equation}
\label{eq:UniformConvergence2}
	\sup_{\Db \in \mathcal{D}} \left| F_{\Xb}(\Db) -  \mathbb{E}_{\xb \sim \mathbb{P}} f_{\xb}(\Db) \right| \leq \eta_{n}.
\end{equation}
where $\eta_{n}$ depends on the penalty $g$, the data distribution $\mathbb{P}$, the set $\Scal(r)$ (via its covering number) and the targeted probability level $1-2e^{-x}$. Thus, with high probability,
\[
\Delta F_{\Xb}(r) \geq \Delta f_{\mathbb{P}}(r) - 2\eta_{n}
\]
with
\begin{eqnarray}
\label{eq:DefDeltaFnExp}
\Delta f_{\mathbb{P}}(r)
& \defin & \inf_{\Db \in \Scal(r)} \Delta f_{\mathbb{P}}(\Db;\Dbo)\\
\textrm{where}\  f_{\mathbb{P}}(\Db) 
& \defin & \mathbb{E}_{\xb \sim \mathbb{P}} f_{\xb}(\Db).
\end{eqnarray}
As a result, showing that $\Delta f_{\mathbb{P}}(r)>0$ will imply that, with high probability, the function $\Db \mapsto F_{\Xb}(\Db)$ admits a local minimum $\hat{\Db}$ such that $\|\hat{\Db}-\Dbo\|_{\fro} < r$, provided that the number of training samples $n$ satisfies $\eta_{n} < \Delta f_{\mathbb{P}}(r)/2$. 
For the $\ell^{1}$ penalty $g(\alphab) = \lambda \|\alphab\|_{1}$, the generative model considered in Section~\ref{sec:gen_model}, and the oblique manifold $\mathcal{D}$, a direct application of the results of \citep{Gribonval:2013tx} yields $\eta_{n} \leq c \sqrt{\tfrac{(mp+x) \cdot \log n}{n}}$ for some explicit constant $c$. The desired result follows when the number of training samples satisfies
\[
\frac{n}{\log n} \geq (mp+x) \cdot \frac{4 c^{2}}{\left[\Delta f_{\mathbb{P}}(r)\right]^{2}}.
\]
This is slightly too weak in our context where the interesting regime is when $\Delta f_{\mathbb{P}}(r)$ is non-negative but small. Typically, in the noiseless regime, we target an arbitrary small radius $r>0$ through a penalty factor $\lambda \asymp r$ and get $\Delta f_{\mathbb{P}}(r) = O(r^{2})$.
Since $c$ is a fixed constant, the above direct sample complexity estimates apparently suggests $n/\log n = \Omega(mp r^{-2})$, a number of training sample that grow arbitrarily large when the targeted resolution $r$ is arbitrarily small. Even though this is the behavior displayed in recent related work \citep{Schnass:2013vd,Arora:2013vq,Schnass:2014vk}, this is not fully satisfactory, and we get more satisfactory {\em resolution independent} sample complexity estimates $n = \Omega(mp)$ through more refined Rademacher averages and Slepian's lemma in Section~\ref{sec:proofthmmainfinite}. Incidentally we also gain a $\log n$ factor.

\subsection{Robustness to outliers}\label{sec:outliers}
Training collections are sometimes contaminated by {\em outliers}, i.e., training samples somehow irrelevant to the considered training task in the sense that they do not share the ``dominant'' properties of the training set. Considering a collection $\Xb$ of $n_{\textrm{in}}$ {\em inliers} and $n_{\textrm{out}}$ outliers, and $\Xb_{\textrm{in}}$ (resp. $\Xb_{\textrm{out}}$) the matrix extracted from~$\Xb$ by keeping only its columns associated to inliers (resp. outliers), we have
\[
(n_{\textrm{in}}+n_{\textrm{out}}) \cdot \Delta F_{\Xb}(r) \geq n_{\textrm{in}} \cdot  \Delta F_{\Xb_{\textrm{in}}}(r) + n_{\textrm{out}}\cdot \Delta F_{\Xb_{\textrm{out}}}(r).
\]
As a result, the robustness of the learning process with respect to the contamination of a ``clean'' training set~$\Xb_{\textrm{in}}$ with outliers will follow from two quantitative bounds: a lower bound $\Delta F_{\Xb_{\textrm{in}}}(r)>0$ for the contribution of inliers, together with an upper bound on the perturbating effects $n_{\textrm{out}}\cdot |\Delta F_{\Xb_{\textrm{out}}}(r)|$ of outliers.

For classical penalty functions $g$ with $g(\mathbf{0}) = 0$, such as sparsity-inducing norms, one easily checks that for any $\Db$ we have
$0 \leq n_{\textrm{out}} \cdot F_{\Xb_{\textrm{out}}}(\Db) \leq \tfrac{1}{2} \|\Xb_{\textrm{out}}\|_{\fro}^{2}$ \citep[see, e.g.,][]{Gribonval:2013tx} hence the upper bound
\begin{equation}
\label{eq:NaiveOutlierBound}
n_{\textrm{out}} \cdot |\Delta F_{\Xb_{\textrm{out}}}(r)| \leq \tfrac{1}{2}  \|\Xb_{\textrm{out}}\|_{\fro}^{2}.
\end{equation}
This implies the robustness to outliers provided that:
\[
\|\Xb_{\textrm{out}}\|_{\fro}^{2} < 2 n_{\textrm{in}} \cdot \Delta F_{\Xb_{\textrm{in}}}(r).
\]
In our context, in the interesting regime we have (with high probability) $\Delta F_{\Xb_{\textrm{in}}}(r) = O(r^{2})$ with $r$ arbitrarily small and $\lambda \asymp r$. Hence, the above analysis suggests that $\|\Xb_{\textrm{out}}\|_{\fro}^{2}/n_{\textrm{in}}$ should scale as $O(r^{2})$: the more ``precision'' we require (the smaller $r$), the least robust with respect to outliers. 

In fact, the considered learning approach is much more robust to outliers that it would seem at first sight: in Section~\ref{sec:robustnesstooutliers}, we establish an improved bound on $n_{\textrm{out}}\cdot |\Delta F_{\Xb_{\textrm{out}}}(r)|$: under the assumption that $\Dbo$ is complete (i.e., $\Dbo$ is a frame with lower frame bound $A^{o}$), we obtain when $\lambda \asymp r$
\begin{equation}
\label{eq:RefinedOutlierBound}
n_{\textrm{out}} \cdot |\Delta F_{\Xb_{\textrm{out}}}(r)| 
 \leq \frac{18 p^{3/2}}{\sqrt{k}}  \|\Xb_{\textrm{out}}\|_{1,2}   \Big( \Exp |\alpha|  \frac{ r \bar{\lambda}}{ (A^{o})^{3/2}}
 \Big),
\end{equation}
where $\|\Xb_{\textrm{out}}\|_{1,2} \defin \sum_{i \in \textrm{out}} \|\signal^{i}\|_{2}$. The upper bound on $n_{\textrm{out}}\cdot |\Delta F_{\Xb_{\textrm{out}}}(r)|$ now scales as $O(r^2)$ when $\lambda \asymp r$, and we have robustness to outliers provided that 
\[
\|\Xb_{\textrm{out}}\|_{1,2} <  n_{\textrm{in}} \cdot  \frac{ \Delta F_{\Xb_{\textrm{in}}}(r) }{r^{2}} \cdot \frac{r}{ \bar{\lambda}}
\cdot \big[
\frac{\sqrt{k}}{18 p^{3/2}}    \frac{(A^o)^{3/2}} { \Exp |\alpha|   } 
\big].
\]
This is now resolution-independent in the regime $\bar{\lambda} \asymp r$.
\subsection{Closed-form expression}
As the reader may have guessed, lower-bounding $\Delta f_{\mathbb{P}}(r)$ is the key technical objective of this paper. 
One of the main difficulties arises from the fact that $f_{\signal}(\Db)$ is only implicitly defined through the minimization of $\Lcal_{\signal}(\Db,\alphab)$ with respect to the coefficients $\alphab$. 

{\em From now on we concentrate on the $\ell^{1}$ penalty, $g(\alphab) = \lambda \|\alphab\|_{1}$.}
 We leverage here a key property of the function $f_\xb$.
Denote by $\alphab^{\star} = \alphab_{\signal}^{\star}(\Db) \in\Real^p$ a solution of problem~(\ref{eq:fi}), that is, the minimization defining $f_\xb$. 
By the convexity of the problem, there always exists such a solution such that, denoting $\J \defin\{ j\in\SET{p};\, \alphab^{\star}_j\neq 0\}$ its support, 
the dictionary $\DbJ \in \Real^{m\times |\J|}$ restricted to the atoms indexed by $\J$ has linearly independent columns 
(hence $\DbJT\DbJ$ is invertible)~\citep{Fuchs2005}. Denoting $\sb^{\star} = \sb_{\signal}^{\star}(\Db) \in \{-1,0,1\}^p$ the sign of $\alphab^{\star}$ and $\J$ its support, $\alphab^{\star}$ has a closed-form expression in terms of $\DbJ$, $\xb$ and $\sb^{\star}$~\citep[see, e.g.,][]{Wainwright2009,Fuchs2005}. This property is appealing in that it makes it possible to obtain a closed-form expression for $f_\xb$,  {\em provided that we can control the sign pattern of $\alphab^{\star}$}.
In light of this remark, it is natural to define:
\begin{definition}\label{def:phi}
Let $\sb \in \{-1,0,1\}^p$ be an arbitrary sign vector and $\J = \J(\sb)$ be its support. For $\xb \in \Real^m$ and $\Db \in \Real^{m\times p}$,
we define 
\begin{equation}
\label{eq:defphi}
\phi_\xb(\Db|\sb) \defin   \inf_{\alphab\in\Real^p, \ \support(\alphab)\subset\J }\tfrac{1}{2}\|\xb - \Db\alphab\|_2^2+\lambda\sb^\top\alphab.
\end{equation}
Whenever $\DbJT \DbJ$ is invertible,  the minimum is achieved at $\hat{\alphab} = \hat{\alphab}_{\signal}(\Db|\sb)$ defined by
\begin{equation}
\label{eq:ClosedFormMinimizer}
\hat{\alphab}_\J = \DbJ^{+} \xb -\lambda \big(\DbJT\DbJ \big)^{-1}  \sb_\J \in \Real^{\J}\quad
\text{and}\quad 
\hat{\alphab}_{\J^c} = \zerob,
\end{equation}
and we have
\begin{equation}\label{eq:phi}
\phi_\xb(\Db|\sb) =  
\frac{1}{2} \big[ \|\xb\|_2^2 - (\DbJT \xb - \lambda \sb_\J)^\top 
(\DbJT \DbJ)^{-1} 
(\DbJT \xb -\lambda \sb_\J ) \big].
\end{equation}
Moreover, if $\sign(\hat{\alphab}) = \sb$, then
\begin{equation}
\begin{split}
\phi_\xb(\Db|\sb) = &  \min_{\alphab\in\Real^p, \ \sign(\alphab)=\sb }\tfrac{1}{2}\|\xb - \Db\alphab\|_2^2+\lambda\sb^\top\alphab 
\\ = &  \min_{\alphab\in\Real^p,\ \sign(\alphab)=\sb } \Lcal_{\xb}(\Db,\alphab) = \Lcal_{\xb}(\Db,\hat{\alphab}).
\end{split}
\end{equation}
\end{definition}
Hence, with $\sb^{\star}$ the sign of a minimizer $\alphab^{\star}$, we have 
 $f_\xb(\Db)  = \phi_\xb(\Db|\sb^{\star})$.
While $\alphab^{\star}$ is unknown, in light of the generative model $\signal = \Dbo \alphabo + \varepsilonb$ for inliers (see Section~\ref{sec:gen_model}), a natural guess for $\sb^\star$ is $\sbo = \sign(\alphabo)$. 

\subsection{Closed form expectation and its lower bound}
Under decorrelation assumptions, one can compute
\begin{equation}
\Delta \phi_\PP(\Db;\Dbo|\sbo) \defin \Exp\ \Delta \phi_{\signal}(\Db;\Dbo|\sbo).
\end{equation}
We use the shorthands $\GramJbo = \GramJb(\Dbo)$, $\ThetaJbo \defin \ThetaJb(\Dbo)$, and $\PJbo \defin \PJb(\Dbo)$.

\begin{proposition}\label{prop:delta_phi}
Assume that both $\LRIP_{k}(\Dbo) <1$ and $\LRIP_{k}(\Db)<1$ so that $\DbJ$ and $\DboJ$ have linearly independent columns for any $\J$ of size $k$. Under Assumption~\ref{assume:basic} we have
 \begin{eqnarray}
\Delta\phi_\PP(\Db;\Dbo|\sbo) 
&=&
\tfrac{\Exp\{\alpha^{2}\}}{2} \cdot\ \Exp_{J} \trace \DboJT  ( \Ib - \PJb ) \DboJ\notag\\
&&- 
\lambda \cdot \Exp\{|\alpha|\} \cdot\ \Exp_{J} \trace \big( \DboJpinv-\DbJ^+ \big) \DboJ \notag\\
&&+
\tfrac{\lambda^{2}}{2}  \cdot\  
\Exp_{J} \trace\left(\ThetaJbo-\ThetaJb\right).
\end{eqnarray}
\end{proposition}
The proof is in Appendix~\ref{sec:expectationphi}. 
In light of this result we switch to the reduced regularization parameter 
 $\bar{\lambda} \defin \frac{\lambda}{\Exp\ |\alpha|}$.
Our main bound leverages Proposition~\ref{prop:delta_phi} and Lemma~\ref{lem:bias_expectation} (Appendix~\ref{app:bounds}). 
\begin{proposition}\label{prop:maindeltaphi}
Consider a dictionary $\Dbo \in \Real^{m \times p}$ such that
\begin{eqnarray}
\label{eq:DefLRIPMax}
\LRIP_{k}(\Dbo) & \leq & \frac{1}{4}\\
\label{eq:DefOpNormMax}
k & \leq & \frac{p}{16(\triple \Dbo\triple_{2}+1)^{2}}.
\end{eqnarray}
Under the basic signal model (Assumption~\ref{assume:basic}):
\begin{itemize}
\item when $\bar{\lambda} \leq 1/4$, for any $r \leq 0.15$ we have, uniformly for all $\Db \in \Scal(r;\Dbo)$:
\begin{equation}\label{eq:MainLowerBound}
\Delta \phi_{\PP}(\Db;\Dbo|\sbo) 
\geq
\frac{\Exp\ \alpha^{2}}{8} \cdot \frac{k}{p} \cdot r \Big(r-r_{\min}(\bar{\lambda})\Big).
\end{equation}
with $r_{\min}(\bar{\lambda})  \defin \tfrac{2}{3} C_{\min} \cdot \bar{\lambda} \cdot \left(1+2\bar{\lambda}\right)$.
\item if in addition
\(
\bar{\lambda}  <  \frac{3}{20 C_{\min}},
\)
then $r_{\min}(\bar{\lambda}) < 0.15$ and the lower bound in~\eqref{eq:MainLowerBound} is non-negative for all $r \in (r_{\min}(\bar{\lambda}),0.15]$.
\end{itemize}
\end{proposition}
The proof is in Appendix~\ref{app:bounds}.

\subsection{Exact recovery}

The analysis of $\Delta \phi_{\PP}(\Db;\Dbo|\sbo)$ would suffice for our needs if the sign of the minimizer $\hat{\alphab}_{\signal}(\Db)$ was guaranteed to always match the ground truth sign $\sbo$.   In fact, if the equality $\sign(\hat{\alphab}_{\signal}(\Db)) = \sbo$ held unconditionally on the radius $r$, then the analysis conducted up to Proposition~\ref{prop:maindeltaphi} would show (assuming a large enough number of training samples) the existence of a local minimum of $F_{\Signal}(\cdot)$ within a ball $\Bcal((1+o(1)) r_{\min})$. Moreover, given the lower bound provided by Proposition~\ref{prop:maindeltaphi}, the {\em global minimum of $F_{\Signal}(\cdot)$ restricted over the ball} $\Bcal((1+o(1))r_{\min})$ would in fact be {\em global over the potentially much larger ball} $\Bcal(0.15)$.

However, with the basic signal model (Assumption~\ref{assume:basic}), the equality $\Delta f_{\signal}(\Db;\Dbo) = \Delta \phi_{\signal}(\Db;\Dbo|\sbo)$ has no reason to hold in general. This motivates the introduction of stronger assumptions involving the cumulative coherence of $\Dbo$ and the bounded signal model (Assumption~\ref{assume:bounded}).
\begin{proposition}[Exact recovery; bounded model]\label{prop:simplified_exact_recovery}
Let $\Dbo$ be a dictionary in $\Real^{m \times p}$ such that 
\begin{equation}
\label{eq:DefCumCoherMax}
\mu_{k}^{o} \defin \mu_{k}(\Dbo) < \frac 12.
\end{equation}
Consider the bounded signal model (Assumption~\ref{assume:bounded}), 
\(
\bar{\lambda} \leq \frac{\loweralpha}{2 \cdot \Exp\ |\alpha|}
\)
and  $r < C_{\max} \cdot \bar{\lambda}$ where
\begin{equation}\label{eq:simplified_exact_recovery}
C_{\max} \defin \frac{2}{7}\cdot \frac{\Exp\ |\alpha|}{M_{\alphab}}  \cdot (1-2\mu_{k}^{o}). 
\end{equation}
If the relative noise level satisfies
\begin{equation}\label{eq:DefMaxRelativeNoiseLevel}
\frac{M_{\varepsilonb}}{M_{\alphab} } < \frac{7}{2} \left(C_{\max} \cdot \bar{\lambda}-r\right),
\end{equation}
then, for $\Db \in \mathcal{D}$ such that $\|\Db-\Dbo\|_{\fro}=r$, $\hat{\alphab}_{\signal}(\Db|\sbo)$ is almost surely the unique minimizer in $\Real^{p}$ of $\alphab \mapsto \frac{1}{2}\|\xb-\Db\alphab\|_2^2+\lambda\|\alphab\|_1$, and we have
\begin{eqnarray}
\sign( \hat{\alphab}_{\signal}(\Db|\sbo)) &=& \sbo\\
f_{\signal}(\Db) &=& \phi_{\signal}(\Db|\sbo)\\
\Delta f_{\signal}(\Db;\Dbo) &=& \Delta \phi_{\signal}(\Db;\Dbo|\sbo).
\end{eqnarray}
\end{proposition}

\subsection{Proof of Theorem~\ref{thm:mainasymp}}

Noticing that $\loweralpha \leq M_{\alphab}$, we let the reader check that assumption~\eqref{eq:DefRangeMax} implies $\tfrac{\loweralpha}{4 \Exp\ | \alpha|} \leq \tfrac{3}{20 C_{\min}}$. Hence, by~\eqref{eq:LambdaMax} we have
\[
\bar{\lambda} < \tfrac{\loweralpha}{4 \Exp\ | \alpha|} \leq \min \left(\frac 14, \frac{3}{20C_{\min}}, \frac{\loweralpha}{2 \cdot \Exp\ |\alpha|} \right),
\]
where we use the inequality $\loweralpha \leq \Exp\ |\alpha|$. 
Assumptions~\eqref{eq:DefCumCoherMaxMainTheorem} and~\eqref{eq:DefKMax} imply~\eqref{eq:DefLRIPMax} and~\eqref{eq:DefOpNormMax}, and we have $\bar{\lambda} \leq \min(\tfrac{1}{4},\tfrac{3}{20 C_{\min}})$, hence we can leverage Proposition~\ref{prop:maindeltaphi}. 
Similarly,  assumption~\eqref{eq:DefCumCoherMaxMainTheorem} implies~\eqref{eq:DefCumCoherMax}, and we have $\bar{\lambda} \leq \tfrac{\loweralpha}{2 \cdot \Exp\ |\alpha|}$, hence we can also apply Proposition~\ref{prop:simplified_exact_recovery}. Furthermore, assumption~\eqref{eq:DefRangeMax} implies $C_{\min}<C_{\max}$, and we have  $\bar{\lambda} \leq \tfrac{1}{4}$, hence $\tfrac{2}{3} C_{\min} \cdot \bar{\lambda} \cdot (1+2\bar{\lambda}) \leq C_{\min} \cdot \bar{\lambda }< C_{\max} \cdot \bar{\lambda}$. Finally, the fact that $\bar{\lambda} \leq \tfrac{\loweralpha}{2 \cdot \Exp\ |\alpha|}$ further implies $C_{\max} \cdot \bar{\lambda} \leq 0.15$. Putting the pieces together,  we have $\tfrac{2}{3} C_{\min} \cdot \bar{\lambda} \cdot (1+2\bar{\lambda})  \leq C_{\min} \cdot \bar{\lambda} < C_{\max} \cdot \bar{\lambda} \leq 0.15$, and for any $r \in  \left(C_{\min} \cdot \bar{\lambda}  , C_{\max} \cdot \bar{\lambda}\right)$ we obtain
\begin{equation}\label{eq:asymptoticbound}
\Delta f_{\PP}(r) \geq \frac{\Exp\ \alpha^{2}}{8} \cdot \frac{k}{p} \cdot r \Big(r-C_{\min} \cdot \bar{\lambda} \Big) > 0.
\end{equation}
as soon as the relative noise level satisfies~\eqref{eq:DefMaxRelativeNoiseLevel}. 

\subsection{Proof of  Theorem~\ref{thm:mainfinite}}\label{sec:proofthmmainfinite}
In order to prove  Theorem~\ref{thm:mainfinite}, we need to control the deviation of the average of functions $\Delta\phi_{\signal^i}(\Db;\Dbo|\sbo)$ around its expectation, uniformly in the ball $\{ \Db, \ \| \Db - \Dbo \|_F \leq r \}$.

\subsubsection{Review of Rademacher averages.}
We first review results on Rademacher averages. Let $\mathcal{F}$ be a set of measurable functions on a measurable set $\mathcal{X}$, and $n$ i.i.d.~random variables $X_1,\dots,X_n$, in $\mathcal{X}$. We assume that all functions are bounded by $B$ (i.e., $|f(X)| \leq B$ almost surely). Using usual symmetrisation arguments~\cite[Sec.~9.3]{boucheron2013concentration}, we get
\begin{equation*}\begin{split}
\EE_{X} \sup_{f \in \mathcal{F}} \bigg( \frac{1}{n} \sum_{i=1}^n f(X_i) - \EE_{X} f(X) \bigg)\\
\leq 2 \EE_{X,\varepsilon}
\sup_{ f \in \mathcal{F}} \bigg( \frac{1}{n} \sum_{i=1}^n \varepsilon_i f(X_i) \bigg),
\end{split}
\end{equation*}
where $\varepsilon_{i}, 1 \leq i \leq n$ are independent Rademacher random variables, i.e., with values $1$ and $-1$ with equal probability $\tfrac{1}{2}$. Conditioning on the data $X_1,\dots,X_n$, the function $\varepsilon \in \RR^{n} \mapsto 
\sup_{ f \in \mathcal{F}} \big( \frac{1}{n} \sum_{i=1}^n \varepsilon_i f(X_i) \big)$ is convex. Therefore, if $\eta$ is an independent standard normal vector, by Jensen's inequality, using that the normal distribution is symmetric and $\EE|\eta_i| = \sqrt{2/\pi}$, we get
\begin{equation*}
\begin{split}
\EE_{X,\varepsilon}
\sup_{ f \in \mathcal{F}} & \bigg( \frac{1}{n} \sum_{i=1}^n \varepsilon_i f(X_i) \bigg)\\
&= \sqrt{\pi/2} \cdot
\EE_{X,\varepsilon}
\sup_{ f \in \mathcal{F}} \bigg( \frac{1}{n} \sum_{i=1}^n \varepsilon_i  \EE |\eta_i| f(X_i) \bigg)\\
& \leq  \sqrt{\pi/2} \cdot
\EE_{X,\eta}
 \sup_{ f \in \mathcal{F}} \bigg( \frac{1}{n} \sum_{i=1}^n \eta_i f(X_i) \bigg).
\end{split}
\end{equation*}
Moreover, the random variable $Z = \sup_{ f \in \mathcal{F}} \big( \frac{1}{n} \sum_{i=1}^n (f(X_i) -\EE f(X))\big)$
only changes by at most $2B/n$ when changing a single of the $n$ random variables. Therefore, by Mac Diarmid's inequality, we obtain with probability at least $1-e^{-x}$:
$
Z \leq \EE Z + B \sqrt{\frac{2x}{n}}.
$
We may thus combine all of the above, to get, with probability at least $1-e^{-x}$:
\begin{equation}
\label{eq:rademacher}
\begin{split}
 \sup_{f \in \mathcal{F}} &\bigg( \frac{1}{n} \sum_{i=1}^n f(X_i) - \EE f(X) \bigg)\\
& \leq 2 \sqrt{\pi/2} \cdot
\EE_{X,\eta}
 \sup_{ f \in \mathcal{F}} \bigg( \frac{1}{n} \sum_{i=1}^n \eta_i f(X_i) \bigg)+ B \sqrt{\frac{2x}{n}}.
 \end{split}
\end{equation}
Note that in the equation above, we may also consider the absolute value of the deviation by redefining $\mathcal{F}$ as $\mathcal{F} \cup ( - \mathcal{F})$.

We may now prove two lemmas that will prove useful in our uniform deviation bound.

\begin{lemma}[Concentration of a real-valued function on matrices $\Db$]
\label{lemma:concentration}
If $h_1,\dots,h_n$ are random real-valued i.i.d.~functions on $\{ \Db, \ \| \Db - \Dbo \|_F \leq r \}$, such that they are almost surely bounded by $B$ on this set, as well as, $R$-Lipschitz-continuous (with respect to the Frobenius norm). Then, with probability greater than $1-e^{-x}$:
$$
\sup_{ \| \Db - \Dbo \|_F \leq r } \bigg| \tfrac{1}{n} \sum_{i=1}^n h_i(\Db) -   \EE h(\Db) \bigg| \leq  4 \sqrt{\tfrac{\pi}{2}} 
\tfrac{R r  \sqrt{mp}}{\sqrt{n}}   + B \sqrt{\tfrac{2x}{n}}.
$$
\end{lemma}
\begin{proof}
Given Eq.~(\ref{eq:rademacher}), we only need to provide an upper-bound on
$\EE
\sup_{ \| \Db - \Dbo \|_F \leq r }\big| \frac{1}{n} \sum_{i=1}^n \eta_i h_i(\Db)  \big|$ for $\eta$ a standard normal vector.
Conditioning on the draw of functions $h_{1},\ldots,h_{n}$, consider two Gaussian processes, indexed by $\Db$, $A_{\Db} =  \frac{1}{n} \sum_{i=1}^n \eta_i h_i(\Db)$
and $C_{\Db} = \frac{R}{\sqrt{n}} \sum_{i=1}^m \sum_{j=1}^p \zeta_{ij} ( \Db - \Dbo)_{ij}$, where $\eta$ and $\zeta$ are standard Gaussian vectors. We have, for all $\Db$ and $\Db'$,
$
\EE | A_{\Db} - A_{\Db'} |^2  \leq \frac{R^{2}}{n} \| \Db - \Db' \|_F^2 = \EE | C_{\Db} - C_{\Db'} |^2 
$.

Hence, by Slepian's lemma~\cite[Sec.~3.3]{Massart2003}, $\EE \sup_{ \| \Db - \Dbo \|_F \leq r }  A_{\Db}
\leq
\EE \sup_{ \| \Db - \Dbo \|_F \leq r } C_{\Db} = \frac{R r}{\sqrt{n}} \EE \| \zeta \|_F \leq \frac{R r  \sqrt{mp}}{\sqrt{n}} $.
Thus, by applying the above reasoning to the functions $h_i$ and $-h_i$ and taking the expectation with respect to the draw of $h_{1},\ldots,h_{n}$, we get: 
$\EE
\sup_{ \| \Db - \Dbo \|_F \leq r }\big| \frac{1}{n} \sum_{i=1}^n \eta_i h_i(\Db)  \big| \leq 
2  \frac{R r  \sqrt{mp}}{\sqrt{n}},$
hence the result.
 \end{proof}
 
 \begin{lemma}[Concentration of matrix-valued function on matrices $\Db$]
\label{lemma:concentrationmat}
Consider $g_1,\dots,g_n$ random i.i.d.~functions on $\{ \Db, \ \| \Db - \Dbo \|_F \leq r \}$, 
with values in real symmetric matrices of size $s$. Assume that these functions are almost surely bounded by $B$ (in operator norm) on this set, as well as, $R$-Lipschitz-continuous
 (with respect to the Frobenius norm, i.e., $\triple g_{i}(\Db)\triple_{2} \leq B$ and $\triple g_i(\Db) - g_i(\Db') \triple_2 \leq R \| \Db - \Db' \|_F$). Then, with probability greater than $1-e^{-x}$:
\begin{equation*}
\begin{split}
\sup_{ \| \Db - \Dbo \|_F \leq r } & 
\bigg\triple \frac{1}{n} \sum_{i=1}^n g_i(\Db) -   \EE g(\Db) \bigg\triple_{2}\\
& \leq  4 \sqrt{\pi/2} 
\bigg(  \frac{\sqrt{2 mp } R r }{\sqrt{n}}   +   \frac{ B \sqrt{8 s } }{\sqrt{n}} \bigg)   + B \sqrt{\frac{2x}{n}}.
\end{split}
\end{equation*}
\end{lemma}
\begin{proof}
For any symmetric matrix $\mathbf{M}$, $\triple \mathbf{M} \triple_2 = \sup_{\| \zb \|_2 \leq 1} | \zb^\top \mathbf{M} \zb |$. Given Eq.~(\ref{eq:rademacher}), we only need  to upper-bound 
\begin{equation*}
\begin{split}
\EE
\sup_{ \| \Db - \Dbo \|_F \leq r } & \big\triple \frac{1}{n} \sum_{i=1}^n \eta_i g_i(\Db)  \big\triple_2\\
& =
\EE
\sup_{ \| \Db - \Dbo \|_F \leq r, \|\zb\|_{2} \leq 1}\big| \frac{1}{n} \sum_{i=1}^n \eta_i \zb^{\top}g_i(\Db)\zb  \big|,
\end{split}
\end{equation*}
 for $\eta$ a standard normal vector. We thus consider two Gaussian processes, indexed by $\Db$ and $\| \zb\|_2 \leq 1$, $A_{\Db,\zb} =  \frac{1}{n} \sum_{i=1}^n \eta_i \zb^\top g_i(\Db) \zb$
and $C_{\Db,\zb} = \frac{\sqrt{2} R}{\sqrt{n}} \sum_{i=1}^m \sum_{j=1}^p \zeta_{ij} ( \Db - \Dbo)_{ij} +
\frac{2B \sqrt{2} }{\sqrt{n}} \sum_{i=1}^s \xi_i \zb_i$, where $\eta$ and $\zeta$ are, again, standard normal vectors. 
We have, for all $(\Db,\zb)$ and $(\Db',\zb')$,
\begin{equation*}
\begin{split}
\EE & | A_{\Db,\zb} - A_{\Db',\zb'} |^2  \\
& \leq \tfrac{1}{n} \big(
R \| \Db - \Db'\|_F + \big| \zb^\top g_i(\Db) \zb -  (\zb')^\top g_i(\Db) \zb' \big|
\big)^2 \\
& \leq \tfrac{1}{n} \big(
R \| \Db - \Db'\|_F + 2 B \| \zb - \zb'\|_2
\big)^2 \\
& \leq \tfrac{2}{n} R^2 \| \Db - \Db'\|_F^2 + \tfrac{8B^2}{n}   \| \zb - \zb'\|_2^2
 =
\EE | C_{\Db,\zb} - C_{\Db',\zb'} |^2  .
\end{split}
\end{equation*}

Applying Slepian's lemma to $A_{\Db,\zb}$ and to $-A_{\Db,\zb}$, we get
\begin{equation*}
\begin{split}
 \EE &
\sup_{ \| \Db - \Dbo \|_F \leq r }\bigg\triple\tfrac{1}{n} \sum_{i=1}^n \eta_i g_i(\Db)  \bigg\triple_{  2} \\ 
&\leq 2  \tfrac{\sqrt{2} Rr}{\sqrt{n}} \EE \| \zeta\|_F +2  \tfrac{2B \sqrt{2} }{\sqrt{n}} \EE \| \xi\|_2\\
&\leq    \tfrac{\sqrt{8 mp } Rr}{\sqrt{n}}   +   \tfrac{ B \sqrt{32 s } }{\sqrt{n}} ,
\end{split}
\end{equation*}

hence the result.
\end{proof}

\subsubsection{Decomposition of $\Delta\phi_{\signal}(\Db;\Dbo|\sbo)$.}
Our goal is to uniformly bound the deviations of $\Db \mapsto \Delta\phi_{\signal}(\Db;\Dbo|\sbo)$ from its expectation on $\Scal(\Dbo;r)$. 
With the notations of Appendix~\ref{sec:boundphi}, we use the following decomposition
\begin{equation*}
\begin{split}
\Delta\phi_{\signal}(\Db;\Dbo|\sbo)
= & \big[ \Delta\phi_{\signal}(\Db;\Dbo|\sbo) - \Delta\phi_{\alphab,\alphab}(\Db;\Dbo)  \big]\\
& + \Delta\phi_{\alphab,\alphab}(\Db;\Dbo)\\
=& h(\Db) + \Delta\phi_{\alphab,\alphab}(\Db;\Dbo),
\end{split}
\end{equation*}
with $\Delta\phi_{\alphab,\alphab}(\Db;\Dbo)
:= \frac{1}{2} \alphaboT  \DboT  (\Ib - \PJb) \Dbo \alphabo$
and $h(\Db) := \big( \Delta\phi_{\signal}(\Db;\Dbo|\sbo) - \Delta\phi_{\alphab,\alphab}(\Db;\Dbo)  \big) $. 

For the first term,  by Lemma~\ref{le:LipschitzAndBoundDeltaPhi} in Appendix~\ref{sec:boundphi}, the function $h$ on $\Bcal(\Dbo;r)$ is almost surely $L$-Lipschitz-continuous with respect to the Frobenius metric and almost surely bounded by $c = Lr$, where we denote \[\sqrt{1-\LRIP} \defin \sqrt{1-\LRIP_{k}(\Dbo)}-r > 0\] and
\begin{equation*}
\begin{split}
L \defin  \tfrac{1}{\sqrt{1-\LRIP}} & \cdot \left(M_{\varepsilonb} + \tfrac{\lambda \sqrt{k}}{\sqrt{1-\LRIP}}\right) \\
& \cdot 
\left(
2\sqrt{1+\URIP_{k}(\Dbo)}M_{\alphab} + M_{\varepsilonb} + \tfrac{\lambda \sqrt{k}}{\sqrt{1-\LRIP}}
\right) .
\end{split}
\end{equation*}
We can thus apply  Lemma~\ref{lemma:concentration}, with $B=c=Lr$ and $R=L$. 

Regarding the second term, since $(\Ib-\PJb) \Db \alphabo = (\Ib-\PJb) \DbJ \alphabo_{\J} = 0$, one can rewrite it as 
\begin{equation*}
\begin{split}
\Delta &\phi_{\alphab,\alphab}(\Db;\Dbo)\\
&= \frac{1}{2} \alphaboT  ( \Db - \Dbo)^\top  (\Ib - \PJb) ( \Db - \Dbo) \alphabo\\
&= \frac{1}{2} {\rm vec}( \Db - \Dbo)^\top
\big\{
\alphabo \alphaboT \otimes    (\Ib - \PJb) 
\big\}
 {\rm vec}( \Db - \Dbo).
\end{split}
\end{equation*}
where $ \mathbf{A} \otimes \mathbf{B} $ denotes the Kronecker product between two matrices (see, e.g.,~\cite{Horn1990}).
Thus, with $g(\Db) = \alphabo \alphaboT \otimes    (\Ib - \PJb) $ a random matrix-valued function with $s = mp$, we have an upper-bound of $B' = M_\alphab^2$ (as the eigenvalues of $ \mathbf{A} \otimes \mathbf{B} $ are products of eigenvalues of $\mathbf{A}$ and eigenvalues of $\mathbf{B}$~\citep{Horn1990}) and, by Lemma~\ref{lem:RIPBounds} and Lemma~\ref{le:LipBoundsRIP} in Appendix~\ref{sec:boundphi}, a Lipschitz-constant $R' =  M_\alphab^2   ( 1- \LRIP)^{-1/2}$.
We may thus apply Lemma~\ref{lemma:concentrationmat} to show that uniformly, the deviation of
$\Delta\phi_{\alphab,\alphab}(\Db;\Dbo) $ are bounded by $\|{\rm vec} (\Db-\Dbo)\|_{2}^{2} =  r^2 $ times the deviations of $g(\Db)$ in operator norm.

We thus get, with probability  at least $ 1 - 2 e^{-x}$, deviations from the expectations upper-bounded by:
\begin{equation*}
\begin{split}
4 \sqrt{\tfrac{\pi}{2}} &
\tfrac{L r  \sqrt{mp}}{\sqrt{n}}   
+ Lr \sqrt{\tfrac{2x}{n}}\\
& + 
r^2 \Bigg(
4 \sqrt{\tfrac{\pi}{2}} 
\bigg(  \tfrac{\sqrt{2 mp }   }{\sqrt{n}} \tfrac{M_\alphab^2 r}{\sqrt{1-\LRIP}}   +   \tfrac{ M_\alphab^2 \sqrt{8 mp } }{\sqrt{n}} \bigg) +   M_\alphab^2 \sqrt{\tfrac{2x}{n}}
\Bigg),
\end{split}
\end{equation*}
We notice that $R'r = M_{\alphab}^{2} r/\sqrt{1-\LRIP} < B'$ since $r < \sqrt{1-\LRIP}$, hence this is less than $\beta r \sqrt{\frac{2x}{n}} + \beta' r \sqrt{\frac{mp}{n}}$ with
\[
\beta \defin L + r M_{\alphab}^{2},\qquad
\beta'  \defin 4 \sqrt{\frac{\pi}{2}} \left(L + 3\sqrt{2} r M_{\alphab}^{2}\right) \leq 12 \sqrt{\pi} \beta
\]
Overall, with probability at least $1-2e^{-x}$, the deviations of $\Db \mapsto \Delta\phi_{\signal}(\Db;\Dbo|\sbo)$ from its expectation on $\Scal(\Dbo;r)$ are uniformly bounded by $\eta_{n} \defin r (L+M_{\alphab}^{2}r) \left(\sqrt{2x/n}+12\sqrt{\pi mp/n}\right)$.

\subsubsection{Sample complexity.}
As briefly outlined in Section~\ref{sec:sample_complexity}, with $n_{\textrm{in}}$ inliers, the existence of a local minimum of $F_{\Signal}(\cdot)$ within a radius~$r$ around $\Dbo$ is guaranteed with probability at least $1-2e^{-x}$ as soon as $2\eta_{n_{\textrm{in}}} < \Delta f_{\PP}(r)$. Combining with the asymptotic lower bound~\eqref{eq:asymptoticbound} and the above refined uniform control over $\Scal(\Dbo;r)$, $\eta_{n}$ , it is sufficient to have
\[
2r(L+M_{\alphab}^{2}r) \cdot \left(\sqrt{\tfrac{2x}{n_{\textrm{in}}}}+   12\sqrt{\tfrac{\pi mp}{n_{\textrm{in}}}}\right)
 <  \tfrac{\Exp\ \alpha^{2}}{8} \cdot \tfrac{k}{p} \cdot r \Big(r-C_{\min} \cdot \bar{\lambda}  \Big).
\]
i.e.
\begin{equation}\label{eq:Th2-1}
 {n_{\textrm{in}}} \geq  
 \bigg(
\sqrt{2x} + 12 \sqrt{\pi mp}
 \bigg)^{2}
 \cdot \left(\tfrac{16}{\Exp\ \alpha^{2}} \cdot \tfrac{p}{k} \cdot \tfrac{L+M_{\alphab}^{2}r}{\Big(r-C_{\min} \cdot \bar{\lambda}\Big) }\right)^{2}
\end{equation}
By~\eqref{eq:DefCumCoherMaxMainTheorem} we have $\max(\LRIP_{k}(\Dbo),\URIP_{k}(\Dbo)) \leq 1/4$, hence $2\sqrt{1+\URIP_{k}(\Dbo)} \leq \sqrt{5}$. Moreover, since $r < C_{\max} \cdot \bar{\lambda} \leq 0.15$, we have
\(
\sqrt{1-\LRIP} = \sqrt{1-\LRIP_{k}(\Dbo)}-r \geq \sqrt{3/4}-0.15 \geq \sqrt{1/2}.
\)
As a result 
\begin{equation*}
\begin{split}
L & \leq \sqrt{2} (M_{\varepsilonb}+\lambda \sqrt{2k}) \cdot \left( \sqrt{5} M_{\alphab} + M_{\varepsilonb} + \lambda \sqrt{2k} \right)\\
& = \sqrt{10} M_{\alphab} (M_{\varepsilonb}+\lambda \sqrt{2k}) + \sqrt{2} (M_{\varepsilonb}+\lambda \sqrt{2k})^{2}
\end{split}
\end{equation*}
Further, since
\(
\lambda \sqrt{2k} 
= \bar{\lambda} \Exp |\alpha| \sqrt{2k} 
= \bar{\lambda} \sqrt{2/k} \Exp \|\alphab\|_{1}  
\leq \bar{\lambda} \sqrt{2/k} \Exp \sqrt{k}\|\alphab\|_{2}  
\leq \bar{\lambda} \sqrt{2} M_{\alphab},
\)
we have
\begin{equation}\label{eq:Th2-2}
L+M_{\alphab}^{2}r
\leq \sqrt{20} M_{\alphab}^{2} \cdot \left( 
r
+ \frac{M_{\varepsilonb}}{M_{\alphab}} + \bar{\lambda}
+ \left(\frac{M_{\varepsilonb}}{M_{\alphab}} + \bar{\lambda}\right)^{2}\right).
\end{equation}
Eqs~\eqref{eq:Th2-1} and~\eqref{eq:Th2-2} with the bound $(\sqrt{2x}+12\sqrt{\pi mp})^{2} \lesssim mp+x$ yield our main sample complexity result~\eqref{eq:mainsamplecomplexity}.

\subsubsection{Robustness to outliers.} \label{sec:robustnesstooutliers}
In the presence of outliers, we obtain the naive robustness to outliers~\eqref{eq:mainoutliers} in Theorem~\ref{thm:mainfinite} using the reasoning sketched in Section~\ref{sec:outliers}. with the naive bound~\eqref{eq:NaiveOutlierBound}. Obtaining the ``resolution independent'' robustness result~\eqref{eq:mainoutliersbis} requires refining the estimate of the impact of outliers on the cost function $F_{\Signal}(\Db)$ by gaining two factors: one factor $O(r)$ (thanks to a Lipschitz property), and one factor $O(\lambda)$ (thanks to the completeness of the dictionary). 

\paragraph{Gaining a first factor $O(r)$ using a Lipschitz property.}
The arguments of \cite[Lemma 3 and Corollary 2]{Gribonval:2013tx} can be straightforwardly adapted to show that for any signal $\signal$, the function $\Db \mapsto f_{\signal}(\Db)$ is uniformly locally Lipschitz on the convex ball $\{\Db \in \RR^{m \times p}: \|\Db-\Dbo\|_{\fro} \leq r\}$ (not restricted to normalized dictionaries). Its Lipschitz constant is bounded by $L_{\signal}(r) \defin \sup_{\|\Db-\Dbo\|_{\fro} \leq r} L_{\signal}(\Db)$ with
\(
L_{\signal}(\Db) \defin \|\alphab\|_{2} \cdot \|\signal-\Db\alphab\|_{2}.
\)
where we denote $\alphab = \alphab_{\signal}(\Db)$ a coefficient vector minimizing $\Lcal_{\signal}(\Db,\alphab)$. It follows that
\[
n_{\textrm{out}} |\Delta F_{\Signal_{\textrm{out}}}(r)|
 \leq \left(\sum_{i \in \textrm{out}} L_{\signal^{i}}(r)\right) \cdot r.
\]
Compared to the naive bound~\eqref{eq:NaiveOutlierBound}, we already gained a first factor $r$, provided we uniformly bound the Lipschitz constants $L_{\signal}(r)$. 
\paragraph{Gaining a second factor $O(\lambda)$ under a completeness assumption.}
Introducing
\[
C(\Db) \defin \sup_{\mathbf{u}\neq 0, \mathbf{u} \in \textrm{span}(\Db)} \inf_{\betab: \Db\betab=\mathbf{u}} \frac{\|\betab\|_{1}}{\|\mathbf{u}\|_{2}},
\]
we first show that
\(
\|\alphab\|_{2} \leq C(\Db) \cdot \|\signal\|_{2}.
\)
Indeed, denoting $P$ the orthonormal projection onto $\textrm{span}(\Db)$, by definition of $\alphab$ we have, for any signal $\signal$ and any coefficient vector $\betab$,
\begin{eqnarray}
\tfrac{1}{2}\|\signal-P\signal\|_{2}^{2} &+& \tfrac{1}{2} \|P\signal-\Db\alpha\|_{2}^{2} + \lambda \|\alphab\|_{1}\notag\\
&=&
\tfrac{1}{2} \|\signal-\Db\alpha\|_{2}^{2} + \lambda \|\alphab\|_{1}\notag\\
& \leq & \tfrac{1}{2}\|\signal-\Db\betab\|_{2}^{2} + \lambda \|\betab\|_{1}\label{eq:BetavsAlpha}\\
&=& \tfrac{1}{2}\|\signal-P\signal\|_{2}^{2} + \tfrac{1}{2} \|P\signal-\Db\betab\|_{2}^{2} + \lambda \|\betab\|_{1}.\notag
\end{eqnarray}
Specializing to the minimum $\ell^{1}$ norm vector $\betab$ such that $\Db \betab = P\signal$ yields
\[
\|\alphab\|_{2} \leq \|\alphab\|_{1} \leq \|\betab\|_{1} \leq C(\Db) \cdot \|P\signal\|_{2} \leq C(\Db) \cdot \|\signal\|_{2}.
\]
To complete the control of $L_{\signal}(\Db) = \|\alphab\|_{2} \cdot \|\signal-\Db\alphab\|_{2}$ we now bound $\|\signal-\Db\alphab\|_{2}$. 
A first approach that does not require any further assumption on $\Db$ consists in specializing~\eqref{eq:BetavsAlpha} to $\betab = 0$, yielding $\|\signal-\Db\alphab\|_{2} \leq \|\signal\|_{2}$,  $L_{\signal}(\Db) \leq C(\Db) \cdot \|\signal\|_{2}^{2}$, and finally
\[
n_{\textrm{out}} |\Delta F_{\Signal_{\textrm{out}}}(r)| \leq C(r) \cdot \|\Signal_{\textrm{out}}\|_{\fro}^{2} \cdot r,
\]
with
\[
C(r) \defin \sup_{\|\Db-\Dbo\|_{\fro} \leq r} C(\Db).
\]
However, as the reader may have noticed, this still lacks one $O(r)$ factor for our needs. This is obtained in the regime of interest $\lambda \asymp r$ under the assumption that $\Db$ is {\em complete} ($\textrm{span}(\Db) = \RR^{m}$). In this case we introduce 
\[
C'(\Db) \defin \sup_{\mathbf{u} \neq 0} \frac{\|\mathbf{u}\|_{2}}{\|\Db^{\top} \mathbf{u}\|_{\infty}} < \infty
\]
and
\[
C'(r) \defin \sup_{\|\Db-\Dbo\|_{\fro} \leq r} C'(\Db).
\]
By the well known optimality conditions for the $\ell^{1}$ regression problem, $\alphab = \hat{\alphab}_{\signal}(\Db)$ satisfies $\|\Db^{\top}(\signal-\Db \alphab)\|_{\infty} = \lambda$, hence
\begin{eqnarray*}
\|\signal-\Db\alphab\|_{2} & \leq & C'(\Db) \cdot \|\Db^{\top}(\signal-\Db \alphab)\|_{\infty} \leq C'(\Db) \cdot \lambda.
\end{eqnarray*}
Overall we get $L_{\signal}(\Db) \leq \lambda \cdot C(\Db)\cdot C'(\Db) \cdot \|\signal\|_{2}$ and eventually
\begin{equation}\label{eq:Th2-3}
n_{\textrm{out}} | \Delta F_{\Signal_{\textrm{out}}}(r)| \leq  \sum_{i \in \textrm{out}} \|\signal^{i}\|_{2} \cdot \Exp |\alpha| \cdot C(r)\cdot C'(r) \cdot r \bar{\lambda}.
\end{equation}
To conclude, we now bound $C(r)$ and $C'(r)$. 
Note that  as soon as $\Db\betab = \mathbf{u}$, since $\|\mathbf{u}\|_{2}^{2} = \langle \betab, \Db^{\top} \mathbf{u}\rangle  \leq \|\betab\|_{1} \|\Db^{\top}\mathbf{u}\|_{\infty}$, we have $C'(\Db) \leq C(\Db)$.

\begin{lemma}
Assume $\Db \in \RR^{m \times p}$ is a frame with lower frame bound $A$ such that $A \|\signal\|_{2}^{2} \leq \|\Db^{\top}\signal\|_{2}^{2}$ for any signal $\signal$. Then $C'(\Db) \leq \sqrt{p/A}$. If in addition, $\LRIP_{k}(\Db) < 1$ then
\[
C(\Db) \leq \frac{2}{A} \cdot \frac{p}{\sqrt{k}} \cdot \frac{1+\URIP_{k}(\Db)}{\sqrt{1-\LRIP_{k}(\Db)}}.
\]
\end{lemma}
\begin{proof}
For any $\signal$ we have $\|\Db^{\top}\signal\|_{\infty}^{2} \geq \|\Db^{\top}\signal\|_{2}^{2}/p \geq A \|\signal\|_{2}^{2}/p$ hence the bound on $C'(\Db)$.
To bound $C(\Db)$ we define $P_{T}$ the orthoprojector onto $\textrm{span}(\Db_{T})$ where $T \subset \SET{p}$, $\rb_{0} = \signal$ and for $i \geq 1$
\begin{eqnarray*}
T_{i} & = & \arg\max_{|T| \leq k} \|P_{T} \rb_{i-1}\|_{2}\\
\rb_{i} & = & \rb_{i-1}-P_{T_{i}} \rb_{i-1}\\
\alphab_{i} & s.t. & P_{T_{i}} \rb_{i-1} = \Db_{T_{i}} \alphab_{i}.
\end{eqnarray*}
We notice that for any $\rb$ 
\begin{equation*}
\begin{split}
\sup_{|T| \leq k}  \|P_{T}\rb\|_{2}^{2} 
& \geq \sup_{|T|\leq k} \frac{\|\Db^{\top}_{T}\rb\|_{2}^{2}}{1+\URIP_{k}}
   \geq \frac{1}{1+\URIP_{k}} \cdot \frac{k}{p} \|\Db^{\top} \rb\|_{2}^{2}\\
& \geq \frac{A \ell}{(1+\URIP_{k}) p} \|\rb\|_{2}^{2} =: \gamma^{2} \|\rb\|_{2}^{2}.
\end{split}
\end{equation*}
As a result for any $i \geq 1$, $\|\rb_{i}\|_{2}^{2} = \|\rb_{i-1}\|_{2}^{2}- \|P_{T_{i}}\rb_{i-1}\|_{2}^{2} \leq (1-\gamma^{2}) \|\rb_{i-1}\|_{2}^{2}$ hence by induction $\|\rb_{i}\|_{2}^{2} \leq (1-\gamma^{2})^{i} \cdot \|\signal\|_{2}^{2}$.
This implies
\begin{equation*}
\begin{split}
\|\alphab_{i}\|_{1}  \leq \sqrt{k} \|\alphab_{i}\|_{2} 
& \leq \sqrt{\tfrac{k}{1-\LRIP_{k}}} \|\Db_{T_{i}}\alphab_{i}\|_{2}
 \leq \sqrt{\tfrac{k}{1-\LRIP_{k}}} \|\rb_{i-1}\|_{2}\\
&\leq \sqrt{\tfrac{k}{1-\LRIP_{k}}} (\sqrt{1-\gamma^{2}})^{i-1} \|\signal\|_{2}
\end{split}
\end{equation*}
Denoting $\alphab = \sum_{i\geq 1} \alphab_{i}$ we have $\signal = \Db \alphab$ and 
\begin{eqnarray*}
\|\alphab\|_{1} 
\leq \sum_{i \geq 1} \|\alphab_{i}\|_{1} 
&\leq& \sqrt{\tfrac{k}{1-\LRIP_{k}}} \cdot  \|\signal\|_{2} \cdot \sum_{i \geq 1} (\sqrt{1-\gamma^{2}})^{i-1}\\
 &=& \sqrt{\tfrac{k}{1-\LRIP_{k}}} \cdot  \|\signal\|_{2} \cdot \frac{1}{1-\sqrt{1-\gamma^{2}}}\\
&=& 
 \sqrt{\tfrac{k}{1-\LRIP_{k}}} \cdot  \|\signal\|_{2} \cdot \frac{1+\sqrt{1-\gamma^{2}}}{\gamma^{2}}\\
& \leq & 
 \sqrt{\tfrac{k}{1-\LRIP_{k}}} \cdot  \|\signal\|_{2} \cdot \frac{2}{\gamma^{2}}.
\end{eqnarray*}
\end{proof}

We may now provide a control on both  $C'(r)$ and $C(r)$.

\begin{corollary}\label{cor:usinglowerframebound}
Assume $\Dbo \in \RR^{m \times p}$ is a frame with lower frame bound $A^{o}$ such that $A^{o} \|\signal\|_{2}^{2} \leq \|(\Dbo)^{\top}\signal\|_{2}^{2}$ for any signal $\signal$,  and $\max \{ \LRIP_{k}(\Dbo), \URIP_{k} (\Dbo) \} \leq \frac{1}{4}$. Consider $r \leq \min \{ \sqrt{A^{o}}/2, \sqrt{ 1- \LRIP_{k}(\Dbo) }  \}$ and let $\URIP = \URIP(r) \defin    ( \sqrt{ 1 +\URIP_{k}(\Dbo) } + r )^2 - 1$ and $\LRIP= \LRIP(r) \defin 1 -  ( \sqrt{ 1- \LRIP_{k}(\Dbo) } - r )^2$.  Then, for any $\Db$ such that $\| \Db - \Dbo \|_F \leq r$,  $C'(\Db) \leq \frac{8}{A^{o}} \cdot \frac{p}{\sqrt{k}} \cdot \frac{1+\URIP }{\sqrt{1-\LRIP }}$ and $C(\Db) \leq  \sqrt{ 4 p / A^{o}}$.
\end{corollary}
\begin{proof}
From the proof of Lemma~\ref{lem:RIPBounds}, for any $\Db$ such that $\| \Db - \Dbo\|_F \leq r$, we have $ \LRIP_{k}(\Db) \leq \LRIP$. Using a similar reasoning, we get:
 $ \URIP_{k}(\Db) \leq \URIP$.
 Moreover, 
using the triangular inequality, we have
\[
\|\Db^{\top}\signal\|_{2} \geq \|\DboT\signal\|_{2}-\|(\Dbo-\Db)^{\top}\signal\|_{2} \geq \sqrt{A^{o}} \|\signal\|_{2} - r \|\signal\|_{2},
\]
and thus with $ A = A^{o}/4$ and $r\leq \sqrt{A^{o}}/2$, $\Dbo$ is a frame with lower frame bound $A$. We may thus apply the lemma above, to obtain the desired results.
\end{proof}

\paragraph{Summary.} With the assumptions of Theorem~\ref{thm:mainfinite},  we thus obtain from~\eqref{eq:Th2-3} and Corollary~\ref{cor:usinglowerframebound} the following bound:
\begin{equation}
n_{\textrm{out}} | \Delta F_{\Signal_{\textrm{out}}}(r)| \leq  \|\Xb_{\textrm{out}}\|_{1,2}   \cdot \Exp |\alpha|  \cdot \frac{8}{A^{o}} \cdot \frac{p}{\sqrt{k}} \cdot \frac{1+\URIP }{\sqrt{1-\LRIP }} \cdot    \sqrt{  \frac{4 p }{ A^{o}}}  \cdot r \bar{\lambda},
\end{equation}
where $\|\Xb_{\textrm{out}}\|_{1,2} \defin \sum_{i \in \textrm{out}} \|\signal^{i}\|_{2}$. Assumption~\eqref{eq:DefCumCoherMaxMainTheorem} implies 
$\RIP_{k}(\Dbo) \leq 1/4$, and the other assumptions of Theorem~\ref{thm:mainfinite} imply $r  < 0.15$. It follows that 
\[
8\frac{1+\URIP }{\sqrt{1-\LRIP }} \leq8 \frac{(\sqrt{1+1/4}+0.15)^{2}}{\sqrt{1-1/4}-0.15} \leq 18.
\]
With this refined bound, we obtain the ``resolution independent'' robustness result~\eqref{eq:mainoutliersbis} in Theorem~\ref{thm:mainfinite} using the same reasoning sketched in Section~\ref{sec:outliers} with the naive bound~\eqref{eq:NaiveOutlierBound}. 

\input{conclusion.tex}

\section*{Acknowledgements}
Many thanks to Karin Schnass for suggesting to make our life much easier with a boundedness rather than sub-Gaussian assumption in the signal model, to Martin Kleinsteuber for helping to disentangle sample complexity from local stability, and to Nancy Bertin for suggesting the cinematographic reference in the title. 
\bibliographystyle{plainnat}
\bibliography{./main_bibliography}

\appendix
\input{expectationphi.tex}
\input{lowerbound.tex}
\input{signrecovery.tex}
\section{Technical lemmata}\label{app:technical}
The final section of this appendix gathers  technical lemmas required by the main results of the paper.

\input{localDL_appendix.tex}

\begin{IEEEbiographynophoto}{R{\'e}mi Gribonval}(FM'14)  is a Senior Researcher with Inria (Rennes, France), and the scientific leader of the PANAMA research group on sparse audio processing. A former student at  {\'E}cole Normale Sup{\'e}rieure (Paris, France), he received the Ph. D. degree in applied mathematics from Universit{\'e} de Paris-IX Dauphine (Paris, France) in 1999, and his Habilitation {\`a} Diriger des Recherches in applied mathematics from Universit{\'e} de Rennes~I (Rennes, France) in 2007. His research focuses on mathematical signal processing, machine learning, approximation theory and statistics, with an emphasis on sparse approximation, audio source separation and compressed sensing. 
\end{IEEEbiographynophoto}

\begin{IEEEbiographynophoto}{Rodolphe Jenatton} received the PhD degree from the Ecole Normale Superieure, Cachan, France, in 2011 under the supervision of Francis Bach and Jean-Yves Audibert. He then joined the CMAP at Ecole Polytechnique, Palaiseau, France, as a postdoctoral researcher working with Alexandre d'Aspremont. From early 2013 until mid 2014, he worked for Criteo, Paris, France, where he was in charge of improving the statistical and optimization aspects of the ad prediction engine. He is now a machine learning scientist at Amazon Development Center Germany, Kurf{\"u}rstenddamm, Berlin. His research interests revolve around machine learning, statistics, (convex) optimization, (structured) sparsity and unsupervised models based on latent factor representations.
\end{IEEEbiographynophoto}

\begin{IEEEbiographynophoto}{Francis Bach} graduated from the Ecole Polytechnique, Palaiseau, France, in 1997. He received
the Ph.D. degree in 2005 from the Computer Science Division at the University of California, Berkeley. He is the leading researcher of the Sierra project-team of INRIA in the Computer Science Department of the Ecole Normale Supérieure, Paris, France. His research interests include machine learning, statistics, optimization, graphical models, kernel methods, and statistical signal processing. He is currently the action editor of the Journal of Machine Learning Research and associate editor
of IEEE Transactions in Pattern Analysis and Machine Intelligence.
\end{IEEEbiographynophoto}

\end{document}

%% file: macros.tex
\newtheorem{theorem}{Theorem}
\newtheorem{lemma}{Lemma}

\newtheorem{proposition}{Proposition}
\newtheorem{corollary}{Corollary}
\newtheorem{definition}{Definition}
\newtheorem{remark}{Remark}
\newtheorem{example}{Example}
\newtheorem{assumption}{Assumption}

\def\fro{{\scriptscriptstyle \mathrm{F}}}
\DeclareFontEncoding{FMS}{}{}
\DeclareFontSubstitution{FMS}{futm}{m}{n}
\DeclareFontEncoding{FMX}{}{}
\DeclareFontSubstitution{FMX}{futm}{m}{n}
\DeclareSymbolFont{fouriersymbols}{FMS}{futm}{m}{n}
\DeclareSymbolFont{fourierlargesymbols}{FMX}{futm}{m}{n}
\DeclareMathDelimiter{\triple}{\mathord}{fouriersymbols}{152}{fourierlargesymbols}{147}

\def\ab{{\mathbf a}}

\def\rb{{\mathbf r}}
\def\xb{{\mathbf x}}

\def\zb{{\mathbf z}}
\def\bb{{\mathbf b}}

\def\wb{{\mathbf w}}

\def\db{{\mathbf d}}
\def\sb{{\mathbf s}}
\def\ub{{\mathbf u}}
\def\vb{{\mathbf v}}

\def\alphab{{\boldsymbol\alpha}}
\def\betab{{\boldsymbol\beta}}

\def\varepsilonb{{\boldsymbol\varepsilon}}

\def\zerob{{\mathbf 0}}

\def\Ub{{\mathbf U}}
\def\Xb{{\mathbf X}}

\def\Wb{{\mathbf W}}
\def\Vb{{\mathbf V}}
\def\Ib{{\mathbf I}}
\def\Ab{{\mathbf A}}
\def\Bb{{\mathbf B}}
\def\Cb{{\mathbf C}}
\def\Gb{\mathbf G}

\def\Db{{\mathbf D}}

\def\Sb{{\mathbf S}}

\def\Kb{{\mathbf K}}
\def\KJb{\Kb_{\J}}

\def\Tb{{\mathbf T}}
\def\PJb{\Pb_{\J}}
\def\PJbo{\PJb^{o}}

\def\thetab{{\boldsymbol\theta}}
\def\GramJb{\Gb_{\J}}
\def\GramJbo{\GramJb^{o}}
\def\Thetab{\mathbf{H}}
\def\ThetaJb{\Thetab_{\J}}
\def\ThetaJbo{\ThetaJb^{o}}

\def\RIP {\delta}
\def\LRIP {\underline{\delta}}
\def\URIP {\overline{\delta}}

\def\Real{{\mathbb{R}}}

\newcommand{\SET}[1]{\llbracket 1; #1\rrbracket}

\def\Bcal{\mathcal{B}}

\def\Lcal{\mathcal{L}}

\def\Dcal{\mathcal{D}}

\def\Scal{\mathcal{S}}

\def\diag{{\mathrm{diag}}}
\def\Diag{{\mathrm{Diag}}}

\def\support{{\mathrm{support}}}
\def\sign{{\mathrm{sign}}}
\def\trace{{\mathrm{Tr}}}

\font\dsrom=dsrom10 scaled 1200
\newcommand{\indicator}[1]{\textrm{\dsrom{1}}_{#1}}
\def\Exp{{\mathbb{E}}}
\def \Pr {\PP}

\def\defin{\triangleq}

\def\alphabo{\alphab^{o}}
\def\alphaboJ{\alphab^{o}_{\J}}
\def\alphaboT{[\alphabo]^{\top}}
\def\alphaboJT{[\alphab^{o}_{\J}]^{\top}}
\def\dbo{{\mathbf d}_0}
\def\sbo{{\mathbf s}^{o}}

\def\sboJ{\sbo_{J}}
\def\sboJT{[\sboJ]^{\top}}
\def\Dbo{\Db^{o}}
\def\DboJ{\Dbo_{\J}}
\def\DbJ{\Db_{\J}}
\def\DboT{[\Dbo]^{\top}}
\def\DbT{\Db^{\top}}
\def\DboJT{[\DboJ]^{\top}}
\def\DbJT{\DbJ^{\top}}
\def\DboJpinv{[\DboJ]^{+}}

\def\J{{\mathrm{J}}}

\def\loweralpha{\underline{\alpha}}
\def\upperalpha{\overline{\alpha}}

\newcommand{\Matrix}[1]{\mathbf{#1}}

\def\signal{\mathbf{x}}
\def\Signal{\Matrix{X}}
\def\db{\mathbf{d}}
\def\Db{\Matrix{D}}
\def\Lcal{\mathcal{L}}

\def\Ub{\Matrix{U}}

\def\RR{\mathbb{R}}
\def\PP{\mathbb{P}}
\def\EE{\mathbb{E}}
\def\defin{\triangleq}

\def\Id{\mathbf{Id}}

\def\zb{\mathbf{z}}
\def\Ab{\mathbf{A}}
\def\fro{F}
\def\Wb{\mathbf{W}}
\def\vb{\mathbf{v}}
\def\PJb{\mathbf{P}_{\J}}

%% file: sketchedresults.tex

\section{Main results}\label{sec:main_results}

Our main results, described below, show that under appropriate scalings of the dictionary dimensions $m$, $p$, number of training samples $n$, 
and model parameters, 
the sparse coding problem~(\ref{eq:min_fn}) admits a local minimum in a neighborhood of $\Dbo$ of controlled size, 
for appropriate choices of the regularization parameter $\lambda$. 
The main building blocks of the results (Propositions~\ref{prop:delta_phi}-\ref{prop:maindeltaphi}-\ref{prop:simplified_exact_recovery}) and the high-level structure of their proofs are given in Section~\ref{sec:proofoutline}. The most technical lemmata are postponed to the Appendix.

\subsection{Stable local identifiability}
We begin with asymptotic results ($n$ being infinite), in the absence of outliers.

\begin{theorem}[Asymptotic results, bounded model, no outlier]\label{thm:mainasymp}
Consider the following  assumptions:
\begin{itemize}
\item {\bf Coherence and sparsity level:} consider $\Dbo \in \mathcal{D}$ and $k$ such that 
\begin{eqnarray}
\label{eq:DefCumCoherMaxMainTheorem}
 \mu_{k}(\Dbo) & \leq & 1/4\\
 \label{eq:DefKMax}
k & \leq & \frac{p}{16 (\triple \Dbo\triple_{2}+1)^{2}}.
\end{eqnarray}
\item {\bf Coefficient distribution:}
assume the Basic \& Bounded signal model (Assumptions~\ref{assume:basic} \&~\ref{assume:bounded}) and
\begin{eqnarray}
 \frac{\Exp\ \alpha^{2}}{M_{\alphab} \Exp\ |\alpha|}
 &>& 
 84  \cdot (\triple \Dbo\triple_{2}+1) \cdot \frac{\frac{k}{p} \cdot  \|\DboT\Dbo-\Ib\|_{\fro}}{1-2\mu_{k}(\Dbo)}.\notag\\
\label{eq:DefRangeMax}
\end{eqnarray}
This implies $C_{\min} < C_{\max}$ where we define
\begin{eqnarray}
C_{\min} & \defin & 24 \kappa_{\alpha}^{2} \cdot (\triple \Dbo\triple_{2}+1) \cdot \frac{k}{p} \cdot \|\DboT\Dbo-\Ib\|_{\fro}, \notag\\
\label{eq:DefRMin}\\
\label{eq:DefRMax}
C_{\max} & \defin &  \frac{2}{7} \cdot \frac{\Exp\ |\alpha|}{M_{\alphab}} \cdot (1-2\mu_{k}(\Dbo)).
\end{eqnarray}
\item {\bf Regularization parameter:} consider a small enough regularization parameter,
\begin{equation}
\label{eq:LambdaMax}
\lambda  \leq \frac{\loweralpha}{4}.
\end{equation}
Denoting $\bar{\lambda}  \defin \tfrac{\lambda}{\Exp\ |\alpha|}$, this implies $C_{\max} \cdot \bar{\lambda} \leq 0.15$. 
\item {\bf Noise level:} assume a small enough relative noise level,
\begin{equation}
\label{eq:DefMaxRelativeNoiseLevelGivenReg}
\frac{M_{\varepsilonb}}{M_{\alphab} } < \frac{7}{2} \cdot (C_{\max}-C_{\min}) \cdot \bar{\lambda}.
\end{equation}
\end{itemize}
Then, for any resolution $r>0$ such that
\begin{eqnarray}
\label{eq:RadiusRange}
C_{\min} \cdot \bar{\lambda} < r < C_{\max} \cdot \bar{\lambda},
\end{eqnarray}
and 
\begin{equation}\label{eq:DefMaxRelativeNoiseLevelThm}
\frac{M_{\varepsilonb}}{M_{\alphab} } < \frac{7}{2} \left(C_{\max} \cdot \bar{\lambda}-r\right),
\end{equation}
the function $\Db \in \mathcal{D} \mapsto \Exp\ F_{\Signal}(\Db)$ admits a local minimum $\hat{\Db}$ such that $\|\hat{\Db}-\Dbo\|_{\fro} < r$.
\end{theorem}
\begin{remark}[Limited over-completeness of $\Dbo$]
It is perhaps not obvious how strong a requirement is assumption~\eqref{eq:DefRangeMax}. 
On the one hand, its left hand side is easily seen to be less than one (and as seen above can be made arbitarily close to one with appropriate coefficient distribution). On the other hand by the Welsh bound $\|\DboT\Dbo-\Id\|_{\fro} \geq \sqrt{p(p-m)/m}$, the bound $\triple \Dbo \triple_{2} \geq  \|\Dbo\|_{\fro}/\sqrt{m} = \sqrt{p/m}$, and the assumption $\mu_{k}(\Dbo) \leq 1/4$, its right hand side is bounded from below by $\Omega(k\sqrt{(p-m)/m^{2}})$. Hence, a consequence of assumption~\eqref{eq:DefRangeMax} is that Theorem~\ref{thm:mainasymp} only applies to dictionaries with limited over-completeness, with $p \lesssim m^{2}$. 
This is likely to be an artifact from the use of coherence in our proof, and a degree of overcompleteness $p=O(m^2)$ covers already interesting practical settings: for example \cite{Mairal2010} consider $m = (8 \times 8)$ patches with $p = 256$ atoms $< m^{2} = 64^2 = 4096$)
\end{remark}
Since $\frac{k}{p} \cdot  \|\DboT\Dbo-\Ib\|_{\fro} \leq k\mu_{1}(\Dbo)$ and $\mu_{k}(\Dbo) \leq k \mu_{1}(\Dbo)$, a crude upper bound on the rightmost factor in~\eqref{eq:DefRangeMax} is $k\mu_{1}(\Dbo)/(1-2k\mu_{1}(\Dbo))$, which appears in many coherence-based sparse-recovery results. 

\subsubsection{Examples}
Instantiating Theorem~\ref{thm:mainasymp} on a few examples highlights the strength of its main assumptions.
\begin{example}[Incoherent pair of orthonormal bases]\label{ex:2}
When $\Dbo$ is an incoherent dictionary in $\Real^{m \times p}$, i.e., a dictionary with (plain) coherence $\mu = \mu_{1}(\Dbo) \ll 1$, we have the estimates \citep{Tropp2004}
\(
\mu_{k}(\Dbo) \leq k\mu
\)
and 
\[
\|\DboT\Dbo-\Ib\|_{\fro} \leq \sqrt{p(p-1)\mu^{2}} \leq p\mu.
\]
Assumption~\eqref{eq:DefCumCoherMaxMainTheorem} therefore holds as soon as $k \leq 1/(4\mu)$. In the case where $p=2m$ and $\Dbo$ is not only incoherent but also a union of two orthonormal bases, we further have $\triple\Dbo\triple_{2} = \sqrt{2}$ hence assumption~\eqref{eq:DefKMax} is fulfilled as soon as $k \leq p/100 = m/50$. Moreover, the right hand side in~\eqref{eq:DefRangeMax} reads
\[
84 \cdot (\triple \Dbo\triple + 1) \cdot \frac{\frac{k}{p} \cdot  \|\DboT\Dbo-\Ib\|_{\fro}}{1-2\mu_{k}(\Dbo)} \leq   \frac{203k\mu}{1-2k\mu} \leq 406k\mu,
\]
and assumption~\eqref{eq:DefRangeMax} holds provided that $\Exp \alpha^{2}/(M_{\alphab} \Exp |\alpha|)$ exceeds this threshold. We discuss below concrete signal settings where this condition can be satisfied:
\begin{itemize}
\item {\bf i.i.d.~bounded coefficient model:} on the one hand, consider nonzero coefficients drawn i.i.d.~with $\PP(|\alphab_{j}| < \loweralpha|j \in J)=0$. The almost-sure upper-bound $M_{\alphab}$ on $\|\alphab\|_{2}$ implies the existence of $\upperalpha \geq \loweralpha$ such that $\PP(|\alphab_{j}| > \upperalpha|j \in J)=0$. As an example, consider coefficients drawn i.i.d.~with $\PP(\alphab_{j} = \pm \upperalpha | j \in \J) = \pi \in (0,1)$ and $\PP(\alphab_{j} = \pm \loweralpha | j \in \J) = 1-\pi$. For large $\upperalpha$ we have $\Exp\alpha^{2} = \pi \upperalpha^{2} + (1-\pi) \loweralpha^{2} \asymp \pi \upperalpha^{2}$, $\Exp|\alpha| \asymp \pi \upperalpha$, and $M_{\alphab} = \sqrt{k} \upperalpha$. This yields
\[
\lim_{\upperalpha \to \infty}\Exp \alpha^{2}/(M_{\alphab} \Exp |\alpha|) = 1/\sqrt{k},
\]

This shows the existence of a coefficient distribution satisfying~\eqref{eq:DefRangeMax} as soon as $406 k\mu < 1/\sqrt{k}$, that is to say $k < 1/(406\mu)^{2/3}$. 
In the maximally incoherent case, for large $p$, we have $\mu = 1/\sqrt{m} \asymp p^{-1/2}$, and conditions~\eqref{eq:DefCumCoherMaxMainTheorem}-\eqref{eq:DefKMax}-\eqref{eq:DefRangeMax} read $k = O(p^{1/3})$.

\item {\bf fixed amplitude profile coefficient model:} on the other hand, completely relax the independence assumption and consider essentially the coefficient model introduced by~\citet{Schnass:2013vd} where $\alphab_{j} = \epsilon_{j} \mathbf{a}_{\sigma(j)}$ with i.i.d. signs $\epsilon_{j}$ such that $\PP(\epsilon_{j}=\pm 1)=1/2$, a random permutation $\sigma$ of the index set $\J$, and $\mathbf{a}$ a given vector with entries $\mathbf{a}_{j} \geq \loweralpha, j \in \J$. 
This yields
\[
\Exp \alpha^{2}/(M_{\alphab} \Exp |\alpha|) = \tfrac{1}{k} \|\mathbf{a}\|_{2}^{2}/(\|\mathbf{a}\|_{2} \cdot \tfrac{1}{k}\|\mathbf{a}\|_{1}) = \|\mathbf{a}\|_{2}/\|\mathbf{a}\|_{1},
\]
which can be made arbitrarily close to one even with the constraint $\mathbf{a}_{j} \geq \loweralpha, j \in \J$. This shows the existence of a coefficient distribution satisfying~\eqref{eq:DefRangeMax} as soon as $406 k\mu < 1$, a much less restrictive condition leading to $k=O(p^{1/2})$. The reader may notice that such distributions concentrate most of the energy of $\alphab$ on just a few coordinates, so in a sense such vectors are much sparser than $k$-sparse.
\end{itemize}
\end{example}

\begin{example}[Spherical ensemble]\label{ex:3}
Consider $\Dbo \in \Real^{m \times p}$ a typical draw from the spherical ensemble, that is a dictionary obtained by normalizing a matrix with standard independent Gaussian entries. As discussed above, condition~\eqref{eq:DefRangeMax} imposes overall dimensionality constraints $p \lesssim m^{2}$. Moreover, using usual results for such dictionaries~\citep[see, e.g.,][]{Candes2009}, the condition in \eqref{eq:DefCumCoherMaxMainTheorem} is satisfied as soon as $\mu_k \leq k \mu_1 \approx k \sqrt{\log p} / \sqrt{m} = O(1)$, i.e., $k = O( \sqrt{m / \log p} )$, while the condition
in \eqref{eq:DefKMax} is satisfied as long as $k = O(m)$ (which is weaker).
\end{example}

\subsubsection{Noiseless case: exact recovery}
In the noiseless case ($M_{\varepsilonb} = 0$),~\eqref{eq:DefMaxRelativeNoiseLevelGivenReg} imposes no lower bound on admissible regularization parameter. Hence, we deduce from Theorem~\ref{thm:mainasymp} that a local minimum of $\Exp\ F_{\Signal}(\cdot)$ can be found arbitrarily close to $\Dbo$, provided that the regularization parameter $\lambda$ is small enough. 
This shows that the reference dictionary $\Dbo$ itself is in fact a local minimum of the problem considered by~\citet{Gribonval2010,Geng2011},  
\begin{equation}
\label{eq:DefExactL1}
\min_{\Db \in \Dcal} F_{\Xb}^{0}(\Db)\ \mbox{where}\ F_{\Xb}^{0}(\Db) \defin \min_{\Ab: \Db\Ab = \Xb} \|\Ab\|_{1}.
\end{equation}
Note that here we consider a different random sparse signal model, and yet recover the same results together with a new extension to the noisy case.

\subsubsection{Stability to noise}\label{sec:stabilitytonoise}
In the presence of noise, conditions~\eqref{eq:LambdaMax} and~\eqref{eq:DefMaxRelativeNoiseLevelGivenReg} respectively impose an upper and a lower limit on admissible regularization parameters, which are only compatible for small enough levels of noise
\[
M_{\varepsilonb} \lesssim \loweralpha (1-2\mu_{k}(\Dbo)).
\]
In scenarios where $C_{\min} \ll C_{\max}$ (i.e., when the left hand side in~\eqref{eq:DefRangeMax} is large enough compared to its right hand side), admissible regularization parameters are bounded from below given~\eqref{eq:DefMaxRelativeNoiseLevelGivenReg} as $\bar{\lambda} \gtrsim \frac{M_{\varepsilonb}}{M_{\alphab} C_{\max}}$, therefore limiting the achievable ``resolution'' $r$ to
\begin{equation}
\label{eq:AchievableResolutionNoisy}
\begin{split}
r > C_{\min} \bar{\lambda} \gtrsim & \frac{M_{\varepsilonb}}{M_{\alphab} } \cdot \frac{C_{\min}}{C_{\max}} \\
\asymp & \frac{M_{\varepsilonb}}{\sqrt{\Exp\ \alpha^{2}}} \cdot \kappa_{\alpha}  \cdot \triple \Dbo\triple_{2} \cdot   \frac{\frac{k}{p} \cdot  \|\DboT\Dbo-\Ib\|_{\fro}}{1-2\mu_{k}(\Dbo)}.
\end{split}
\end{equation}
Hence, with enough training signals and in the absence of outliers, the main resolution-limiting factors are 
\begin{itemize}
\item the relative noise level $M_{\varepsilonb}/\sqrt{\Exp\ \alpha^{2}}$: the smaller the better;
\item the level of typical ``flatness'' of $\alphab$ as measured by $\kappa_{\alpha}$: the peakier (the smaller $\kappa_{\alpha}$) the better;
\item the coherence of the dictionary as measured jointly by $\mu_{k}(\Dbo)$ and $\frac{k}{p} \cdot  \|\DboT\Dbo-\Ib\|_{\fro}$: the least coherent the better.
\end{itemize}

Two other resolution-limiting factors are the finite number of training samples $n$ and the presence of outliers, which we now discuss.

\subsection{Robust finite sample results}\label{sec:samplecomplexity}

We now trade off precision for concision and express finite sample results with two non-explicit constants~$C_{0}$ and $C_{1}$. Their explicit expression in terms of the dictionary and signal model parameters can be tracked back by the interested reader in the proof of Theorem~\ref{thm:mainfinite} (Section~\ref{sec:proofthmmainfinite}), but they are left aside for the sake of concision.

\begin{theorem}[Robust finite sample results, bounded model]\label{thm:mainfinite}
Consider a dictionary $\Dbo \in \mathcal{D}$ and a sparsity level $k$ satisfying the assumptions~\eqref{eq:DefCumCoherMaxMainTheorem}-\eqref{eq:DefKMax} of Theorem~\ref{thm:mainasymp}, and the Basic \& Bounded signal model (Assumptions~\ref{assume:basic} \&~\ref{assume:bounded}) with   parameters satisfying the assumption~\eqref{eq:DefRangeMax}. 
 There are two constants $C_{0}, C_1>0$ independent of all considered parameters with the following property.

Given a reduced regularization parameter $\bar{\lambda}$ and a noise level satisfying assumptions~\eqref{eq:LambdaMax} and~\eqref{eq:DefMaxRelativeNoiseLevelGivenReg},  a radius $r$ satisfying~\eqref{eq:RadiusRange} and~\eqref{eq:DefMaxRelativeNoiseLevelThm}, 
 and a confidence level $x>0$, when $n_{\textrm{in}}$ training samples are drawn according to the Basic \& Bounded signal model with
\begin{equation}\label{eq:mainsamplecomplexity}
\begin{split}
n_{\textrm{in}}>
C_{0} & \cdot 
\left( mp +x\right)  \cdot p^{2} \cdot \left(\tfrac{M_{\alphab}^{2}}{  \Exp  \| \alphab\|_2^{2}}\right)^{2} \\
& \cdot 
\left(
\tfrac{r+\left(\tfrac{M_{\varepsilonb}}{M_{\alphab}}+\bar{\lambda}\right)+\left(\tfrac{M_{\varepsilonb}}{M_{\alphab}}+\bar{\lambda}\right)^{2} }
{r-C_{\min} \cdot \bar{\lambda}}
\right)^{2},
\end{split}
\end{equation}
we have: with probability at least $1-2e^{-x}$,  the function $\Db \in \mathcal{D} \mapsto F_{\Signal}(\Db)$ admits a local minimum $\hat{\Db}$ such that $\|\Db-\Dbo\|_{\fro} < r$.
Moreover, this is robust to the addition of outliers $\Signal_{\textrm{out}}$ provided that
\begin{equation}\label{eq:mainoutliers}
\tfrac{\|\Signal_{\textrm{out}}\|_{\fro}^{2}}{n_{\textrm{in}} } \leq 
 \Exp  \| \alphab\|_2^{2} \cdot \left[
\tfrac{1}{4p} \cdot  \Big(1-\tfrac{C_{\min} \cdot \bar{\lambda}}{r} \Big)-
C_{1} \sqrt{\tfrac{(mp+x)}{n_{\textrm{in}}}}
\right] \cdot r^{2}.
\end{equation}
As soon as the dictionary is coherent, we have $C_{\min} \neq 0$, hence the constraint~\eqref{eq:RadiusRange} implies that the right hand side of~\eqref{eq:mainoutliers} scales as $O(r^{2}) = O(\lambda^{2})$. In the noiseless case, this imposes a tradeoff between the seeked resolution~$r$, the tolerable total energy of outliers, and the number of inliers. With a more refined argument, we obtain the alternative condition
\begin{equation}\label{eq:mainoutliersbis}
\begin{split}
\tfrac{\|\Signal_{\textrm{out}}\|_{1,2}}{n_{\textrm{in}}} \leq 
3  \tfrac{\sqrt{k} \Exp\ \|\alphab\|_{2}^{2} } { \Exp |\alpha|   } & \cdot
\left[
 \tfrac{1}{p} \cdot  \left(1-\tfrac{C_{\min} \cdot \bar{\lambda}}{r}\right) -
C_{1}   \sqrt{\tfrac{(mp+x)}{n_{\textrm{in}}}}
\right]\\
& \cdot  \tfrac{r}{\bar{\lambda}} \cdot 
\tfrac{(A^o)^{3/2}}{ 18 p^{3/2}},
\end{split}
\end{equation}
where $A^{o}$ is the lower frame bound of $\Dbo$, i.e., such that $A^{o} \|\signal\|_{2}^{2} \leq \|(\Dbo)^{\top}\signal\|_{2}^{2}$ for any signal $\signal$.
\end{theorem}

The factor $M_{\alphab}^{2}/\Exp\ \|\alphab\|_{2}^{2} = {``}\sup \| \alphab \|_2^2 {''} / \Exp\ \|\alphab\|_{2}^{2} $ in the right hand side of~\eqref{eq:mainsamplecomplexity} is always greater than 1, but typically remains bounded (note that if the distribution of $\alphab$ allows outliers, they could be treated within the outlier model). In the symmetric decaying model of \citet{Schnass:2013vd} where $\alphab$ is a randomly permuted and signed flipped version of a given vector, this factor is equal to one.

Even though the robustness to outliers is expressed in~\eqref{eq:mainoutliers} as a control of $\|\Signal_{\textrm{out}}\|_{\fro}^{2}/n_{\textrm{in}}$, it should really be considered as a control of an {\em outlier to inlier energy ratio}: $\|\Signal_{\textrm{out}}\|_{\fro}^{2}/[n_{\textrm{in}} \Exp\ \|\alphab\|_{2}^{2}]$, and similarly with a proper adaptation in~\eqref{eq:mainoutliersbis}. One may notice that the robustness to outliers expressed in Theorem~\ref{thm:mainfinite} is somehow a ``free'' side-effect of the conditions that hold on inliers with high probability, rather than the result of a specific design of the  cost function $F_{\Signal}(\Db)$.

\subsubsection{Example: orthonormal dictionary}
Consider $p=m$ and $\Dbo$ an orthonormal dictionary in $\Real^{m \times p}$. Since $\mu_{k}(\Dbo)=0$, $\triple\Dbo\triple_{2} = 1$ and $\|\DboT\Dbo-\Ib\|_{\fro}=0$, 
assumption~\eqref{eq:DefKMax} reads\footnote{Improved constants in Theorem~\ref{thm:mainasymp} are achievable when specializing to orthonormal dictionaries, they are left to the reader.} $k \leq p/64$, 
assumptions~\eqref{eq:DefCumCoherMaxMainTheorem} and~\eqref{eq:DefRangeMax} impose no constraint, and $C_{\min} = 0$. 
Moreover, the reader can check that if $M_{\varepsilonb} < \lambda \leq \loweralpha/4$, then ~\eqref{eq:LambdaMax}-\eqref{eq:DefMaxRelativeNoiseLevelThm} hold for $0<r<\tfrac{2(\lambda-M_{\varepsilonb})}{7M_{\alphab}}$. 
\begin{itemize}
\item {\bf Low-noise regime:} if  $M_{\varepsilonb} <  \loweralpha/4$ and $k \leq p/64$, then choosing $M_{\varepsilonb} <  \lambda \leq \loweralpha/4$ yields: 
\begin{itemize}
\item by Theorem~\ref{thm:mainasymp} (the limit of large $n$), $\EE F_{\Xb}(\Db)$ admits a local minimum \emph{exactly at} $\Dbo$;
\item by Theorem~\ref{thm:mainfinite}, \emph{even though the regularization parameter cannot be made arbitrarily small}, we obtain that for any confidence level $x>0$ and \emph{arbitrary small precision} $r>0$, $F_{\Xb}(\Db)$ admits a local minimum within radius $r$ around $\Dbo$ with probability at least $1-2e^{-x}$ provided that
\[
n  = \Omega \Big((mp^{3}  + xp^{2}) \left(\tfrac{M_{\varepsilonb}/M_{\alphab}}{r}\right)^{2}\Big).
\]
While the orthogonality of the dictionary remarkably allows to achieve an arbitrary precision despite the presence of noise, we still have to pay a price for the presence of noise through a resolution-dependent sample complexity.
\end{itemize}
\item {\bf Noiseless regime ($M_{\varepsilonb}=0$):} with $\lambda \asymp r$, an arbitrary resolution $r$ is reached with a {\em resolution independent} number of training samples 
\[
n = \Omega(mp^{3}+xp^{2}).
\]
This is robust to outliers provided $\|\Signal_{\textrm{out}}\|_{1,2}/n_{\textrm{in}}$ does not exceed a {\em resolution independent} threshold.
\end{itemize}
The case of orthonormal dictionaries is somewhat special in the sense that orthonormality yields $C_{\min} = 0$ and breaks the forced scaling $r \asymp \bar{\lambda}$ otherwise imposed by~\eqref{eq:RadiusRange}. Below we discuss in more details the more generic case of non-orthonormal dictionaries in the noiseless case.

\subsubsection{Noiseless case: exact recovery and resolution independent sample complexity}
Consider now the noiseless case ($M_{\varepsilonb}=0$) without outlier ($ \Signal_{\textrm{out}}=0 $). In general we have $C_{\min} > 0$ hence the best resolution $r>0$ guaranteed by Theorem~\ref{thm:mainasymp} in the asymptotic regime is $r = r_{\min} \defin C_{\min} \cdot \bar{\lambda} > 0$. When $C_{\max} > 2 C_{\min}$, Theorem~\ref{thm:mainfinite} establishes that the only slightly worse resolution $r=2r_{\min}$ can be achieved with high probability with a number of training samples $n$ which is {\em resolution independent}. More precisely~\eqref{eq:mainsamplecomplexity} indicates that when $M_{\alphab}^{2}/\Exp\ \|\alphab\|_{2}^{2} \approx 1$, 
it is sufficient to have a number of training samples 
\[
n = \Omega(mp^{3})
\]
to ensure the existence of a local minimum within a radius $r$ around the ground truth dictionary $\Dbo$, where {\em the resolution $r$ can be made arbitrarily fine by choosing $\lambda$ small enough}. This recovers the known fact that, with high probability, the function $F_{\Xb}^{0}(\Db)$ defined in~\eqref{eq:DefExactL1} has a local minimum {\em exactly} at $\Dbo$, as soon as $n = \Omega(mp^{3})$.
Given our boundedness assumption, the probabilistic decay as $e^{-x}$ is expected and show that as soon as  $ n \geq \Omega(mp^{3})$, the infinite sample result is reached quickly.

In terms of outliers, both \eqref{eq:mainoutliers} and \eqref{eq:mainoutliersbis} provide a control of the admissible ``energy'' of outliers. Without additional assumption on $\Dbo$, the allowed energy of outliers in \eqref{eq:mainoutliers} has a leading term in $r^2$, i.e., to guarantee a high precision, we can only tolerate a small amount of outliers as measured by the ratio $\|\Signal_{\textrm{out}}\|_{\fro}^{2}/n_{\textrm{in}} $. However, when the dictionary $\Dbo$ is complete --a rather mild assumption-- the alternate ratio  $\|\Signal_{\textrm{out}}\|_{1,2}/n_{\textrm{in}}$ does not need to scale with the targeted resolution $r$ for $r = 2C_{\min} \bar{\lambda}$. In the proof, this corresponds to replacing the control of the minimized objective function by that of its variations.

The above described resolution-independent results are of course specific to the noiseless setting. In fact, as described in Section~\ref{sec:stabilitytonoise}, the presence of noise when the dictionary is not orthonormal imposes an absolute limit to the resolution $r > r_{\min}$ we can guarantee with the techniques established in this paper. When there is noise, \cite{Arora:2013vq} discuss why it is in fact impossible to get a sample complexity with better than $1/r^2$ dependency.

%% file: conclusion.tex
\section{Conclusion and discussion} \label{sec:discussion}

We conducted an asymptotic as well as precise finite-sample analysis of the local minima of sparse coding in the presence of noise, 
thus extending prior work which focused on noiseless settings~\citep{Gribonval2010,Geng2011}.
Given a probabilistic model of sparse signals that only combines assumptions on certain first and second order moments, and almost sure boundedness, we have shown that a local minimum exists with high probability around the reference dictionary, under cumulative-coherence assumptions on the ground truth dictionary. 
We have shown the robustness of the approach to the presence of outliers, provided a certain ``outlier to inlier energy ratio'' remains small enough.
In contrast to related prior work, the sample complexity estimates we obtained are independent of the precision of the predicted recovery. 
Similarly, the admissible level of outliers under some additional completeness assumption has been shown to be harmless to the targeted resolution.

Our study could be further developed in multiple ways. First, we may target more realistic of widely accepted generative models for $\alphabo$ such as 
the spike and slab models of~\citet{Ishwaran2005}, 
or signals with compressible priors~\citep{Gribonval2011}. 
Second, one may want to deal with other constraint sets $\mathcal{D}$ on the dictionary to deal with related problems such as structured dictionary learning \citep{Gribonval:2013tx} or blind calibration. This may yield improved sample complexity estimates where, e.g., a factor $mp$ could be replaced with the upper box-counting dimension of $\mathcal{D}$. Moreover, more refined estimates in the spirit of \cite{Maurer2010} could possibly provide sample complexity estimates that no longer depend on the signal dimension $m$, or fast rates $\eta_{n} = O(1/n)$, rather than $\eta_{n} = O(1/\sqrt{n})$ which would both translate into better sample complexity estimates (e.g., $mp^{2}$ rather than $mp^{3}$ with fast rates). Note here that the lower-bound recently proved by~\citet{jung2014performance} leads to a sample complexity of at least $p^2$, which still leaves room for improvement (either for the lower or upper bounds).

Third, the analysis could potentially be extended to other penalties than $\ell^{1}$, e.g., with mixed norms promoting group sparsity. A related problem is that of considering complex-valued rather than only real-valued dictionary learning problems. The recent results of \citet{Vaiter:2014vu} establishing the stable recovery of a generalized notion of ``support'' through a generalized irrepresentability condition might be instrumental with this respect.

\paragraph{Beyond exact recovery, and beyond coherence ?}

The spirit of our analysis, as described in Section~\ref{sec:proofoutline}, is that one can approximate the empirical cost function $\Db \mapsto \Delta F_{\Signal}(\Db;\Dbo)$ by the expectation of the idealized cost function $\Db \mapsto \Exp_{\signal} \Delta \phi_{\signal}(\Db;\Dbo|\sbo)$ (Proposition~\ref{prop:delta_phi}). A simple restricted isometry property is enough to show the existence of a local minimum of the latter which is both close to $\Dbo$ (Proposition~\ref{prop:maindeltaphi}) and global on a large ball around $\Dbo$. However, we use more heavy artillery to control how closely $\Delta F_{\Signal}$ is approximated by $\Exp_{\signal}\Delta \phi_{\signal}$: a cumulative coherence assumption coupled with the assumption that nonzero coefficients are bounded from below. Using exact recovery arguments (Proposition~\ref{prop:simplified_exact_recovery}), this implies that in a neighborhood of $\Dbo$ of controlled (but small) size, we have almost surely equality between $\phi_{\signal}(\Db)$ and $f_{\signal}(\Db)$.

While this route has the merit of a relative simplicity\footnote{From a certain point of view \ldots}, it is also introduces several limitations:
\begin{itemize}
\item {\em limited sparsity}: the cumulative coherence assumption restricts much more the admissible sparsity levels than a simple restricted isometry property assumption.
\item {\em local vs global}:  Proposition~\ref{prop:simplified_exact_recovery} controls the quality of the approximation of  $\Exp \Delta F_{\Signal}$ by $\Exp_{\signal}\Delta \phi_{\signal}$ on a neighborhood whose size $r$ cannot exceed $O(\lambda)$. 
In contrast, using only a RIP assumption, Proposition~\ref{prop:maindeltaphi} provides a lower bound~\eqref{eq:MainLowerBound} of $\Db \mapsto \Exp_{\signal}\Delta \phi_{\signal}(\Db;\Dbo|\sbo)$ which is valid on a large neighborhood of $\Dbo$ of radius $r = O(1)$.

Even though dictionaries in $\RR^{m \times p}$ can be at much higher mutual Frobenius distances that $O(1)$, one cannot envision to significantly improve over the radius $r=O(1)$ for which $\Exp \Delta F_{\Signal}(r) > 0$. To see why, consider $\Db$ a dictionary of coherence $\mu_{1}(\Db)$, and $i,j$ a pair of distinct atoms such that $|[\db^{i}]^{\top}\db^{j}| = \mu_{1}(\Db)$. Consider $\Db'$ obtained by permuting these two atoms and possibly flipping the sign of one of them: then $F_{\Signal}\Db') = F_{\Signal}(\Db)$, and $\|\Db'-\Db\|_{\fro}^{2} = 2 \|\db^{i}\pm \db^{j}\|_{2}^{2} = 2-2\mu_{1}(\Db)$. Hence, $\Db'$ is within radius $r \leq \sqrt{2(1-\mu_{1}(\Db))} = O(1)$ of $\Db$ (in the Frobenius distance) but $\Delta F_{\Signal}(\Db';\Db) = 0$.

Of course, Proposition~\ref{prop:simplified_exact_recovery} is sufficient to prove the desired existence of a local minimum $\hat{\Db}$ of $\Db \mapsto \Exp \Delta F_{\Signal}(\Db)$ (Theorem~\ref{thm:mainasymp}). However, controlling the quality of the approximation of  $\Exp \Delta F_{\Signal}$ by $\Exp_{\signal}\Delta \phi_{\signal}$ on a much larger neighborhood would seem desirable, since it would show that $\hat{\Db}$ is not only a local minimum, but also that it is {\em global over a ball of large radius $r=O(1)$ around $\Dbo$}. This has the potential of opening the way to algorithmic results in terms of the practical optimization of $F_{\Signal}(\Db)$ rather than just properties of this cost function, in the spirit of the recent results~\cite{Agarwal:2013tya} etc. establishing the size of the basin of convergence of an alternate minimization approach based on exact $\ell^{1}$ minimization.
\end{itemize}

To address the above limitations, one can envision an analysis that would replace the assumption on $\mu_{k}(\Dbo)$ by an assumption on $\LRIP_{k}(\Dbo)$. This would imply, e.g., to replace Lemma~\ref{lem:uniquelasso} and Lemma~\ref{lem:alphabound} to obtain recovery results {\em with high probability} rather than {\em almost surely}, through an explicit expression of $\hat{\alphab}_{\signal}(\Db|\sbo) - \alphabo$ and a control of its $\ell^{\infty}$ norm with high probability, in the spirit of~\citet{Candes2009}. As a by-product of such improvements, one 
can expect to remove the unnecessarily conservative assumption~\eqref{eq:coeffthreshold} involving $\loweralpha$, but also replacing $\loweralpha$ with $\Exp\ |\alpha|$ in Theorem~\ref{thm:mainasymp} (assumption~\eqref{eq:LambdaMax}) and Theorem~\ref{thm:mainfinite}, as well as replacing $M_{\alphab}$ and $M_{\varepsilonb}$ with expected values rather than worst case quantities. To support these improvements, a promising approach consists in exploiting convex duality to directly lower bound $F_{\Signal}(\Db) - F_{\Signal}(\Dbo)$ without resorting to exact recovery. This also has the potential to yield guarantees where assumption~\eqref{eq:DefRangeMax} is relaxed, thus encompassing very overcomplete dictionaries beyond the $p \lesssim m^{2}$ barrier faced in this paper.

%% file: expectationphi.tex

\section{Control of $\Delta\phi_{\signal}(\Db;\Db'|\sbo)$}

\subsection{Expression of $\Delta\phi_{\signal}(\Db;\Db'|\sbo)$}\label{sec:expressionphi}

By Definition~\ref{def:phi} we have 
\begin{equation*}
\begin{split}
\phi_\xb&(\Db|\sbo)\\
& = 
\frac{1}{2} \big[ \|\xb\|_2^2 - ( \Db_\J^\top \xb - \lambda \sb_\J)^\top 
(\Db_\J^\top \Db_\J)^{-1} 
(\Db_\J^\top \xb -\lambda \sb_\J ) \big] \\
&= \frac{1}{2} \|\xb\|_2^2 - \frac{1}{2} \xb^\top \PJb \xb + \lambda \sb_{\J}^\top  \Db_{\J}^+ \xb 
- \frac{\lambda^2}{2} \sb_{\J}^\top \ThetaJb \sb_{\J}.
\end{split}
\end{equation*}
Since $\xb = \Dbo \alphabo + \varepsilonb = \DboJ [\alphabo]_{\J} + \varepsilonb$, it follows that for any pair $\Db,\Db'$
 \begin{equation}
 \label{eq:DevelDeltaPhi1}
 \begin{split}
 \Delta  \phi_{\xb} & (\Db;\Db'|\sbo) \\
= &
\frac{1}{2} \xb^\top [ \PJb'  -  \PJb] \xb 
 - \lambda \sb_{\J}^\top \big[  [\DbJ']^{+}  - \Db_{\J}^+ \big] \xb \\
& +\frac{\lambda^2}{2} \sb_{\J}^\top [  \ThetaJb'   - \ThetaJb ]  \sb_{\J}
\\
 =& 
\Delta\phi_{\alphab,\alphab} 
+ \Delta\phi_{\alphab,\varepsilonb}
+\Delta\phi_{\varepsilonb,\varepsilonb}
+\Delta\phi_{\sb,\alphab}
+\Delta\phi_{\sb,\varepsilonb}
+\Delta\phi_{\sb,\sb}\notag
\end{split}
\end{equation}
with the following shorthands
\begin{eqnarray}
\Delta\phi_{\alphab,\alphab}(\Db;\Db')
&\defin &
\frac{1}{2} \alphaboT  \DboT  (\PJb' - \PJb) \Dbo \alphabo
 \notag\\
\Delta\phi_{\alphab,\varepsilonb}(\Db;\Db')
&\defin& 
\varepsilonb^{\top}  (\PJb' - \PJb)\Dbo \alphabo \notag\\
\Delta\phi_{\varepsilonb,\varepsilonb}(\Db;\Db')
&\defin&
\frac{1}{2} \varepsilonb^\top (\PJb' - \PJb)\varepsilonb\notag\\
\Delta\phi_{\sb,\alphab}(\Db;\Db')
&\defin&
-\lambda \sb_{\J}^\top  ([\DbJ']^{+} - \Db_{\J}^+ ) \Dbo \alphabo
\notag\\
\Delta\phi_{\sb,\varepsilonb}(\Db;\Db')
&\defin&
-\lambda \sb_{\J}^\top ( [\DbJ']^{+} - \Db_{\J}^+) \varepsilonb \notag\\
\Delta\phi_{\sb,\sb}(\Db;\Db')
&\defin&
\frac{\lambda^2}{2} \sb_{\J}^\top (\ThetaJb'  -  \ThetaJb) \sb_{\J}\notag
\end{eqnarray}

\subsection{Expectation of $\Delta\phi_{\signal}(\Db;\Dbo|\sbo)$}\label{sec:expectationphi}

Specializing to $\Db' = \Dbo$ we have
\begin{eqnarray}
\Delta\phi_{\alphab,\alphab}
&\defin &
\frac{1}{2} \alphaboT  \DboT  (\Ib - \PJb) \Dbo \alphabo\notag\\
\Delta\phi_{\alphab,\varepsilonb}
&\defin& 
 \varepsilonb^{\top}  (\Ib - \PJb)\Dbo \alphabo\notag\\
 \Delta\phi_{\sb,\alphab}
&\defin&
-\lambda \sb_{\J}^\top  (\Ib - \Db_{\J}^+ \Dbo) \alphabo
\end{eqnarray}
Moreover, under the basic signal model (Assumption~\ref{assume:basic}), by the decorrelation between $\alphab$ and $\varepsilonb$ we have
\[
\Exp\{\Delta\phi_{\alphab,\varepsilonb} \}
=
\Exp\{\Delta\phi_{\sb,\varepsilonb} \}
=
0.
\]
Moreover, we can rewrite 
\begin{eqnarray*}
\Delta\phi_{\alphab,\alphab}
&=&
\frac12 \cdot \trace\Big( \alphabo_{\J} [\alphabo]_{\J}^\top \cdot \DboJT  ( \Ib - \PJb ) \DboJ \Big)\\
\Delta\phi_{\varepsilonb,\varepsilonb}
&=& \frac{1}{2} \cdot \trace \Big(\varepsilonb\varepsilonb^{\top} \cdot (\PJbo-\PJb)\Big)\\
 \Delta\phi_{\sb,\alphab}
 &=& 
 -\lambda \cdot \trace \Big( \alphabo_{\J}\, \sb_{\J}^\top \cdot 
( \Ib-\Db_{\J}^+ \DboJ) 
\Big)\\
\Delta\phi_{\sb,\sb}
&=& 
\frac{\lambda^{2}}{2} \cdot 
\trace\left(\ThetaJbo-\ThetaJb\right).
\end{eqnarray*} 
Since $\PJb$ is an orthoprojector onto a subspace of dimension $k$, $\trace(\PJbo-\PJb) = k-k = 0$, hence
\begin{eqnarray*}
\Exp\{\Delta\phi_{\alphab,\alphab}\}
&=&
\tfrac{\Exp\{\alpha^{2}\}}{2} \cdot \Exp_{J}\Big\{\trace\big(  \DboJT  ( \Ib - \PJb ) \DboJ \big)\Big\}\\
\Exp\{\Delta\phi_{\varepsilonb,\varepsilonb} \}
&=& 
\tfrac{\Exp\{\varepsilon^{2}\}}{2} \cdot  \Exp_{J}\Big\{\trace ( \PJbo-\PJb)\Big\} = 0\\
\Exp\{ \Delta\phi_{\sb,\alphab} \}
 &=& 
- \lambda \cdot \Exp\{|\alpha|\} \cdot \Exp_{J}\Big\{ \trace \left( \Ib- \Db_{\J}^+ \DboJ\right)\Big\}\\
\Exp\{\Delta\phi_{\sb,\sb}\}
&=& 
\tfrac{\lambda^{2}}{2} \cdot 
\Exp_{J}\left\{ \trace\left(\ThetaJbo-\ThetaJb\right)\right\}.
\end{eqnarray*}

%% file: lowerbound.tex
\subsection{Proof of Proposition~\ref{prop:maindeltaphi}}\label{app:bounds}
The lower bound for $\Delta \phi_\PP(\Db;\Dbo|\sbo)$ relies on a series of lemmatas whose proof is postponed to Appendix~\ref{app:technical}. 
\begin{lemma}
\label{lem:RIPBounds}
Let $\Db \in \Real^{m\times p}$ be a dictionary such that $\LRIP_{k}(\Db)<1$. 
Then, for any $\J$ of size $k$, the $\J \times \J$ matrix $\ThetaJb$ is well defined and we have 
\begin{eqnarray}
\label{eq:BoundThetaJ}
\triple\ThetaJb\triple_{2} 
& \leq & \frac{1}{1-\LRIP_{k}(\Db)}.
\\
\label{eq:BoundDThetaJ}
\triple \DbJ^{+} \triple_2 
&\leq& \frac{1}{\sqrt{1-\LRIP_{k}(\Db)}}.
\end{eqnarray}
Moreover, for any $\Db'$ such that $\|\Db'-\Db\|_{\fro} \leq r < \sqrt{1-\LRIP_{k}(\Db)}$
we have
\begin{equation}
1-\LRIP_{k}(\Db') \geq (\sqrt{1-\LRIP_k(\Db)}-r)^{2} \defin 1-\LRIP
\end{equation}
\end{lemma}
\begin{lemma}\label{le:LipBoundsRIP}
For any $\LRIP<1$, $\Db,\Db'$ such that $\max(\LRIP_{k}(\Db),\LRIP_{k}(\Db')) \leq \LRIP$, and $\J$ of size $k$,  we have
\begin{eqnarray*}
\triple\Ib-\DbJ^{+}\DbJ'\triple_{2}  
\leq&  (1-\LRIP)^{-1/2} & \|\Db-\Db'\|_{\fro}\\
\triple \ThetaJb'  -  \ThetaJb\triple_{2} 
\leq& 2 (1-\LRIP)^{-3/2} & \|\Db-\Db'\|_{\fro}\\
\triple [\DbJ']^{+}-\DbJ^{+}\triple
 \leq & 2 (1-\LRIP)^{-1} & \|\Db-\Db'\|_{\fro}\\
\triple \PJb'-\PJb\triple_{2} 
\leq & 2 (1-\LRIP)^{-1/2} & \|\Db-\Db'\|_{\fro}.
\end{eqnarray*}
\end{lemma}

\begin{lemma}
\label{lem:paramonto} Denote $\mathcal{D}$ the oblique manifold. Given any $\Db_{1},\Db_{2} \in \Dcal$, there exists a matrix $\Wb \in \Real^{m \times p}$  with $\diag(\Db^{\top}\Wb) = 0$ and $\diag(\Wb^{\top}\Wb)=\Ib$, i.e., $\Wb = [\wb^{1}, \ldots, \wb^{p}]$, $\wb^{j} \perp \db_{1}^{j}$, $\|\wb^{j}\|_{2}=1$, $j \in \SET{p}$, 
and a vector 
\(
\thetab \defin \thetab(\Db_{1},\Db_{2}) \in [0,\ \pi]^{p}
\)
such that
\begin{eqnarray}
\Db_{2} &=& \Db_{1}\Cb(\thetab) +  \Wb \Sb(\thetab)\\
 \Cb(\thetab) & \defin & \Diag(\cos\thetab)\\
 \Sb(\thetab) & \defin & \Diag(\sin\thetab)
 \end{eqnarray} 
where $\cos \thetab$ (resp. $\sin \thetab$) is the vector with entries $\cos \thetab_{j}$ (resp. $\sin \thetab_{j}$).
Moreover, we have 
\begin{eqnarray}
\frac{2}{\pi}\thetab_{j} & \leq & \|\db^{j}_{2}-\db^{j}_{1}\|_2 = 2 \sin \left(\frac{\thetab_{j}}{2}\right) \leq \thetab_{j},\quad \forall j,\\
\frac{2}{\pi}\|\thetab\|_{2}& \leq & \|\Db_{2}-\Db_{1}\|_{\fro} \leq \|\thetab\|_{2}
\end{eqnarray}
Vice-versa, $\Db_{1} = \Db_{2}\Cb(\thetab)+\Wb' \Sb(\thetab)$ where $\Wb'$ has its unit columns orthogonal to those of $\Db_{2}$.
\end{lemma}
The above lemma involves $\thetab(\cdot,\cdot)$, with is related to the geodesic distance $\|\thetab\|_{2}$ on the oblique manifold $\mathcal{D}$~\citep{Absil2008}. Our main technical bounds exploit this distance.

\begin{lemma}\label{lem:bias_expectation}
Consider two dictionaries $\Db,\Dbo \in \Real^{m \times p}$ and scalars $\LRIP,A,B$ such that
\begin{eqnarray}
A & \geq & \max\left\{\|\DbT\Db-\Ib\|_{\fro},\|\DboT\Dbo-\Ib\|_{\fro})\right\} 
\label{eq:FrobNormAssumption}\\
B & \geq & \max\left\{\triple \Db\triple_{2}, \triple \Dbo\triple_{2}\right\}
\label{eq:OpNormAssumption}\\
\LRIP & \geq & \max\{\LRIP_{k}(\Db),\LRIP_{k}(\Dbo)\}.\label{eq:RIPAssumption}
\end{eqnarray}
Then, with $\thetab = \thetab(\Dbo,\Db)$:
\begin{eqnarray}
\label{eq:leading_expectation}
\Exp_{\J} \trace  \DboJT  ( \Ib - \PJb ) \DboJ
&\geq& 
\tfrac{k}{p} \|\thetab\|_{2}^{2} (1-\tfrac kp \tfrac{B^{2}}{1-\LRIP}) \\
\label{eq:bias_expectation}
\left|\Exp_{\J}\trace \left(\Ib - \DbJ^{+}\DboJ\right)\right|
&\leq& 
\tfrac{k}{p} \tfrac{\|\thetab\|_{2}^{2}}{2} +   \tfrac{k^{2}}{p^{2}} \tfrac{AB}{1-\LRIP}  \|\thetab\|_{2}\\
\label{eq:bias_expectation_signsign}
\left|\Exp_{\J}\trace\left(\ThetaJbo-\ThetaJb\right)\right|
&\leq& 
\tfrac{k^{2}}{p^{2}} \tfrac{4AB}{(1-\LRIP)^2} \|\thetab\|_{2}.
\end{eqnarray}
\end{lemma}

Equipped with these lemmatas we first establish a lower bound on $\Delta \phi_{\PP}(\Db;\Dbo)$ for a fixed pair $\Db,\Dbo$.
\begin{lemma}\label{lem:maindeltaphi}
Consider two dictionaries $\Db,\Dbo \in \Real^{m \times p}$ and scalars $\LRIP,A,B$ such that~\eqref{eq:FrobNormAssumption}-\eqref{eq:OpNormAssumption}-\eqref{eq:RIPAssumption} hold.
Consider the basic signal model (Assumption~\ref{assume:basic}) and assume that the reduced regularization parameter satisfies
\begin{equation}
\label{eq:assumptionBlambda}
\frac{k}{p} \frac{B^{2}}{1-\LRIP} +\bar{\lambda} \kappa_{\alpha}^{2} \leq \frac{1}{2}.
\end{equation}
Then, we have the lower bound
\begin{equation}\label{eq:MainLowerBoundFixedPair}
\Delta \phi_{\PP}(\Db;\Dbo) 
\geq
\frac{\Exp\ \alpha^{2}}{4} \cdot \frac{k}{p} \cdot \|\Db-\Dbo\|_{\fro} \cdot \Big[\|\Db-\Dbo\|_{\fro}-r_{0}\Big].
\end{equation}
where
\begin{eqnarray}
r_{0} &\defin& (1+2\bar{\lambda}) \cdot \bar{\lambda} \kappa_{\alpha}^{2} \cdot \frac{k}{p} \cdot \frac{2AB}{(1-\LRIP)^{2}}.
\end{eqnarray}
\end{lemma}
\begin{proof}
Under the basic signal model, applying Proposition~\ref{prop:delta_phi} and Lemma~\ref{lem:bias_expectation}, yields the bound
\begin{equation*}
\begin{split}
\Delta \phi_{\PP}(\Db;\Dbo) 
\geq
\tfrac{\Exp\ \alpha^{2}}{2} \tfrac{k}{p} \|\thetab\|_{2}
\Big[
&\|\thetab\|_{2} 
\left(
1-\tfrac{k}{p} \tfrac{B^{2}}{1-\LRIP}
-\bar{\lambda} \kappa_{\alpha}^{2}
\right)\\
&-(1+2\bar{\lambda}) \cdot \bar{\lambda}\kappa_{\alpha}^{2} \cdot 
 \tfrac{k}{p} \cdot \frac{2AB}{(1-\LRIP)^{2}}
\Big].
\end{split}
\end{equation*}
By assumption~\eqref{eq:assumptionBlambda} it follows that
\begin{equation*}
\begin{split}
\Delta \phi_{\PP}(\Db;\Dbo) 
\geq
\tfrac{\Exp\ \alpha^{2}}{4}  & \tfrac{k}{p} \|\thetab\|_{2} \\
& \cdot \Big[
 \|\thetab\|_{2} 
-(1+2\bar{\lambda}) \cdot \bar{\lambda} \kappa_{\alpha}^{2} \cdot 
 \tfrac{k}{p} \cdot \tfrac{2AB}{(1-\LRIP)^{2}}
\Big].
\end{split}
\end{equation*}
We conclude using the fact that $\|\thetab\|_{2} \geq \|\Db-\Dbo\|_{\fro}$.
\end{proof}

We now show that the lower bound does not only hold for a given pair $\Db,\Dbo$: given $\Dbo$, we identify a radius $r$ such that $\Delta \phi_{\PP}(\Db;\Dbo)  > 0$ for any $\Db \in \Scal(r)$. This establishes Proposition~\ref{prop:maindeltaphi}.

\begin{proof}[Proof of Proposition~\ref{prop:maindeltaphi}]
Define the shorthands $A^{o} \defin \|\DboT\Dbo-\Ib\|_{\fro}$ and $B^{o}  \defin  \triple \Dbo\triple_{2}$. Consider $r \leq 1$ and $\Db \in \Scal(r)$, we have
\begin{eqnarray*}
\triple \Db\triple_{2} 
& \leq& 
\triple \Dbo \triple_{2} + \triple \Db-\Dbo\triple_{2} \\
&\leq & B^{o}+ \|\Db-\Dbo\|_{\fro} = B^{o}+r \leq B^{o}+1\\
\|\DbT\Db-\Ib\|_{\fro} 
&\leq& 
\|\DbT\Db-\DbT\Dbo\|_{\fro}\\
&& + \|\DbT-\DboT\|_{\fro}\cdot \triple \Dbo\triple_{2}\\
&& + \|\DboT\Dbo-\Ib\|_{\fro} \\
&\leq& 
\triple\DbT\triple_{2} \cdot \|\Db-\Dbo\|_{\fro} \\
&&+ \|\DbT\Dbo-\DboT\Dbo\|_{\fro} +A^{o} \\
& \leq & A^{o}+2Br  = A^{o}+2(B^{o}+1)r
\end{eqnarray*}
By construction, assumptions~\eqref{eq:FrobNormAssumption}-\eqref{eq:OpNormAssumption} of Lemma~\ref{lem:maindeltaphi} therefore hold with $B \defin B^{o}+1$ and $A \defin A^{o} + 2Br$. 
Denoting $\LRIP \defin 1/2$ and $\LRIP^{o} \defin \LRIP_{k}(\Dbo)$ we have $\LRIP^{o} \leq \LRIP < 1$.  Since $\bar{\lambda} \leq 1/4$ and $\kappa_{\alpha} \leq 1$, assumptions~\eqref{eq:DefLRIPMax} and~\eqref{eq:DefOpNormMax} imply that
\begin{eqnarray*}
\frac kp \frac{B^{2}} {1-\LRIP} + \bar{\lambda} \kappa_{\alpha}^{2} & \leq & \frac kp 2B^{2} +  \bar{\lambda} \leq  \frac 12.
\end{eqnarray*}
showing that assumption~\eqref{eq:assumptionBlambda} of Lemma~\ref{lem:maindeltaphi} also holds. Now we observe that\begin{eqnarray*}
0.15 \leq \sqrt{3/4}-\sqrt{1/2} \leq \sqrt{1-\LRIP^{o}}-\sqrt{1-\LRIP}.
\end{eqnarray*}
Hence, by Lemma~\ref{lem:RIPBounds}, when $r \leq 0.15$ we have $\sqrt{1-\LRIP_{k}(\Db)} \geq \sqrt{1-\LRIP^{o}}-r \geq \sqrt{1-\LRIP}$, that is to say the remaining assumption~\eqref{eq:RIPAssumption} of Lemma~\ref{lem:maindeltaphi} holds with $\LRIP$, and we can leverage Lemma~\ref{lem:maindeltaphi}. This yields
\[
\Delta \phi_{\PP}(\Db;\Dbo) 
\geq
\tfrac{\Exp\ \alpha^{2}}{4} \cdot \tfrac{k}{p}  \cdot \|\Db-\Dbo\|_{\fro} \cdot \Big[\|\Db-\Dbo\|_{\fro}-\tfrac{\gamma}{2} \tfrac AB\Big]
\]
with
\begin{eqnarray*}
\gamma &\defin& (1+2\bar{\lambda}) \cdot \bar{\lambda}\kappa_{\alpha}^{2} \cdot \tfrac kp \tfrac{4 B^{2}}{(1-\LRIP)^{2}}\\
&=&(1+2\bar{\lambda}) \cdot \bar{\lambda}\kappa_{\alpha}^{2} \cdot \tfrac kp \cdot 16 B^{2}
\leq \left(1 + 2\bar{\lambda}\right)  \cdot \bar{\lambda}\kappa_{\alpha}^{2}  \leq \tfrac {3}{8}
\end{eqnarray*}
where we used~\eqref{eq:DefOpNormMax} once more. Since $\gamma \leq 3/8 \leq 1/2$ and $\|\Db-\Dbo\|_{\fro}=r$ we have
\begin{eqnarray*}
\|\Db-\Dbo\|_{\fro}-\tfrac{\gamma}{2} \tfrac AB
&=&
r-\gamma \tfrac{(A^{o}+2Br)}{2B}
=
r(1-\gamma )-\gamma \tfrac{A^{o}}{2B}\\
&\geq& \tfrac{1}{2}\left(r-\gamma \tfrac{A^{o}}{B}\right).
\end{eqnarray*}
To summarize, noticing that 
\[
\gamma  \frac{A^{o}}{B} = (1+2\bar{\lambda})\bar{\lambda} \cdot 16 \kappa_{\alpha}^{2} \frac{k}{p} B A^{o} = \tfrac{2}{3} C_{\min} (1+2\bar{\lambda})\bar{\lambda}  = r_{\min}(\bar{\lambda})
\]
we have shown that holds~\eqref{eq:MainLowerBound} for any $r \leq 0.15$ and $\Db \in \Scal(r;\Dbo)$.
Finally, since $1+2\bar{\lambda} \leq 3/2$, the assumption that $\bar{\lambda} < \frac{3}{20 C_{\min}}$ implies that
\(
r_{\min}(\bar{\lambda}) 
 = \left(1 + 2\bar{\lambda}\right)  \cdot \bar{\lambda} \cdot \frac{2}{3} C_{\min}  < 0.15.
\)
\end{proof}

\subsection{Control of $h(\Db)$}\label{sec:boundphi}
To obtain finite sample results in Section~\ref{sec:proofthmmainfinite}, we need to control 
\(
h(\Db) = \Delta \phi_{\signal}(\Db;\Dbo|\sbo) -\Delta\phi_{\alphab,\alphab}(\Db;\Dbo).
\)
\begin{lemma}\label{lem:maindeltaphi}\label{le:LipschitzAndBoundDeltaPhi}
Consider a dictionary $\Dbo \in \Real^{m \times p}$, $k$ and $r>0$ such that $r < \sqrt{1-\LRIP_{k}(\Dbo)}$, and define
\[
\sqrt{1-\LRIP} \defin \sqrt{1-\LRIP_{k}(\Dbo)}-r > 0.
\]
Under the Bounded signal model (Assumption~\ref{assume:bounded}), 
the function $h(\Db)$ is almost surely Lipschitz on $\Bcal(\Dbo;r)$ with respect to the Frobenius metric $\rho(\Db',\Db) \defin \|\Db'-\Db\|_{\fro}$, with Lipschitz constant upper bounded by
\begin{equation}
\label{eq:MainLipBoundFixedPair}
\begin{split}
L \defin &  \tfrac{1}{\sqrt{1-\LRIP}}  \left(M_{\varepsilonb} + \tfrac{\lambda \sqrt{k}}{\sqrt{1-\LRIP}}\right)  \\
& \cdot
\left(
2\sqrt{1+\URIP_{k}(\Dbo)}M_{\alphab} + M_{\varepsilonb} + \tfrac{\lambda \sqrt{k}}{\sqrt{1-\LRIP}}
\right) 
\end{split}
\end{equation}
As a consequence, $|h(\Db)| = |h(\Db)-h(\Dbo)|$ is almost surely bounded on $\Bcal(\Dbo;r)$ by $c \defin Lr$.
\end{lemma}

\begin{proof}
Denote $\rho = \|\Db'-\Db\|_{\fro}$. Combining Lemma~\ref{lem:RIPBounds} and Lemma~\ref{le:LipBoundsRIP} we bound the operator norms of the matrix differences appearing in the terms $\Delta\phi_{\alphab,\varepsilonb}(\Db;\Db')$ to $\Delta\phi_{\sb,\sb}(\Db;\Db')$ in Equation~\eqref{eq:DevelDeltaPhi1}:
\begin{eqnarray*}
|\Delta\phi_{\alphab,\varepsilonb}(\Db;\Db')|
&\leq& 
M_{\varepsilonb} \cdot \triple \PJb' - \PJb \triple_{2} \cdot  \sqrt{1+\URIP_{k}(\Dbo)} M_{\alphab} \\
& \leq& \rho \cdot 2M_{\varepsilonb} \sqrt{1+\URIP_{k}(\Dbo)} M_{\alphab} \cdot (1-\LRIP)^{-1/2}\\
|\Delta\phi_{\sb,\alphab}(\Db;\Db')|
&\leq&
\lambda\sqrt{k}  \triple[\Db'_{\J}]^{+} - \Db_{\J}^+  \triple_{2} \sqrt{1+\URIP_{k}(\Dbo)} M_{\alphab}\\
&\leq& \rho \cdot 2\lambda \sqrt{k}  \sqrt{1+\URIP_{k}(\Dbo)} M_{\alphab}  \cdot  (1-\LRIP)^{-1}\\
|\Delta\phi_{\varepsilonb,\varepsilonb}(\Db;\Db')|
&\leq&
\tfrac{1}{2} \triple \PJb'-\PJb\triple_{2} \cdot M_{\varepsilonb}^2\\
&\leq& \rho \cdot M_{\varepsilonb}^{2} (1-\LRIP)^{-1/2}\\
|\Delta\phi_{\sb,\varepsilonb}(\Db;\Db')|
&\leq&
\lambda \sqrt{k} \cdot \triple [\Db'_{\J}]^{+} - \Db_{\J}^+\triple_{2} \cdot  M_{\varepsilonb}\\
&\leq& \rho \cdot  2\lambda \sqrt{k} M_{\varepsilonb} \cdot (1-\LRIP)^{-1}\\
|\Delta\phi_{\sb,\sb}(\Db;\Db')|
&\leq&
\tfrac{\lambda^2}{2}  \cdot \triple \ThetaJbo  -  \ThetaJb\triple_{2} \cdot k\\
&\leq& 
\rho \cdot  \lambda^2 k \cdot  (1-\LRIP)^{-3/2}
\end{eqnarray*}
Since $h(\Db)-h(\Db') = \Delta \phi_{\signal}(\Db;\Db'|\sbo) -\Delta\phi_{\alphab,\alphab}(\Db;\Db')$, we obtain the desired bound on the Lipschitz constant by summing the right hand side of the above inequalities. To conclude, observe that $h(\Dbo) = 0$.
\end{proof}

%% file: signrecovery.tex

\section{Sign pattern recovery: proof of Proposition~\ref{prop:simplified_exact_recovery}}\label{sec:recovery}
\begin{lemma}\label{lem:CoherenceBounds}
Consider $\Db \in \Real^{m\times p}$ with normalized columns and $\J \subseteq \SET{p}$ with $|\J| \leq k$. We have
$
\delta_{k}(\Db) \leq \mu_{k-1}(\Db)
$
hence 
$
\triple \DbJT\DbJ-\Ib\triple_2 \leq \mu_{k-1}(\Db), \quad \text{and}\quad
\triple\DbJ\DbJT \triple_2 =\triple\DbJT \DbJ \triple_2 \leq 1+\mu_{k-1}(\Db)  .
$ 
Similarly, it holds 
$
\triple \DbJT\DbJ-\Ib\triple_\infty \leq \mu_{k-1}(\Db), \quad
\triple\DbJT \DbJ \triple_\infty \leq 1+\mu_{k-1}(\Db)\quad \text{and}\quad \triple\Db_{\J^c}^\top \DbJ \triple_\infty \leq \mu_{k}(\Db).
$
If we further assume $\mu_{k-1}(\Db)<1$, then $\ThetaJb = (\DbJT\DbJ)^{-1}$ is well-defined and
\begin{equation*}
\begin{split}
\max &\Big\{ \triple \ThetaJb-\Ib \triple_\infty, \triple \ThetaJb-\Ib \triple_2,
 \triple \Db_{\J^c}^\top \DbJ \big( \DbJT \DbJ \big)^{-1} \triple_\infty 
\Big\} \\
& \leq \frac{\mu_{k}}{1-\mu_{k-1}},
\end{split}
\end{equation*}
along with
$
\max\Big\{ \triple \ThetaJb \triple_\infty, \triple \ThetaJb \triple_2  \Big\} \leq \frac{1}{1-\mu_{k-1}}.
$
\end{lemma}
\begin{proof}
These properties are already well-known \citep[see, e.g.][]{Tropp2004,Fuchs2005}. 
\end{proof}

\begin{lemma}\label{lem:uniquelasso}
Consider a dictionary $\Db \in \Real^{m\times p}$, a support set $\J \subseteq \SET{p}$ such that $\Db_\J^\top\Db_\J$ is invertible, a sign vector $\sb \in \{-1,1\}^{\J}$, and $\xb \in \Real^m$ a signal. 
If the following two conditions hold
\[
\begin{cases}
\sign\Big( \DbJ^+ \xb -\lambda (\DbJT\DbJ)^{-1}\sb  \Big) = \sb,\\
\| \Db_{\J^c}^\top (\Ib - \PJb) \xb \|_\infty + \lambda \triple \Db_{\J^c}^\top \DbJ (\DbJT\DbJ)^{-1} \triple_\infty < \lambda,
\end{cases}
\]
then
$\hat{\alphab}_{\signal}(\Db|\sb)$ is the unique solution of $\min_{\alphab \in \Real^p} [\frac{1}{2}\|\xb-\Db\alphab\|_2^2+\lambda\|\alphab\|_1]$
and we have $\sign(\hat{\alphab}_\J)=\sb$.
\end{lemma}
\begin{proof}
We first check that $\hat{\alphab}$ is a solution of the Lasso program.
It is well-known \citep[e.g., see][]{Fuchs2005,Wainwright2009} that this statement is equivalent to the existence of a subgradient $\zb \in \partial \|\hat{\alphab}\|_1$
such that $-\Db^\top(\xb-\Db\hat{\alphab})+\lambda\zb=0$, where $\zb_j=\sign(\hat{\alphab}_j)$ if $\hat{\alphab}_j \neq 0$, and $|\zb_j|\leq 1$ otherwise.
We now build from $\sb$ such a subgradient. Given the definition of $\hat{\alphab}$ and the assumption made on its sign, we can take $\zb_\J\defin\sb$.
It now remains to find a subgradient on $\J^c$ that agrees with the fact that $\hat{\alphab}_{\J^c}=\zerob$.
More precisely, we define $\zb_{\J^c}$ by
\begin{equation}
\lambda \zb_{\J^c} \defin \Db_{\J^c}^\top(\xb-\Db\hat{\alphab}) = 
\Db_{\J^c}^\top (\Ib - \PJb) \xb + \lambda \Db_{\J^c}^\top \DbJ (\DbJT\DbJ)^{-1} \sb.
\end{equation}
Using our assumption, we have $\|\zb_{\J^c}\|_\infty < 1$.
We have therefore proved that $\hat{\alphab}$ is a solution of the Lasso program. The uniqueness comes from \citep[Lemma~1]{Wainwright2009}.
\end{proof}

\begin{lemma}\label{lem:alphabound} 
Consider $\xb=\DboJ{\alphabo}_\J + \varepsilonb$ for some $(\Dbo,\alphabo,\varepsilonb)\in \Real^{m\times p} \times \Real^p \times \Real^m$, $\sbo$ the sign of $\alphabo$ and $\J$ its support. 
Consider a dictionary $\Db \in \Real^{m\times p}$ such that $\Db_\J^\top\Db_\J$ is invertible.
We have
\begin{equation*}
\begin{split}
\|[\hat{\alphab}_{\signal}(\Db|\sbo)&-\alphabo]_\J\|_\infty
\\
& \leq \triple [\Db_\J^\top\Db_\J]^{-1} \triple_\infty 
\Big[ \lambda + \| \Db_\J^\top\left(\xb-\Db \alphabo\right) \|_\infty \Big].
\end{split}
\end{equation*}
\end{lemma}

\begin{proof}
The proof consists of simple algebraic manipulations. We plug the expression of $\xb$ into that of $\hat{\alphab}$, then use the triangle inequality for $\|.\|_\infty$, along with the definition and the sub-multiplicativity of $\triple.\triple_\infty$.
\end{proof}

\begin{lemma}\label{le:robustlassomin}
Assume that $\mu_{k}(\Db) \leq \mu_{k} < 1/2$. If  
\begin{eqnarray}
\label{eq:Cor3Bound0}
\min_{j \in \J}|[\alphabo]_{j}|
& \geq & 2\lambda\\
\label{eq:Cor3Bound1}
\|\xb-\Db \alphabo\|_{2} 
&< &
\lambda(1-2\mu_{k}).
\end{eqnarray}
then $\hat{\alphab}_{\signal}(\Db|\sbo)$ is the unique solution of $\min_{\alphab \in \Real^p} [\frac{1}{2}\|\xb-\Db\alphab\|_2^2+\lambda\|\alphab\|_1]$.
\end{lemma}

\begin{proof}
Since $\|\db^{j}\|_{2}=1$ for all $j$, we have by assumption~\eqref{eq:Cor3Bound1}:
\begin{eqnarray}
\label{eq:Prop2Bound1}
\|\Db_{\J}^{\top} (\xb-\Db \alphabo)\|_{\infty} 
&\leq & \!\!\!
\|\xb-\Db \alphabo\|_{2} 
< \lambda(1-2\mu_{k})\\
\label{eq:Prop2Bound2}
\| \Db_{\J^c}^\top (\Ib - \PJb) \xb \|_\infty 
&\leq&  \!\!\! \|(\Ib - \PJb) \xb  \|_2\notag\\
 &\leq& \!\!\! \|\signal - \Db \alphabo\|_{2}
< \lambda(1-2\mu_{k})
\end{eqnarray}
where we use the fact that by definition of the orthogonal projector $\PJb$ on the span of $\Db_{\J}$, the vector $\PJb\signal$ is a better approximation to $\signal$ than $\Db \alphabo = \Db_{\J} \alphabo_{\J}$.
By Lemma~\ref{lem:CoherenceBounds} we have
\begin{eqnarray*}
\triple \big(\DbJT \DbJ \big)^{-1} \triple_\infty
&\leq &
\frac{1}{1-\mu_{k-1}(\Db)} \leq \frac{1}{1-\mu_{k}}
\\
 \triple \Db_{\J^c}^\top \DbJ \big( \DbJT\DbJ \big)^{-1} \triple_\infty 
&\leq& \frac{\mu_{k}}{1-\mu_{k-1}(\Db)} \leq \frac{\mu_{k}}{1-\mu_{k}} < 1
\end{eqnarray*}
Exploiting Lemma~\ref{lem:alphabound} and the bounds~\eqref{eq:Cor3Bound0} and~\eqref{eq:Prop2Bound1} we have 
\begin{equation*}
\begin{split}
\|[ \hat{\alphab} &- \alphabo ]_\J\|_\infty\\
 &\leq  
\triple \big(\DbJT \DbJ \big)^{-1} \triple_\infty 
\Big[ \lambda + \|\DbJT(\xb-\Db\alphabo) \|_\infty\Big] \\
& < 
\frac{1}{1-\mu_{k}}
 \cdot \lambda \cdot \left[1+(1-2\mu_{k})\right] 
= 
2\lambda  \leq \min_{j\in\J} \big|[\alphabo]_j\big|,
\end{split}
\end{equation*}
We conclude that $\sign( \hat{\alphab}) = \sign(\alphabo )$. There remains to prove that $\hat{\alphab}$ is the unique solution of the Lasso program, using Lemma~\ref{lem:uniquelasso}. We recall the quantity which needs to be smaller than $\lambda$
\[
\| \Db_{\J^c}^\top (\Ib - \PJb) \xb \|_\infty + \lambda \triple \Db_{\J^c}^\top \DbJ \big( \DbJT \DbJ\big)^{-1} \triple_\infty.
\]
The quantity above is first upper bounded by
\[
\| \Db_{\J^c}^\top (\Ib - \PJb) \xb \|_\infty + \lambda \mu_{k}/(1-\mu_{k}),
\]
and then, exploiting the bound~\eqref{eq:Prop2Bound2},  upper bounded by $\lambda(1-2\mu_{k}) + \lambda \mu_{k}/(1-\mu_{k}) < \lambda$.
Putting together the pieces with $\sign( \hat{\alphab} ) = \sign(\alphabo )$, Lemma~\ref{lem:uniquelasso} leads to the desired conclusion.
\end{proof}

\begin{proof}[Proof of Proposition~\ref{prop:simplified_exact_recovery}]
First observe that almost surely
\begin{eqnarray*}
\|\xb-\Db \alphabo\|_{2} 
& = & 
\| [\Dbo-\Db]_\J [\alphabo]_\J \|_{2} +\|\varepsilonb\|_{2} \\
&\leq & \triple[\Db-\Dbo]_\J\triple_{2} \cdot \| [\alphabo]_\J \|_{2} + M_{\varepsilonb}\\
&\leq & \|\Db-\Dbo\|_\fro \cdot M_{\alphab} + M_{\varepsilonb}.
\end{eqnarray*}
Now, since $\mu_{k-1}(\Dbo) \leq \mu_{k}^{o} \leq 1/2$, using Lemma~\ref{lem:cumcoherball} below and the shorthand $r = \|\Db-\Dbo\|_{\fro}$, we have
\begin{equation*}
\begin{split}
 \lambda (& 1-2\mu_{k}(\Db)) - \|\Db-\Dbo\|_{\fro} M_{\alphab} - M_{\varepsilonb}\\
&\geq \lambda(1-2\mu_{k}^{o})  - M_{\alphab} r - 2\lambda \sqrt{k}  (2+1/2)  r - M_{\varepsilonb} \\
& \geq   \lambda(1-2\mu_{k}^{o}) - \left(M_{\alphab} + 5\lambda \sqrt{k}\right)  r  - M_{\varepsilonb} \\
& \geq  \lambda(1-2\mu_{k}^{o}) - \tfrac{7}{2} M_{\alphab}  r - M_{\varepsilonb} = \tfrac{7}{2} M_{\alphab} \left(C_{\max} \bar{\lambda}-r\right)- M_{\varepsilonb}
\end{split}
\end{equation*}
where we used  $\lambda \sqrt{k} \leq \frac{\loweralpha}{2} \sqrt{k} \leq \frac{M_{\alphab}}{2}$. 
For $r < C_{\max} \cdot \bar{\lambda}$, the assumption on the noise level implies that $\|\Db-\Dbo\|_{\fro}M_{\alphab}  + M_{\varepsilonb}  < \lambda(1-2\mu_{k}(\Db))$, hence we can apply Lemma~\ref{le:robustlassomin}. We conclude by observing that the result applies in particular to $\Db = \Dbo$.
\end{proof}

\begin{lemma}\label{lem:cumcoherball}
Consider $\Db,\Dbo \in \Real^{m\times p}$ with normalized columns such that  
$\|\Db-\Dbo\|_{\fro} \leq r$. For $k \leq p$ we have
\begin{equation}
\mu_{k}(\Db) \leq \mu_{k}(\Dbo) + \sqrt{k} \cdot r \cdot [2+\mu_{k-1}(\Dbo)].
\end{equation}
\end{lemma}
\begin{proof}[Proof of Lemma~\ref{lem:cumcoherball}]
Consider $\J \subseteq \SET{p}$ with $|\J| \leq k$ and $j \notin J$. By the triangle inequality 
\begin{equation*}
\begin{split}
\|\DbJT &\db^{j}\|_{1} \\
 \leq & \ \|\DboJT [\dbo]^{j}\|_{1} 
+ 
\|\DboJT(\db^{j}-[\dbo]^{j})\|_{1}\\
&+
\|(\DbJ-\DboJ)^{\top} \db^{j}\|_{1}\\
 \leq & \
\mu_{k}(\Dbo)
+
\sqrt{k} 
\|\DboJT(\db^{j}-[\dbo]^{j})\|_{2}\\
&+
\sqrt{k} \|(\DbJ-\DboJ)^{\top} \db^{j}\|_{2}
\\
 \leq & \
\mu_{k}(\Dbo)
+
\sqrt{k} 
\sqrt{1+\mu_{k-1}(\Dbo)} \cdot \|\db^{j}-[\dbo]^{j}\|_{2}\\
&+
\sqrt{k} \triple (\DbJ-\DboJ)^{\top}\triple_{2}\\
 \leq & \ 
\mu_{k}(\Dbo)
+
\sqrt{k}  r 
\left[
1+\mu_{k-1}(\Dbo) +1
\right].
\end{split}
\end{equation*}
\end{proof}

%% file: localDL_appendix.tex

\subsection{Proof of Lemma~\ref{lem:RIPBounds}}\label{app:RIP} 

By definition of $\LRIP_{k}(\Db)$ we have, in the sense of symmetric positive definite matrices:
\(
\big(1-\LRIP_{k}(\Db)\big) \cdot \Ib \preceq \Db_{\J}^{\top}\Db_{\J}. 
\)
As a result, $\Db_{\J}^{\top}\Db_{\J}$ is invertible so $\ThetaJb$ is indeed well defined, and
\(
\triple\ThetaJb\triple_{2} 
=
\triple(\Db_{\J}^\top \Db_{\J})^{-1} \triple_{2} 
\leq 1/(1-\LRIP_{k}(\Db))
\).
Moreover
\(
\triple \DbJ^{+} \triple_2 
=
\triple  \ThetaJb \DbJT \triple_2
=  \sqrt{ \triple \ThetaJb \Db_{\J}^\top \Db_{\J}\ThetaJb \triple_2 }
= \sqrt{ \triple \ThetaJb  \triple_2 } \leq \frac{1}{\sqrt{1-\LRIP_{k}(\Db)}}.
\)

Consider now $\Db'$. By the triangle inequality 
for any $\J$ of size $k$ and $\zb \in \Real^{\J}$ we have
\begin{eqnarray*}
\|\DbJ'\zb\|_{2} 
& \geq & \|\DbJ \zb\|_{2} - \|[\DbJ'-\DbJ]\zb\|_{2} \\
&\geq& \big(\sqrt{1-\LRIP_{k}(\Db)} - r\big) \cdot \|\zb\|_{2} = \sqrt{1-\LRIP} \|\zb\|_{2}.
\end{eqnarray*}
where we used the fact that $\triple \DbJ'-\DbJ\triple_{2} \leq \|\DbJ'-\DbJ\|_{\fro} \leq \|\Db'-\Db\|_{\fro}$. 

\subsection{Proof of Lemma~\ref{le:LipBoundsRIP}}

The assumptions combined with Lemma~\ref{lem:RIPBounds} yield
\begin{eqnarray*}
\max(\triple \ThetaJb\triple_{2},\triple \ThetaJb'\triple_{2}) & \leq & (1-\LRIP)^{-1}\\
\max(\triple \DbJ^{+}\triple_{2},\triple [\DbJ']^{+}\triple_{2}) & \leq & (1-\LRIP)^{-1/2}.
\end{eqnarray*}
Moreover, denoting $r = \|\Db-\Db'\|_{\fro}$, we have\\
 $\triple \DbJ-\DbJ'\triple_{2} \leq \|\DbJ-\DbJ'\|_{\fro} \leq r$. It follows that
\begin{eqnarray*}
\triple\Ib-\DbJ^{+}\DbJ'\triple_{2} 
&=&
\triple \DbJ^{+}(\DbJ-\DbJ')\triple_{2}
\leq
\triple \DbJ^{+}\triple_{2}  r \\
&\leq&
r (1-\LRIP)^{-1/2}\\
 \ThetaJb'  -  \ThetaJb 
& = &
 \ThetaJb'(\DbJT\DbJ-[\DbJ']^{\top}\DbJ\\
 &&+[\DbJ']^{\top}\DbJ-[\DbJ']^{\top}\DbJ')\ThetaJb\\
&=& 
 \ThetaJb'(\DbJT-[\DbJ']^{\top})\DbJ\ThetaJb\\
 &&+ \ThetaJb'[\DbJ']^{\top}(\DbJ-\DbJ')\ThetaJb  \\
\triple \ThetaJb'  -  \ThetaJb\triple_{2} 
& \leq & 
\triple \ThetaJb' \triple_{2}  r  \triple [\DbJT]^{+}\triple_{2}\\
&& + \triple [\DbJ']^{+}\triple_{2}  r  \triple \ThetaJb \triple_{2}\\
&\leq&
2r (1-\LRIP)^{-3/2} \\
(\ThetaJb'  -  \ThetaJb)\DbJT
& = &
\ThetaJb'(\DbJT-[\DbJ']^\top)\DbJ\ThetaJb\DbJT\\
&& + \ThetaJb'[\DbJ']^\top(\DbJ-\DbJ')\ThetaJb\DbJT \\
&= & 
\ThetaJb'(\DbJT-[\DbJ']^\top)\PJb\\
&&+ [\DbJ']^{+}(\DbJ-\DbJ')\DbJ^{+}\\
{ [\DbJ']^{+}}-\DbJ^{+}
& = &
 \ThetaJb'([\DbJ']^\top-\DbJT)\\
&&+
 (\ThetaJb'-\ThetaJb)\DbJT\\
 & = &
\ThetaJb'([\DbJ']^\top-\DbJT)(\Ib-\PJb)\\
&&+ [\DbJ']^{+}(\DbJ-\DbJ')\DbJ^{+}\\
\triple [\DbJ']^{+}-\DbJ^{+}\triple_{2}
& \leq & 
\triple \ThetaJb' \triple_{2}  r + \triple [\DbJ']^{+}\triple_{2}  r  \triple \DbJ^{+} \triple_{2}\\
& \leq & 2r (1-\LRIP)^{-1}\\
\PJb'-\PJb & = &
\DbJ'([\DbJ']^{+}-\DbJ^{+})+(\DbJ'-\DbJ)\DbJ^{+}\\
&=&
\DbJ'\ThetaJb'([\DbJ']^\top-\DbJT)(\Ib-\PJb)\\
&&+ \DbJ'[\DbJ']^{+}(\DbJ-\DbJ')\DbJ^{+}\\
&&+(\DbJ'-\DbJ)\DbJ^{+}\\
&=&
([\DbJ']^{+})^{\top}([\DbJ']^\top-\DbJT)(\Ib-\PJb)\\
&&+ (\Ib-\PJb')(\DbJ'-\DbJ)\DbJ^{+}\\
\triple \PJb'-\PJb\triple_{2} & \leq &
\triple[\DbJ']^{+}\triple_{2}  r  + r  \triple\DbJ^{+}\triple_{2}\\
& \leq & 2 r (1-\LRIP)^{-1/2}.
\end{eqnarray*}

\subsection{Proof of Lemma~\ref{lem:paramonto}}\label{app:paramonto}
Each column $\db^j_{2}$ of $\Db_{2}$ can be uniquely expressed as
\(
\db^j_{2} = \ub + \zb,\ \text{with}\ \ub \in \text{span}(\db^j_{1})\ \text{and}\ \ub^\top \zb = 0.
\)
Since $\|\db^j_{2}\|_2=1$, the previous relation can be rewritten as 
\(
\db^j_{2} = \cos(\thetab_j)\db^j_{1} + \sin(\thetab_j) \wb^j, 
\)
for some $\thetab_j \in [0,\pi]$ and some unit vector $\wb^j$ orthogonal to $\db^{j}_{1}$ (except for the case $\thetab_j \in \{0,\pi\}$, the vector $\wb^j$ is unique). The sign indetermination in $\wb^j$ is handled thanks to the convention $\sin(\thetab_j)\geq 0$. 
We have $\|\thetab\|_{\infty} \leq \pi$ and
\begin{eqnarray*}
\|\db^{j}_{2} - \db^{j}_{1}\|_2^{2}  &=&  \| (1-\cos(\thetab_{j}))\db^j - \sin(\thetab_{j})\wb^{j} \|_2^2\\
&=& ( 1 -  \cos(\thetab_{j}) )^{2} + \sin^{2}(\thetab_{j})\\
&=& 2(1-\cos (\thetab_{j})) = 4 \sin^{2} (\thetab_{j}/2).
\end{eqnarray*}
We conclude using the inequalities $\frac{2}{\pi} \leq \frac{\sin u}{u} \leq 1$ for $0 \leq u \leq \pi/2$. 
The result when we interchange $\Db_{1}$ and $\Db_{2}$ is obvious, and $\thetab(\Db_{1},\Db_{2}) = \thetab(\Db_{2},\Db_{1})$ since $\|\db_{1}^{j}-\db_{2}^{j}\|_{2} = \|\db_{2}^{j}-\db_{1}^{j}\|_{2}$ for all $j$.

\subsection{Proof of Lemma~\ref{lem:bias_expectation}}\label{app:bias-expectation}

The proof of Lemma~\ref{lem:bias_expectation} will exploit the following lemmata.
\begin{lemma}\label{lem:ExpectOverJ}
Let $\J \subset{p}$ be a random support and denote by $\delta(i) \defin \indicator{\J}(i)$ the indicator function of $\J$. Assume that for all $i \neq j \in \SET{p}$
\begin{eqnarray*}
\Exp\{\delta(i)\} &=& \frac{k}{p}\\
\Exp\{\delta(i)\delta(j)\} &=& \frac{k(k-1)}{p(p-1)}.
\end{eqnarray*}
Then for any integer $m$ and matrices $\Ab,\Bb \in \Real^{m \times p}$ such that $\diag(\Ab^{\top}\Bb)=0$, we have
\begin{eqnarray}
\label{eq:ExpectOverJNormvJ}
\Exp_{\J} \{\|\Ab_{\J}\|_{\fro}^{2}\}
&=&
\frac{k}{p} \|\Ab\|_{\fro}^{2}\\
\label{eq:ExpectOverJRIPBoundGramDJOffdiag}
\Exp_{\J} \{\|\Ab_{\J}^{\top}\Bb_{\J}\|_{\fro}^{2}\}
&=&  \frac{k(k-1)}{p(p-1)} \cdot \|\Ab^{\top}\Bb\|_{\fro}^{2} 
\end{eqnarray}
\end{lemma}
\begin{proof}
We simply expand
\begin{eqnarray*}
\Exp_{\J}\|\Ab_{\J}\|_{\fro}^{2}
&=&
\Exp\sum_{j \in \SET{p}} \delta(j) \cdot \|\Ab_{\{j\}}\|_{2}^{2}\\
&=&
\sum_{j \in \SET{p}} \tfrac{k}{p} \|\Ab_{\{j\}}\|_{2}^{2}\\ 
\Exp_{\J} \|\Ab_{\J}^{\top}\Bb_{\J}\|_{\fro}^{2} 
&=&
\Exp \sum_{i \in \SET{p}} \sum_{j \in \SET{p},j\neq i} \delta(i)\delta(j) \cdot \Ab_{\{i\}}^{\top}\Bb_{\{j\}}\\
&=&
\sum_{i \in \SET{p}} \sum_{j \in \SET{p}, j \neq i} \tfrac{k(k-1)}{p(p-1)} \cdot \Ab_{\{i\}}^{\top}\Bb_{\{j\}}.
\end{eqnarray*}
\end{proof}

\begin{lemma}\label{lem:LRIP}
Assume that 
\begin{eqnarray*}
\LRIP & \geq & \max\left\{\LRIP_{k}(\Db),\LRIP_{k}(\Dbo)\right\} \\
A & \geq & \max\left\{\|\DbT\Db-\Ib\|_{\fro},\|\DboT\Dbo-\Ib\|_{\fro})\right\}
\end{eqnarray*}
with $\LRIP < 1$, and define
$
\Ub_{k} \defin 
\Exp_\J \big[ \Ib_\J \ThetaJb \Ib_\J^\top \big]
$
and
$
\Vb_{k} \defin  
\Exp_\J \big[ \Ib_\J \ThetaJb\ThetaJbo \Ib_\J^\top \big]
$
where the expectation is taken over all supports $\J$ of size $k$ drawn uniformly at random.
Then we have
\begin{eqnarray}
\|\mathrm{off}(\Ub_{k})\|_{\fro} &\leq& \frac{k(k-1)}{p(p-1)} \frac{A}{1-\LRIP}\\
\|\mathrm{off}(\Vb_{k})\|_{\fro} &\leq& \frac{k(k-1)}{p(p-1)} \frac{2A}{(1-\LRIP)^2}.
\end{eqnarray}
\end{lemma}

\begin{proof}
Since $\ThetaJb\GramJb=\Ib$, using the RIP assumption and Lemma~\ref{lem:RIPBounds} we obtain
\begin{eqnarray*}
\| \mathrm{off}(\ThetaJb)\|_{\fro}
&\leq&
\|\ThetaJb-\Ib\|_{\fro}
= 
\|\ThetaJb(\Ib-\GramJb)\|_{\fro}\\
&\leq&
\frac{1}{1-\LRIP} \|\Ib-\GramJb\|_{\fro}\\
\| \mathrm{off}(\ThetaJb\ThetaJbo)\|_{\fro}
&\leq&
\|\ThetaJb\ThetaJbo-\Ib\|_{\fro}\\
&=& 
\|(\ThetaJb-\Ib)\ThetaJbo+(\ThetaJbo-\Ib)\|_{\fro}\\
&\leq&
\frac{1}{1-\LRIP}\|\ThetaJb-\Ib\|_{\fro}+\|\ThetaJbo-\Ib\|_{\fro}\\
&\leq &
\frac{1}{(1-\LRIP)^2} \|\Ib-\GramJb\|_{\fro}+\frac{1}{1-\LRIP}\|\Ib-\GramJbo\|_{\fro}\\
\end{eqnarray*}
In the following, $\KJb$ denotes either $\ThetaJb$ or $\ThetaJb\ThetaJbo$, and $\Wb_{k}$ either $\Ub_{k}$ or $\Vb_{k}$.
For any $J,J'$ of size $k$, we denote $\KJb^{J\cap J'}$ the restriction of $\KJb$ to the pairs of indices in $J \cap J'$, i.e. $\KJb^{J\cap J'} = \Ib^{\top}_{J \cap J'} \KJb \Ib_{J \cap J'}$, where we recall that $\Ib_{J \cap J'}$ is the restriction of the $p\times p$ identity matrix $\Ib$ to its columns indexed by $J \cap J'$. We obtain
\begin{equation*}
\begin{split}
\|\mathrm{off}&(\Wb_{k})\|_{\fro}^{2} \\
& = \|\mathbb{E}_{J} \mathrm{off}(\KJb)\|_{\fro}^{2}
= \langle \mathbb{E}_{J} \mathrm{off}(\KJb), \mathbb{E}_{J'} \mathrm{off}(\Kb_{\J'})\rangle_{\fro}\\
&= \mathbb{E}_{J,J'} \langle \mathrm{off}(\KJb),  \mathrm{off}(\Kb_{\J'})\rangle_{\fro}\\
& = \mathbb{E}_{J,J'} \langle \mathrm{off}(\KJb^{J\cap J'}), \mathrm{off}(\Kb_{\J'}^{J\cap J'})\rangle\\
& \leq \mathbb{E}_{J,J'} \| \mathrm{off}(\KJb^{J\cap J'})\|_{\fro} \cdot \| \mathrm{off}(\Kb_{\J'}^{J\cap J'})\|_{\fro}\\
&\leq \sqrt{ \mathbb{E}_{J,J'} \| \mathrm{off}(\KJb^{J\cap J'})\|_{\fro}^{2}} \cdot \sqrt{ \mathbb{E}_{J,J'} \| \mathrm{off}(\Kb_{\J'}^{J\cap J'})\|_{\fro}^{2}} \\
& =\mathbb{E}_{J} \mathbb{E}_{J'}\| \mathrm{off}(\KJb^{J\cap J'})\|_{\fro}^{2}
\end{split}
\end{equation*}
Using Lemma~\ref{lem:ExpectOverJ} we obtain
\[
\|\mathrm{off}(\Wb_{k})\|_{\fro}^{2}
\leq \mathbb{E}_{J} \tfrac{k(k-1)}{p(p-1)} \| \mathrm{off}(\KJb)\|_{\fro}^{2}\notag
\]
Specializing to $\Ub_{k}$ and using again Lemma~\ref{lem:ExpectOverJ} we obtain
\begin{eqnarray*}
\|\mathrm{off}(\Ub_{k})\|_{\fro}^{2}
&\leq&
 \mathbb{E}_{J} \tfrac{k(k-1)}{p(p-1)} \tfrac{1}{(1-\LRIP)^2}  \|\Ib-\GramJb\|_{\fro}^{2}\\
&=&
\tfrac{k(k-1)}{p(p-1)}\tfrac{1}{(1-\LRIP)^2}\tfrac{k(k-1)}{p(p-1)}\|\Ib-\DbT\Db\|_{\fro}^{2} \label{eq:Bound1}\notag
\end{eqnarray*}
It follows that
\begin{equation}
\|\mathrm{off}(\Ub_{k})\|_{\fro} \leq \tfrac{k(k-1)}{p(p-1)} \frac{A}{1-\LRIP}.\notag
\end{equation}
Specializing now to $\Vb_{k}$ we obtain similarly
\begin{equation*}
\begin{split}
\|\mathrm{off}&(\Vb_{k})\|_{\fro}^{2}\\
 \leq & \tfrac{k(k-1)}{p(p-1)} \tfrac{2}{(1-\nu)^8}  \\
& \cdot \mathbb{E}_{J} \left\{\|\Ib-\GramJb\|_{\fro}^{2}+\|\Ib-\GramJbo\|_{\fro}^{2}\right\}\\
=& \tfrac{k(k-1)}{p(p-1)}\tfrac{2}{(1-\nu)^8}\tfrac{k(k-1)}{p(p-1)} \\
& \cdot \left\{\|\Ib-\DbT\Db\|_{\fro}^{2} + \|\Ib-\DboT\Dbo\|_{\fro}^{2} \right\}\label{eq:Bound1}
\end{split}
\end{equation*}
and we finally obtain
\begin{equation}
\|\mathrm{off}(\Vb_{k})\|_{\fro} \leq \tfrac{k(k-1)}{p(p-1)} \frac{2A}{(1-\LRIP)^2}.\notag
\end{equation}

\end{proof}

We can now proceed to the proof of Lemma~\ref{lem:bias_expectation}.
\paragraph{Proof of Equation~\eqref{eq:leading_expectation}}

We write $\Db = \Dbo \Cb(\thetab) + \Wb \Sb(\thetab)$ using Lemma~\ref{lem:paramonto}. 
For simplicity we first assume that $\thetab_{j} \neq \pi/2$, for all $j \in \SET{p}$. Hence, the matrix $\Cb(\thetab)$ is invertible and $\Dbo = \Db \Cb^{-1} - \Wb \Tb$ with $\Tb = \Diag(\tan(\thetab_{j}))$.  The columns of $[\Db\Cb^{-1}]_{\J}$ belong to the span of $\DbJ$ hence
\begin{eqnarray*}
\trace \big( \DboJT (\Ib - \PJb) \DboJ  \big)
&=&
\|(\Ib - \PJb) \DboJ\|_{\fro}^{2}\\
&=& \|(\Ib-\PJb)[\Wb\Tb]_{\J}\|_{\fro}^{2}\\
&=&
\|[\Wb\Tb]_\J\|_{\fro}^{2}-\| \PJb [\Wb\Tb]_\J\|_{\fro}^{2}.
\end{eqnarray*}
For the first term, by Lemma~\ref{lem:ExpectOverJ}, we have
\[
\Exp_{\J}\ \|[\Wb \Tb]_\J\|_{\fro}^{2} = \tfrac{k}{p}   \|\Wb \Tb\|_{\fro}^{2} 
\] 
For the second term, since $\PJb = \DbJ\ThetaJb\DbJ^{\top}$, using Lemma~\ref{lem:RIPBounds}, we have the bound
\begin{equation*}
\| \PJb [\Wb\Tb]_\J \|_{\fro}^{2}  
\leq 
\triple \ThetaJb \triple_{2}  \| \DbJ^\top [\Wb\Tb]_\J \|_\fro^2
\leq
\tfrac{1}{1-\LRIP} \| \DbJ^\top [\Wb\Tb]_\J \|_\fro^2,
\end{equation*}
Now, by Lemma~\ref{lem:ExpectOverJ}, 
\begin{equation*}
\Exp_{\J}\ \| \DbJT [\Wb \Tb]_\J \|_\fro^2 
=
\tfrac{k(k-1)}{p(p-1)}   \| \DbT \Wb \Tb \|_\fro^2 
\leq
\tfrac{k^{2}}{p^{2}}  B^{2}  \|\Wb \Tb \|_\fro^2
\end{equation*}
Putting the pieces together, we obtain the lower bound
\begin{eqnarray*}
\Exp_{\J}\ \trace \big( \DboJT (\Ib - \PJb) \DboJ  \big)
&\geq&
\tfrac kp  \|\Wb \Tb\|_\fro^2  \left(1- \tfrac{k}{p} \tfrac{B^{2}}{1-\LRIP} \right).
\end{eqnarray*}
To conclude, we observe that since $\|\wb^{j}\|_{2}=1$ and $\tan u \geq u$ for $0 \leq u \leq \pi/2$,
\[
\|\Wb\Tb\|_{\fro}^{2} = \sum_{j=1}^{p} \tan^{2} (\thetab_{j})^{2} \geq \sum_{j=1}^p \thetab_{j}^{2} = \|\thetab\|_{2}^{2}.
\]
Finally, by continuity the obtained bound also holds when $\thetab_{j} = \pi/2$ for some $j$.

\paragraph{Proof of Equation~\eqref{eq:bias_expectation}}

Applying Lemma~\ref{lem:paramonto} , we write
$
\Dbo = \Db \Cb(\thetab) +\Wb \Sb(\thetab), 
$
and obtain, 
\begin{eqnarray*}
\trace \left( \Ib - \DbJ^{+} \DboJ \right) &=& 
k - \sum_{j \in \J} \cos(\thetab_{j}) - \trace \left( \DbJ^{+} [\Wb \Sb(\thetab)]_\J \right),\\
&=& \sum_{j\in\J} ( 1 - \cos(\thetab_{j}) ) - \trace \left( \ThetaJb \DbJT [\Wb \Sb(\thetab)]_\J \right).
\end{eqnarray*}
The first term is simple to handle since we have, by Lemma~\ref{lem:ExpectOverJ} and the inequality $1-\cos(u) \leq u^{2}/2$, $u \in \Real$,
$$
\Exp_\J\ \sum_{j\in\J} ( 1 - \cos(\thetab_j) ) \leq \Exp_\J\ \frac{\|\thetab_\J\|_2^2}{2} = \frac{k}{p} \cdot \frac{\|\thetab\|_{2}^2}{2}.
$$
We now turn to the second term
\begin{eqnarray*}
\Exp_\J\  \trace \left[ \ThetaJb \DbJT [\Wb \Sb]_\J \right]   
&=& \Exp_\J\  \trace \left[ \ThetaJb (\Db\ \Ib_{\J})^{\top} \Wb \Sb \Ib_\J \right] \\
&=& \Exp_\J\ \trace \left[ \Ib_\J \ThetaJb \Ib_\J^\top \DbT \Wb \Sb \right] \\
&=&  \trace \left[ \Ub_{k} \DbT \Wb \Sb \right]. 
\end{eqnarray*}
Since $\diag(\DbT \Wb \Sb)=0$, it follows
\begin{eqnarray*}
\left| \trace \left[ \Ub_{k} \DbT \Wb \Sb \right] \right| 
&\leq& \|\mathrm{off}(\Ub_{k})\|_{\fro} \cdot  \|\DbT \Wb \Sb\|_{\fro}\\
&\leq& \|\mathrm{off}(\Ub_{k})\|_{\fro} \cdot B \cdot \|\Wb\Sb\|_{\fro}.
\end{eqnarray*}
Since $\|\wb^{j}\|_{2}=1$ and $\sin u \leq u$ for $0 \leq u \leq \pi/2$, we have $\| \Wb \Sb  \|_\fro \leq  \|\thetab\|_{2}$, and we conclude the proof using Lemma~\ref{lem:LRIP} and the fact that $(k-1)/(p-1) \leq k/p$. 

\paragraph{Proof of Equation~\eqref{eq:bias_expectation_signsign}}
Since $\ThetaJb = (\DbJT\DbJ)^{-1}$ and similarly for $\ThetaJbo$ we have
\begin{align}
\ThetaJbo-\ThetaJb
&=
\ThetaJbo(\DbJT\DbJ-\DboJT\DboJ)\ThetaJb\notag\\
&=
\ThetaJbo\Ib_{\J}^{\top}(\Db^{\top}\Db-\DboT\Dbo)\Ib_{\J}\ThetaJb\\
\trace\ \left[\ThetaJbo-\ThetaJb\right] 
&=
\trace\ \left[\Ib_{\J}\ThetaJb\ThetaJbo\Ib_{\J}^{\top}(\Db^{\top}\Db-\DboT\Dbo)\right]\\
\Exp_{\J}\ \trace\ \left[\ThetaJbo-\ThetaJb\right] 
&=
\trace\ \left[\Vb(\Db^{\top}\Db-\DboT\Dbo)\right]
\end{align}
Since $\diag(\Db^{\top}\Db-\DboT\Dbo) = 0$ we further have
\[
\left|\Exp_{\J}\ \trace\ \left[\ThetaJbo-\ThetaJb\right] \right|
\leq \|\mathrm{off}(\Vb)\|_{\fro} \cdot \|\Db^{\top}\Db-\DboT\Dbo\|_{\fro}.
\]
We conclude using Lemma~\ref{lem:LRIP} after noticing that
\begin{eqnarray*}
\|\Db^{\top}\Db-\DboT\Dbo\|_{\fro}
&\leq&
\|\Db^{\top}(\Db-\Dbo)\|_{\fro}\\
&&+\|(\Db^{\top}-\DboT)\Dbo\|_{\fro}\\
&\leq&  2 B \|\Db-\Dbo\|_{\fro} \leq 2B \|\thetab\|_{2}.
\end{eqnarray*}